\documentclass{article}

\usepackage{etoolbox}

\usepackage{natbib}
\newlength{\defbaselineskip}
\usepackage{amsthm}
\usepackage{amsfonts}
\usepackage{thmtools,thm-restate}
\usepackage{bbm}

\newtheorem{definition}{Definition}
\newtheorem{theorem}{Theorem}

\newtheorem{lemma}{Lemma}
\newtheorem{proposition}{Proposition}
\newtheorem{example}{Example}

\newcommand{\curly}[1]{\left\{#1\right\}}
\newcommand{\norm}[1]{\lVert #1 \rVert}

\newcommand{\prob}{\text{P}}
\newcommand{\tr}{\text{tr}}

\newcommand{\abs}[1]{\lvert #1 \rvert}
\newcommand{\Abs}[1]{\left\lvert #1 \right\rvert}
\usepackage{algpseudocode}

\usepackage{color}

\usepackage{algorithm}
\usepackage{algorithmicx}
\usepackage{algpseudocode}
\algrenewcommand\algorithmicindent{0.5em}

\usepackage{mathtools}

\usepackage{mathdots}

\usepackage{wrapfig}

\usepackage{tabularx}

\usepackage{subcaption}

\usepackage{tikz}
\usetikzlibrary{calc,arrows.meta,positioning}
\usepackage{enumitem}

\usepackage{makecell}

\usepackage{multirow}

\usepackage{balance}

\newcommand{\E}{\mathbb{E}}
\DeclareMathOperator*{\argmin}{argmin}
\DeclareMathOperator*{\argmax}{argmax}

\usepackage{pifont}
\newcommand{\xmark}{\ding{53}}

\usepackage{tablefootnote}
\usepackage{authblk}

\title{Ivy: Instrumental Variable Synthesis for Causal Inference}
\date{}
\author{Zhaobin Kuang\footnote{Correspondence to:  zhaobin.kuang@gmail.com},  Frederic Sala,  Nimit Sohoni, Sen Wu, \\
Aldo C\'{o}rdova-Palomera, Jared Dunnmon, James Priest, and Christopher R\'{e}}
\affil{Stanford University}

\begin{document}

\maketitle

\begin{abstract}
A popular way to estimate the causal effect of a variable $x$ on $y$ from observational data is to use an \emph{instrumental variable} (IV): a third variable $z$ that affects $y$ only through $x$. 
The more strongly $z$ is associated with $x$, the more reliable the estimate is, but such \emph{strong IVs} are difficult to find.
Instead, practitioners combine more commonly available \emph{IV candidates}---which are not necessarily strong, or even valid, IVs---into a single ``summary'' that is plugged into causal effect estimators in place of an IV.
In genetic epidemiology, such approaches are known as \emph{allele scores}. Allele scores require strong assumptions---independence and validity of all IV candidates---for the resulting estimate to be reliable.
To relax these assumptions, we propose Ivy, a new method to combine IV candidates that can handle correlated and invalid IV candidates in a robust manner.
Theoretically, we characterize this robustness, its limits, and its impact on the resulting causal estimates.
Empirically, we show that Ivy can correctly identify the directionality of known relationships and is robust against false discovery (median effect size $\le 0.025$) on three real-world datasets with no causal effects, while allele scores return more biased estimates (median effect size $\ge 0.118$).
\end{abstract}

\section{Introduction}
A goal of causal inference is to ascertain the causal relationship between a pair of variables (the \emph{risk factor} $x$ and the \emph{outcome} $y$) from observational data. This is difficult because causal relationships can be distorted by \textit{confounders}: common causes of the risk factor and the outcome that may be unobserved. To address this difficulty, a third variable, called an \emph{instrumental variable} (IV), can be used to estimate causal effect. Informally, an IV only affects the outcome through its effect on the risk factor. IV methods are widely used in practice \citep{angrist1991does, Mokry71, walker2017mendelian, Millwood2019}. In particular, we are motivated by Mendelian randomization (MR) \citep{Burgess15}, a representative use case in which genetic markers serve as IVs to infer causation among clinical variables. 

IV methods are most reliable when the IV $z$ is strongly associated with the risk factor $x$, but such \emph{strong IVs} are often difficult to identify in practice. Instead, practitioners typically rely on more readily available \emph{IV candidates}. These variables may not be strong, or even valid, IVs, but can be used in lieu of an unavailable strong IV. To this end, a two-phase approach can be used: first, \emph{synthesize}: combine the IV candidates into a summary variable, and secondly, \emph{estimate}: plug the summary variable into a causal effect estimator.

In MR, a popular, state-of-the-art approach for the synthesis phase is \emph{allele scores}. The summary variables generated by allele scores are meant to reduce bias in causal estimates \citep{Angrist2008, davies2015many}. In the words of \citet{burgess2017review}, allele scores are a ``recent innovation'' in MR and are a ``recommend[ed]'' way to utilize plentiful IV candidates---but with the caveat that if an IV candidate is not actually a valid IV, allele scores may lead to ``potentially misleading estimates.'' Indeed, allele score methods suffer two main weaknesses: they implicitly assume that the IV candidates (1) are \emph{all} valid IVs and (2) are independent conditioned on the summary variable \citep{sebastiani2012naive}. When these assumptions are not met, as often happens in practice, the resulting estimate may turn out to be unreliable.\footnote{See Appendix~\ref{sec:related-work} for an extended discussion.}

To improve robustness against invalidity and dependencies among the IV candidates while still reaping the benefits of the two-phase approach (e.g., modularity and bias reduction), we propose Ivy, a novel way to synthesize a summary IV from IV candidates. Ivy produces a summary IV by modeling it as a latent variable, and inferring its value based on the statistical dependencies among the IV candidates. Ivy is inspired by recent advances in the theory of weak supervision, leveraging results on structure learning \citep{Varma19}. Ivy targets the synthesis phase and is orthogonal to the effect estimation phase: the summary IV it generates can directly be plugged into IV-based causal effect estimators, whether they are classical \citep{Wald40, Angrist96}, robust \citep{bowden2016consistent, kang2016instrumental}, or modern \citep{Hartford17, athey2019generalized}.

We provide theoretical bounds on the robustness of our approach against invalidity or dependencies among the IV candidates. Specifically,
\setlist{nolistsep}
\begin{itemize}[leftmargin=*]
\itemsep0.3em
\item We analyze the parameter estimation error for Ivy. 
Under weaker assumptions than allele scores, and with sufficiently many samples, Ivy's error scales as $O(1/\sqrt{n})$ for $n$ samples. Even outside of this regime, when Ivy may fail to identify all invalid IVs or dependencies, the resulting error is mild (scaling linearly in the number of misspecified dependencies and undetected invalid IVs). 
\item We translate the error in the parameter estimation into bounds for a downstream parametric causal effect estimator ---the Wald estimator---which is a commonly used estimator in MR. 
\item We further adapt our analysis to show how, in contrast to Ivy, allele scores may produce unreliable estimates in the presence of invalidity or dependency among IV candidates.
\end{itemize}

Empirically, we show that Ivy can more reliably estimate causal effects compared to allele score methods, even with low-quality uncurated IV candidates with potential dependencies and invalidity. On three real-world datasets with no causal effects, Ivy yields median effect size less than $0.025$, while allele scores return more biased estimates (median effect size $\ge 0.118$). This result aligns with our theoretical insights into Ivy and allele scores.

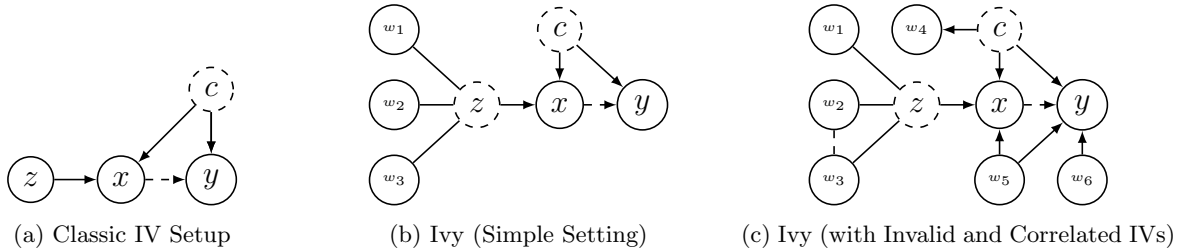
\begin{figure*}[t]
\centering
\begin{subfigure}[b]{0.28\textwidth}
\centering
\begin {tikzpicture}[-latex ,auto ,node distance
=1.2cm and 1.2cm , on grid, semithick ,
state/.style ={ circle, draw, minimum width=0.5cm}, scale=0.50]
\node[state] (C) [dashed] {\large $c$};
\node[state] (Y) [below =of C] {\large $y$};
\node[state] (X) [below left=of C] {\large $x$};
\node[state] (Z) [left=of X] {\large $z$};
\path (Z) edge (X);
\path (C) edge (X);
\path (C) edge (Y);
\path (C) edge (Y);
\path[dashed] (X) edge (Y);
\end{tikzpicture}
\caption{Classic IV Setup}
\label{fig:classic-iv}
\end{subfigure}~
\begin{subfigure}[b]{0.33\textwidth}
\centering
\begin {tikzpicture}[-latex ,auto ,node distance
=1.0cm and 1.1cm , on grid, semithick ,
state/.style ={ circle, draw, minimum width=0.5cm}, scale=0.50]
\node[state, draw] (C) [dashed] {\large $c$};
\node[state, draw] (Y) [below right=of C] {\large $y$};
\node[state] (X) [below =of C] {\large $x$};
\node[state] (Z) [left=of X, dashed] {\large $z$};
\node[state] (W1) [above left=of Z] {\tiny $w_1$};
\node[state] (W2) [left=of Z] {\tiny $w_2$};
\node[state] (W3) [below=of W2] {\tiny $w_3$};
\path (Z) edge[-] (W1);
\path (Z) edge[-] (W2);
\path (Z) edge[-] (W3);
\path (Z) edge (X);
\path (C) edge (X);
\path (C) edge (Y);
\path[dashed] (X) edge (Y);
\end{tikzpicture}
\caption{Ivy (Simple Setting)}
\label{fig:ivy-ideal}
\end{subfigure}~
\begin{subfigure}[b]{0.35\textwidth}
\centering
\begin {tikzpicture}[-latex ,auto ,node distance
=1.0cm and 1.1cm , on grid, semithick ,
state/.style ={ circle, draw, minimum width=0.5cm}, scale=0.50]
\node[state] (C) [dashed] {\large $c$};
\node[state] (Y) [below right=of C] {\large $y$};
\node[state] (X) [below =of C] {\large $x$};
\node[state] (Z) [left=of X, dashed] {\large $z$};
\node[state] (W1) [above left=of Z] {\tiny $w_1$};
\node[state] (W2) [left=of Z] {\tiny $w_2$};
\node[state] (W3) [below=of W2] {\tiny $w_3$};
\node[state] (W4) [left=of C] {\tiny $w_4$};
\node[state] (W5) [below=of X] {\tiny $w_5$};
\node[state] (W6) [below=of Y] {\tiny $w_6$};
\path (Z) edge[-] (W1);
\path (Z) edge[-] (W2);
\path (Z) edge[-] (W3);
\path (Z) edge (X);
\path (C) edge (X);
\path (C) edge (Y);
\path (W5) edge (X);
\path (W5) edge (Y);
\path (W6) edge (Y);
\path (C) edge (W4);
\path[dashed] (X) edge (Y);
\path[dashed] (W2) edge[-] (W3);
\end{tikzpicture}
\caption{Ivy (with Invalid and Correlated IVs)}
\label{fig:ivy-general}
\end{subfigure}
\caption{IV method settings (unobserved variables are dashed; the dashed arrow between $x$ and $y$ is the causal relationship we seek to estimate, dashed edges are dependencies that we seek to infer): (a) the traditional setting with observed strong IV $z$, (b) a simple setting where we do not see $z$, but see noisy weak IV candidates $w_1, w_2, w_3$ independent conditioned on $z$, (c) a more challenging setting that Ivy can handle where some IV candidates have dependencies ($w_2, w_3$), others are invalid ($w_4$ violates unconfoundedness, $w_5$ and $w_6$ violate exclusion restriction, and $w_6$ violates relevance).}
\end{figure*}

\section{Background}
\label{sec:background}
We consider a two-phase approach to estimating causal effects with IV candidates. First, the IV candidates are combined to form a summary (the synthesis phase). Second, in the effect estimation phase, this summary is plugged into an estimator, along with the risk factor and outcome, to produce an effect. Our approach tackles the first phase, and is orthogonal to the second phase. We give background on these ideas below.

We seek to infer the causal relationship between a risk factor $x$ and an outcome $y$. This relationship may be distorted by a confounder $c$, which is a common cause of both $x$ and $y$. To handle confounding, an instrumental variable $z$ may be used. $z$ directly induces a change in $x$ independent of $c$. This change will alter the value of $y$ only through the causal link between $x$ and $y$, enabling us to measure the causal link (Figure~\ref{fig:classic-iv}). We focus on the setting where $x$, $y$, $c$, and $z$ are binary, although our procedure can be extended to handle continuous $x$, $y$, and $c$. A valid IV is a variable satisfying Definition~\ref{def:iv}; otherwise,  it is invalid. 
\begin{definition}[\cite{Burgess15}]
\label{def:iv} \normalfont
An instrumental variable $z$ satisfies 
(i) \emph{Relevance}: $z$ is not independent of the risk factor, i.e.\ $z \not\perp x$; (ii) \emph{Exclusion Restriction}: $z$ can only influence the outcome through $x$, i.e.\ $z \perp y \mid x, c$;  (iii) \emph{Unconfoundedness}: $z$ is independent of the confounder, i.e.\ $z \perp c$.
\end{definition}
Figure~\ref{fig:classic-iv} depicts the setting where a valid IV is observable. The dashed confounder node $c$ indicates that IV methods can deal with unobserved confounders between $x$ and $y$. By contrast, estimating effects without accounting for confounding may lead to failure in distinguishing between spurious correlation and causation. The following is a well-known example of spurious correlation in epidemiology, dismissed by a careful use of IVs.

\begin{example}
\label{example:hdl}
The concentration of high-density lipoprotein (HDL) is negatively correlated with the occurrence of coronary artery disease (CAD) and thus appears protective, but recent  studies suggest that there is no causal link. The correlation is spurious due to confounders such as the concentration of other lipid species
\citep{rye2015hdl}. Nonetheless, the strength of this spurious correlation led to a hypothesized causal link, but drugs developed to raise HDL levels failed to prevent CAD \citep{schwartz2012effects}. This spurious correlation was later dismissed by a series of MR studies \citep{voight2012plasma, holmes2014mendelian, rader2014hdl}.
\end{example}

\subsection{IV Synthesis}
\label{sec:iv-candidates} 
The more strongly a valid IV is associated with the risk factor, the more reliable the resulting causal effect estimate.
However, finding such strong IVs is challenging in practice.  Instead, practitioners often combine more widely available IV candidates---variables that are weakly associated with the risk factors, intercorrelated, or even invalid IVs---into a summary IV. One way to view this procedure is that the summary IV is a prediction of a latent variable that, while unobserved, can serve as a strong IV.

\paragraph{Allele Scores} The use of unweighted/weighted allele scores (UAS/WAS) to synthesize a summary IV is a popular leading approach in MR \citep{burgess2013use, davies2015many, burgess2016combining}. UAS weights each IV candidate equally while WAS weights them based on their associations to the risk factor.
While allele scores can mitigate bias induced by weak IV candidates, they assume that these IV candidates are all valid and independent conditioned on the summary (Figure \ref{fig:ivy-ideal}). Thus, dependencies \citep{sebastiani2012naive} or invalidity \citep{burgess2017review} in IV candidates (Figure \ref{fig:ivy-general}) can still result in unreliable effect estimates when using the summary variable. Our proposed approach, Ivy, can be viewed as a \emph{generalization} of allele scores to lessen these issues.

\subsection{Effect Estimation}
\label{sec:wald}
In the effect estimation phase, the risk factor $x$, the outcome $y$, and the summary (or, when available, the strong IV) $z$ are used in an estimation procedure to obtain an estimate of the causal effect of $x$ on $y$.

In MR, the standard estimator is the Wald ratio ${\beta_{zy}}/{\beta_{zx}}$, where $\beta_{zx}$ and $\beta_{zy}$ are the logistic regression coefficients of predicting $x$ and $y$ using $z$, respectively. While Ivy can be plugged into other estimators, we analyze the estimation phase for the commonly used Wald estimator in MR.

\section{IV Synthesis With Ivy}
\label{sec:theory}

We describe the Ivy framework for instrumental variable synthesis. We begin with our problem setup and assumptions. Then we present
Ivy (Algorithm~\ref{alg:ivy}) and its components.
Next, we theoretically characterize the model parameter estimation error in Ivy due to invalid IV candidates, misspecified dependencies, and sampling noise. Finally, we bound the impact of this error on downstream causal effect estimation.

\subsection{Problem Setup}
\label{sec:setup-assumptions}

We seek to use a valid, but unobserved  IV $z\in\curly{-1,1}$ to infer the causal relationship between the risk factor $x\in\curly{-1,1}$ and the outcome $y\in\curly{-1,1}$. This causal relationship is obscured by potentially unobserved confounders $c\in\curly{-1,1}^d$. The data generation process among $x,y,z$, and $c$ follows some probability distribution $\mathcal{D}$. Although we do not directly observe $z$, we do observe $m$ IV candidates $w \in \curly{-1,1}^m$. Only some of these $m$ IV candidates are valid.

If the IV $z$ could be observed, we could directly plug it into a causal effect estimator; unfortunately, $z$ is rarely known in practice. Thus, the primary challenge is to reliably infer $z$ from $w$, i.e.\ to estimate the distribution $\text{P}(z \mid w)$, and to characterize how this impacts the reliability of downstream causal inference. 

\paragraph{Notation}
We use ``IV candidate'' and ``candidate'' interchangeably. We call candidates that are valid/invalid IVs ``valid/invalid candidates''. We denote the index set of the valid candidates as $V \subseteq [m]$, where $[m] := \curly{1,2,\ldots,m}$. We use $w_V$ to represent the subvector of the vector $w$ indexed by $V$ (i.e. the subvector corresponding to the valid candidates). When the subscript is omitted, $\norm{\cdot}$ denotes the $2$-norm.

\paragraph{Inputs and Outputs} 
We have access to data $\curly{(x^{(i)},y^{(i)},w^{(i)})}_{i=1}^n$: $n$ samples each of the risk $x$, the outcome $y$, and the $m$ IV candidates. Our goal is to produce a causal effect estimate  $\hat{\alpha}_{x \rightarrow y}$ of $x$ on $y$. 

\subsection{Assumptions}
\label{sec:assumptions}
We explain the assumptions made by Ivy, in particular comparing to those made by allele scores. These are described in further depth in Section~\ref{sec:assumptions-appendix}.

First, we describe assumptions on validity. We assume the majority of IV candidates are valid IVs, and for the invalid candidates $(i \not\in V$), $w_i \perp z$. These assumptions weaken those of allele scores, which assume that all candidates are valid IVs. 

Next, we continue with assumptions on dependencies. To allow for dependencies, we model the candidates and $z$ via an Ising model (the canonical binary maximum-entropy distribution with pairwise dependencies). We write the density of the model as
\begin{equation}
\label{eq:ising}
\frac{1}{\mathcal{Z}} \exp  ( \theta_z^* z + \Sigma_{i\in V} \theta_{i}^* w_i z  +  \Sigma_{(i,j)\in E} \theta_{ij}^* w_i w_{j} ),
\end{equation}
where $\mathcal{Z}$ is a normalization constant, $E$ is the set of pairwise dependencies between valid IVs, and the $\theta^*$ terms are the model parameters. While allele scores require the maximal level of sparsity in the model (no dependencies, so that $E$ is empty), our assumptions are weaker: we only require that for each valid IV candidate $w_i$ there are at least two others that are independent of $w_i$ and each other conditioned on $z$, and, conversely, that candidates that are dependent (i.e., in $E$) are all mutually dependent. 
Lastly, we require that on average, valid IV candidates agree with $z$ more often than not. 
We discuss identifiability of causal effects in Appendix~\ref{sec:appendix-identifiability}.

\subsection{Algorithmic Framework}
\label{sec:ivy-framework}
We describe the Ivy framework (Algorithm~\ref{alg:ivy}).
First, because our data may include both valid and invalid IV candidates, and because even the valid candidates may have dependencies, we learn a set of valid candidates and dependencies directly from our data (Algorithm~\ref{alg:sl1}). Next, we learn the mean parameters of the joint distribution of our estimated valid $w_i$'s and $z$, without observing $z$ (Algorithm~\ref{alg:forward}). Concretely, $(\mu^*, O^*)$, the true mean parameters\footnote{These are expectations of the sufficient statistics in \eqref{eq:ising}. $\mathbb{E}[z]$ is also a parameter; we assume it is known, but it can also be estimated (see, for example, \citealt{Ratner19}).}, are $\E[w z]$ and $\E[w w^T]$ (where $\E[w z]$ is a vector with entries $\E[w_i z]$). We observe the $w$'s, so we can easily estimate $O^*$ by $\hat{O}$. More challenging is to estimate $\mu^*$, since we do not observe $z$; we use our learned dependencies and validity to estimate $\mu^*$ by $\hat{\mu}$. Finally, in Algorithm~\ref{alg:causality} we use $\hat{\mu}$ and $\hat{O}$ to form an estimate $\hat{z}$ of $z$. We also describe how to use $\hat{z}$ in a generic IV-based estimator $F$ to get a causal effect estimate (the estimation phase). 
We describe the components of Algorithm~\ref{alg:ivy} in detail.

\begin{algorithm}[t]
\caption{Ivy Algorithmic Framework}
\label{alg:ivy}
\begin{algorithmic}[1]
\Require Data $\curly{(w^{(i)},x^{(i)},y^{(i)})}_{i=1}^n$.
\State $\hat{V}, \hat{E} \leftarrow$ \textsc{StructureLearn} $(\text{data}, \lambda, \gamma, T_1, T_2)$
\State $\hat{\mu} \leftarrow$ \textsc{ParamLearn} $(\text{data}, \hat{V}, \hat{E})$
\State $\hat{\alpha}_{x \rightarrow y} \leftarrow$ \textsc{CausalEst} $(\text{Estimator}, \text{data}, \hat{V}, \hat{\mu})$
\Ensure Causal effect estimate $\hat{\alpha}_{x \rightarrow y}$.
\end{algorithmic}
\end{algorithm}

\paragraph{Step 1: Identify Valid IV Candidates and their Dependencies.} \emph{Inputs}: data and hyperparameters. {\it Outputs}: estimated set of valid candidates $\hat{V}$ and estimated dependency set $\hat{E}$ of $\hat{V}$. Our method for learning the valid IVs and their dependencies is an application of recent approaches for structure learning \citep{Varma19} in graphical models.  The main challenge is that without observing $z$, all of the valid IV candidates
will appear to be correlated, although may be independent conditioned on z. Meanwhile, the valid and invalid candidates form mutually-independent components. We recover both the graph structure and the covariances between the IV candidates (valid and invalid) and $z$ via a robust PCA approach. This enables us to estimate which IVs are valid and their statistical dependencies. The procedure is given in Algorithm~\ref{alg:sl1}.

Concretely, the identification of the valid candidates and their dependencies translates to decomposing a rank-one matrix and a sparse matrix from their sum (Line~\ref{state:matrix-recovery} of Algorithm~\ref{alg:sl1}). Here, the candidate validity ends up being encoded in the rank-one component $\hat{L}$ and the dependencies are encoded in the sparse component $\hat{S}$. Thus, we can threshold the vector corresponding to the rank-one matrix $\hat{L}$ to obtain the valid IVs and then threshold the corresponding submatrix of $\hat{S}$ containing valid IVs to obtain the dependencies. There are several choices of loss functions. For our analysis, we use $ \mathcal{L}(S-L, \hat{\Sigma}) = \frac{1}{2} \tr((S-L)\hat{\Sigma}(S-L)) - \tr(S-L)$.

\begin{algorithm}[t]
\caption{Valid IV and Dependency Learning (\textsc{StructureLearn}) }
\label{alg:sl1}
\begin{algorithmic}[1]
\Require Data $\curly{w^{(i)}}_{i=1}^n$, params. $\lambda$, $\gamma$, $T_1$, and $T_2$.
\State Compute sample covariance matrix $\hat{\Sigma}$ from $w^{(i)}$'s. 
\State
$(\hat{S}, \hat{L}) \hspace{-1mm} \leftarrow \hspace{-4mm} \argmin\limits_{L \succeq 0, \, S-L \succ 0} \hspace{-2.5mm} \mathcal{L} (S-L,\hat{\Sigma}) \hspace{-0.5mm} + \hspace{-1mm} \lambda_n (\gamma\|S\|_{1} + \hspace{-1.5mm} \|L\|_*)$, where $\mathcal{L}$ is a loss function.
\label{state:matrix-recovery}
\State $\hat{\ell} \leftarrow \argmin_\ell \| \hat{L} -  \ell \ell^T\|_F$
\State $\hat{V} \leftarrow \{ j: |(\hat{\Sigma}\hat{\ell})_j| \geq T_1 \}$ \label{state:scale-covariance}
\State $\hat{E} \leftarrow \curly{(i,j): i,j \in \hat{V}, i < j, \hat{S}_{i,j} > T_2}$
\Ensure Estimated valid IV candidate set $\hat{V}$, estimated dependency set $\hat{E}$.
\end{algorithmic}
\end{algorithm}

\begin{algorithm}[t]
\caption{Parameter Learning (\textsc{ParamLearn})}
\label{alg:forward}
\begin{algorithmic}[1]
\Require Data $\curly{w^{(i)}}_{i=1}^n$, $\hat{G} = (\hat{V}, \hat{E})$ where $\hat{V}$ are estimated valid candidates, $\hat{E}$ are edges among them.
\State Form estimated matrix $\hat{O} \leftarrow \frac{1}{n} \sum_{i=1}^n w^{(i)}_{\hat{V}}(w^{(i)}_{\hat{V}})^T$.
\State $\hat{\Omega}\leftarrow \{(i,j): w_i, w_j \text{ are disconnected in } \hat{G} \setminus \{z\}\}$
\State Form matrix $M_{\hat{\Omega}}$ and vector $\hat{q}$ from $\hat{O}$
\State $\hat{\ell} \leftarrow \argmin_{\ell} \| M_{\hat{\Omega}} \ell - \hat{q}\|$
\State $|\hat{\mu}| \leftarrow \exp(\hat{\ell}/2)$
\State Recover $\text{sgn}(\hat{\mu})$
\Ensure Estimated model mean parameters $\hat{O}, \hat{\mu}$. 
\end{algorithmic}
\end{algorithm}

\begin{algorithm}[t!]
\caption{Synthesis \& Causal Effect Estimation (\textsc{CausalEst})}
\label{alg:causality}
\begin{algorithmic}[1]
\Require Data  $\curly{(w^{(i)},x^{(i)},y^{(i)})}_{i=1}^n$, estimated parameters $\hat{O}, \hat{\mu}$, $\hat{V}$, $\hat{E}$, causal effect estimator $F(\cdot)$.
\For{$i \in [n]$} \Comment{Synthesize}
\State $\hat{z}^{(i)} \leftarrow \prob_{\hat{\mu}, \hat{O}} \left(z \mid w_{\hat{V}}^{(i)} \right)$
\EndFor
\State $\hat{\alpha}_{x \rightarrow y} \leftarrow F \left(\curly{(\hat{z}^{(i)},x^{(i)},y^{(i)})}_{i=1}^n \right)$. \Comment{Estimate}
\Ensure Causal effect estimate $\hat{\alpha}_{x \rightarrow y}$. 
\end{algorithmic}
\end{algorithm}

\paragraph{Step 2: Estimate Parameters of the Candidate Model.}
\emph{Inputs}: data, $\hat{G} := (\hat{V}, \hat{E})$. \emph{Outputs}: estimated parameters $\hat{O}, \hat{\mu}$. In Algorithm~\ref{alg:forward}, we learn the mean parameters. We leverage conditional independencies encoded in our estimated dependency structure to obtain these parameters without ever observing $z$, via the agreements and disagreements of the IV candidates. We adapt \cite{Ratner19}. 

Specifically, we set $a_j := w_j z$ for all $j \in \hat{V}$. Then the mean parameter $\mu_i := \mathbb{E}[a_i] = \mathbb{E}[w_iz]$.
Since $z^2=1$, $\mathbb{E}[a_i a_j] = \mathbb{E}[(w_i z)(w_j z)] = \mathbb{E}[w_i w_j]$. We can estimate $\mathbb{E}[w_i w_j]$ from data. Moreover, if $w_i$ and $w_j$ are independent conditioned on $z$ (i.e. $(i,j)$ is an edge in $\hat{\Omega}$), then $\mathbb{E}[a_i a_j]  = \mathbb{E}[a_i] \mathbb{E}[a_j]$, which means $\log\mathbb{E}^2[a_i] + \log\mathbb{E}^2[a_j] = \log\mathbb{E}^2[w_i w_j]$. We form a system of equations $M_{\hat{\Omega}} \ell = q$, with $q$ the vector of $\log \mathbb{E}^2[w_i w_j]$ terms and $\ell$ the vector of $\log \mathbb{E}^2[a_i]$ terms. The matrix $M_{\hat{\Omega}}$ is formed by taking each $(i,j) \not\in \hat{\Omega}$ and adding a row with a $1$ in positions $i$ and $j$ and $0$'s elsewhere.
We solve this to get estimates $\hat{\mu}_i$ of $\mathbb{E}[a_i]$ up to sign; using the assumption that valid candidates agree with $z$ the majority of the time, we recover the signs. This gives $\hat{\mu}$ (and $\hat{O}$ was estimated earlier). 

\paragraph{Step 3: Synthesize IV and Estimate Causal Effect}  \emph{Inputs}: data, $\hat{O}, \hat{\mu}$, $\hat{V}$, $\hat{E}$, and causal effect estimator $F(\cdot)$. {\it Outputs}: causal effect estimate $\hat{\alpha}_{x \rightarrow y}$. Finally, in Algorithm~\ref{alg:causality}, we generate a probabilistically synthesized version of $z$ called $\hat{z}$ from our model parameterized by $\hat{O}, \hat{\mu}$. We obtain samples of $z$ based on these to account for the uncertainty in the synthesized summary IV, concluding synthesis. We then feed these samples, along the risk factor and the outcome, to a causal effect estimator in the estimation phase, producing a causal effect estimate.

\subsection{Theoretical Analysis}
\label{sec:estimation-error}
We theoretically analyze Ivy and provide bounds on its parameter estimation error. We further analyze the error in downstream causal effect estimation using the Wald estimator---a common estimator of causal effects in MR---as a proof-of-concept. We focus on the scaling with respect to the number of samples $n$ and the number of IV candidates $m$.  We present a simplified bound that explains the conceptual result, and provide a more general version in Appendix \ref{sec:ci-proof}.

\paragraph{Parameter Estimation Bound}

We show how the gap between the parameters  $\mu^*$ of \eqref{eq:ising} and our estimated $\hat{\mu}$ decays with the number of samples.\footnote{In Appendix~\ref{sec:ci-proof} we bound $\E[\norm{\hat{O}-O^*}]$ with Lemma~\ref{lem:conc}.} 
We fix $R_{\text{min}}$, the lowest correlation between valid candidates, and $C_{\min}$, the lowest accuracy for a valid candidate. Then, let $c_0, c_1$ be constants and $d$ be the largest degree of a valid IV candidate in $G$.

\begin{restatable}{theorem}{thmcibound}
\label{thm:ci_bound_1}
Let $\hat{\mu}$ be the result of Algorithm~\ref{alg:ivy} run on $n$ samples of $m $ IV candidates, where $m > c_0$. Denote $\mu^*$ to be the mean parameter of \eqref{eq:ising}. If $n > c_1d^2m$, then with probability at least $1-\frac{1}{m}$,
\begin{align*}
\mathbb{E}[\|{\hat{\mu}-\mu^*}\|] 
\leq \frac{16 m^{\frac{5}{2}}}{R_{\min}}  \|M^\dagger\| \sqrt{\frac{2\pi}{n}}.
\end{align*}
\end{restatable}

\paragraph{Remark} The bound on the estimation error goes to $0$ as $O(1/\sqrt{n})$, while it scales as $O(m^{5/2})$ in the number of IV candidates. The bound also depends on the smallest correlation between a pair of valid IVs; the smaller this term, the more samples we need to accurately estimate $\mu^*$. $\|M^\dagger\|$ is the largest singular value of the pseudoinverse of $M := M_{\Omega}$, i.e., the true $M$ formed with the edges from $G$; it indicates the cost of solving our problem (which is independent of $n$).

Under the assumptions in Section~\ref{sec:assumptions}, Ivy can handle invalid candidates and dependencies in $G$. This is because with sufficiently many samples (the requirement $n > c_1d^2m$), the structure learning component correctly identifies valid candidates and the correct dependencies among them, with high probability. The more dependencies that have to be estimated (that is, the larger the number of sources $m$ and degree $d$), the more samples we need. However, once we pass a threshold, we are operating only over valid IVs and a correct model, enabling the estimation error to go to zero. In Appendix \ref{sec:ci-proof}, we present a more technical result, applicable to the low-sample regime. In that case, the structure learning component may not identify all invalid IVs and may leave some edges, and we bound the impact of these unidentified invalid IVs and misspecified dependencies.

\paragraph{Application to Allele Scores}
UAS implicitly follows the conditionally independent model above. Our framework helps obtain new insights on its behavior. Specifically, when the ground truth model is \emph{not} conditionally independent, we can explain the approximation error in the parameters estimated by UAS.

As long as there is at least one misspecified dependency, the parameter error in UAS cannot go to zero. Specifically, let $n \rightarrow \infty$ and suppose there is a dependency between $w_1$ and $w_2$, but we miss it. Then, we do not have conditional independence, so $\mathbb{E}[w_1w_2] \neq \mathbb{E}[a_1]\mathbb{E}[a_2]$. Form $q'$ with $\mathbb{E}[a_1]\mathbb{E}[a_2]$ and $q$ with $\mathbb{E}[w_1w_2]$. We can write $q'-q = \delta e_1$ for some $\delta \neq 0$, since $q$ is only incorrect in one position. Then, $\|\ell'-\ell\| = \|M^\dagger(q'-q)\| = \|M^\dagger(\delta e_1)\| = |\delta| \|M^\dagger e_1\| \ge \tfrac{|\delta|}{\|M\|}$, which is a lower bound that is independent of $n$. Thus we obtain that $\mathbb{E}[\|\mu' - \mu^*\|] > 0$. 

\paragraph{Causal Effect Estimation Error} Next, we bound the causal effect estimation error when using Ivy's synthesized IV. We bound the mean squared error $\mathbb{E}[ (\hat{\alpha}_{x \rightarrow y} - \alpha^*_{x \rightarrow y})^2]$  
between the effect with Ivy's version of $z$ and that with the true $z$, as a function of the parameter error $\mathbb{E}[\norm{\hat{\mu}-\mu^*}] $ we obtained in Theorem~\ref{thm:ci_bound_1}.

We use the popular Wald estimator as an example. Let $\beta^*_{zx}$ and $\beta^*_{zy}$ be the population-level coefficients of $z$ from the logistic regressions to predict $x$ and $y$ under $\mathcal{D}$, and $\hat{\beta}_{\hat{z}y},\hat{\beta}_{\hat{z}x}$ the corresponding regression coefficients of $\hat{z}$. Define $\alpha^*_{x \rightarrow y} := {\beta^*_{zy}}/{\beta^*_{zx}}$ as the population-level Wald estimator. 
Suppose that the population-level logistic loss of $\mathcal{D}$ satisfies Lemma~\ref{lemma:strong-cvx} in Appendix \ref{sec:aux-lemma}, so that it is $\lambda$-strongly convex. Again suppose $m > c_0, n > c_1d^2m$ and large enough such that for some $\kappa \in (0,1)$, $\max\{\lvert \hat{\beta}_{\hat{z}y} - \beta_{zy}^* \rvert, \lvert \hat{\beta}_{\hat{z}x} - \beta_{zx}^* \rvert\} \le \kappa \beta_{zx}^*$, and let $c_2$ be a constant. 

\begin{restatable}{theorem}{thmciwald}
\label{thm:ci-wald}
Run Algorithm~\ref{alg:ivy} on $n$ samples of $m$ IV candidates to synthesize $\hat{z}$'s that are plugged into the Wald estimator to obtain the causal effect estimate $\hat{\alpha}_{x\rightarrow y}$. Then, the error in the estimate $\hat{\alpha}_{x\rightarrow y}$ compared to the true effect $\alpha^*_{x\rightarrow y}$ is bounded as follows:
\begin{align*}
\mathbb{E}&[(\hat{\alpha}_{x \rightarrow y} - \alpha^*_{x \rightarrow y})^2] 
\leq  \sqrt{\frac{1}{n}} \cdot \frac{6000 c_2 m^{\frac{5}{2}}(\beta_{zx}^*+\beta_{zy}^*)^2(1 +  \|M^\dagger\|)}{R_{\min} \lambda (1-\kappa)^2 \beta_{zx}^{*4}}.
\end{align*}
\end{restatable}

Theorem~\ref{thm:ci-wald} quantifies how the estimation error of $z$ propagates to the downstream Wald estimator. The error goes to 0 as $1/\sqrt{n}$, suggesting that, under the conditions we described, we can indeed perform reliable causal inference from weak IV candidates. Our final observation is that model misspecification may lead to nonzero error in the causal estimates (see Section~\ref{sec:lower-bound-causal-effect-appendix}): with even one misspecified dependency, $\mathbb{E}[\|\mu' - \mu^*\|] > 0$ with positive probability. We can lower bound $(\hat{\alpha}_{x \rightarrow y} - \alpha^*_{x \rightarrow y})^2$ in terms of $\mathbb{E}[\|\mu' - \mu^*\|]$, concluding that $(\hat{\alpha}_{x \rightarrow y} - \alpha^*_{x \rightarrow y})^2 > 0$ for such cases.

\section{Experiments}
\label{sec:exp}
We empirically validate that the summary IVs synthesized by Ivy lead to reliable causal effect estimates when plugged into standard causal effect estimators  on real-world healthcare datasets. Specifically,

\begin{itemize}[leftmargin=*]
\item In Section~\ref{sec:uncurated-exp}, we show, in clinically-motivated scenarios where only uncurated (potentially dependent or invalid) IV candidates are available, that Ivy can synthesize a summary IV that leads to more reliable effect estimates than allele scores.
\item In Section~\ref{sec:valid-exp}, in scenarios with hand-picked curated (putatively valid and conditionally independent) IV candidates, we show that Ivy performs comparably well to allele scores.
\item In Section~\ref{sec:synthetic-exp}, we evaluate the Ivy framework on synthetic data and further focus on its robustness against violation of key assumptions. 
\end{itemize}

We describe the datasets, methods, and evaluation metrics and then report our primary findings.\footnote{In Appendix~\ref{sec:extended-experiments}, we give further details about our setup and additional experiments.}

\paragraph{Datasets} 
In collaboration with cardiologists, we selected real-world health data collected from the UK Biobank \citep{sudlow2015uk} for a variety of cardiac conditions. Because heart diseases are a major class of conditions affected by many factors, we examined five factors (for instance, we study the LDL-CAD link, as in \citealt{burgess2016combining}). The most challenging aspect of selecting datasets for causal inference is the lack of ground truth effects. As a result, we have three desiderata for our dataset choices:
\begin{itemize}[leftmargin=*]
\item We need some risk-outcome pairs where strong clinical evidence exists to support that there is \emph{no causal relationship}, while for other pairs, there is strong evidence of a positive relationship;
\item We require standard pairs that have previously been tested against in the MR literature;
\item To evaluate performance in the favorable setting where IV candidates are valid and conditionally independent, we need access to curated sets of candidates. 
\end{itemize}
The five risk factors we use are high-density lipoprotein (HDL), low-density lipoprotein (LDL), systolic blood pressure (SBP), C-reactive protein (CRP), and vitamin D (VTD). The outcome is occurrence of coronary artery disease (CAD). Single-nucleotide polymorphisms (SNPs) associated with these factors are used as IV candidates. These pairs are well-understood by clinicians, enabling us to use these pairs as proxies to the ground truth \citep{c2011association, lieb2013genetic, holmes2014mendelian,  manousaki2016mendelian}. Using the risk factors, outcome, and IV candidates, we extract 11 datasets from the UK Biobank for our experiments (details in Table~\ref{tab:datasets}).

\paragraph{Methods} Ivy produces a summary in the synthesis phase, so we compare to allele scores---UAS and WAS---in Sections~\ref{sec:uncurated-exp}-\ref{sec:synthetic-exp}, as they also produce a summary IV. Additionally, we report results of logistic regression (Assn), which is a proxy for the confounded association between the risk factor and the outcome. 

\paragraph{Metric} After the synthesis phase, we use the summary IV in the estimation phase by plugging it into a causal effect estimator, along with the risk factor and the outcome. In all experiments, we use the Wald ratio to estimate effects. We report the median Wald ratio and its $95\%$ confidence interval (CI). In MR, a CI\ that covers the origin is interpreted as no causal effect, while strictly positive/negative CIs indicate positive/negative causal effects.

\begin{figure}[t]
\begin{subfigure}{\columnwidth}
\centering
\includegraphics[scale=0.8]{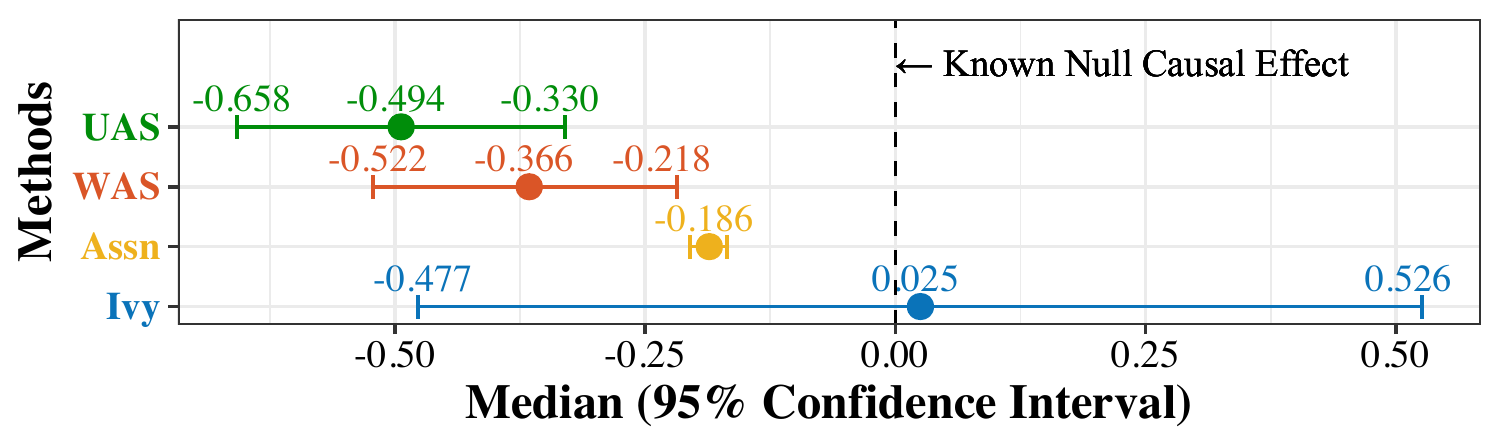}
\caption{\texttt{hdl}$\Rightarrow$\texttt{cad}}
\label{fig:hdl-cad-uncurated}
\end{subfigure}
\begin{subfigure}{\columnwidth}
\centering
\includegraphics[scale=0.8]{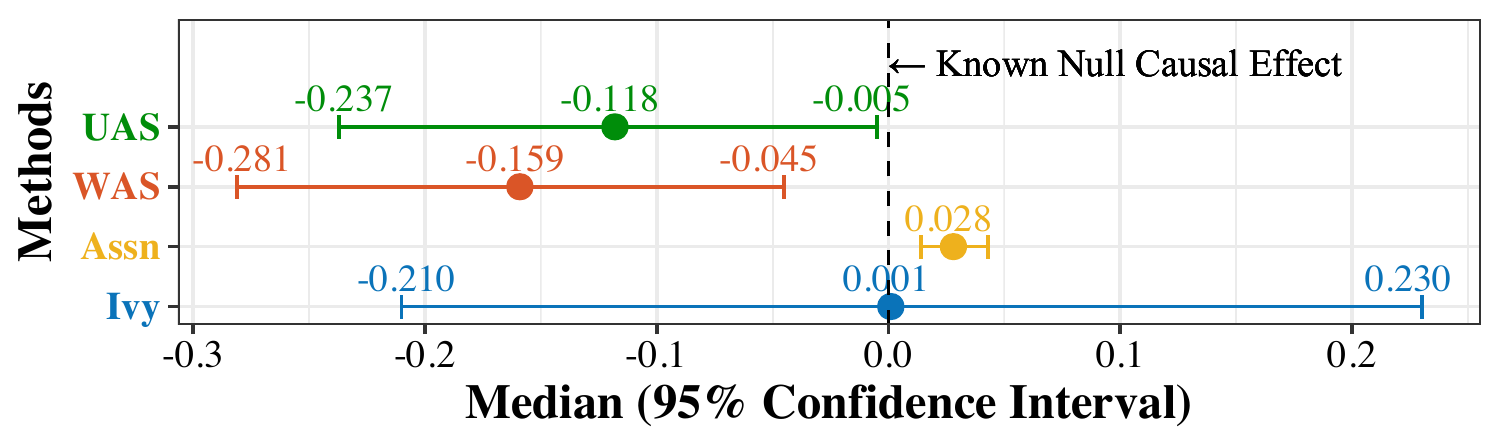}
\caption{\texttt{crp}$\Rightarrow$\texttt{cad}}
\label{fig:crp-cad-uncurated}
\end{subfigure}
\begin{subfigure}{\columnwidth}
\centering
\includegraphics[scale=0.8]{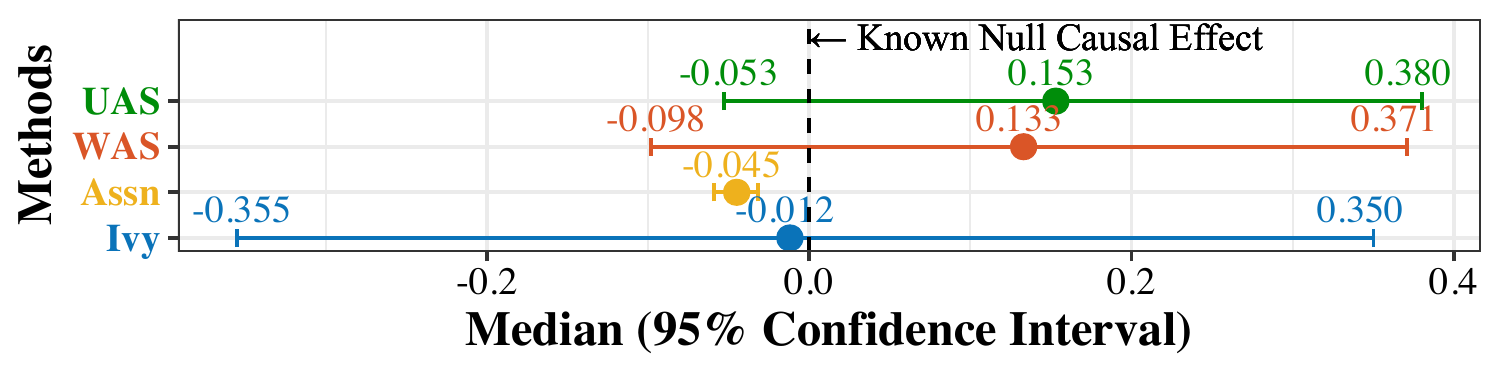}
\caption{\texttt{vtd}$\Rightarrow$\texttt{cad}}
\label{fig:vtd-cad-uncurated}
\end{subfigure}
\caption{Wald ratios in estimating the causal effects of three risk factors (HDL, C-reactive protein, and vitamin D) to the occurrence of coronary artery disease using uncurated IVs. Goal: 0 causal effect.} 
\label{fig:uncurated-exp}
\end{figure}

\subsection{MR with Uncurated IVs}
\label{sec:uncurated-exp}
We first use the summary variable synthesized by Ivy to draw causal inference in common clinical scenarios where only low-quality IV candidates are available. As shown in Figure~\ref{fig:uncurated-exp}, Ivy dismisses \emph{known spurious} correlations on all three of the real-world datasets (median effect size $\le 0.025$); in comparison, allele scores yield more biased estimates (median effect size $\ge 0.118$).

Specifically, we test spurious relationships between three potential risk factors (HDL, CRP, and VTD) and CAD: these are known to be noncausal, so the true effect size is 0. We compare Ivy with UAS, WAS, and Assn.
Results are in Figure~\ref{fig:uncurated-exp}. Both UAS and WAS return negative causal effects for HDL (UAS median: -0.494; WAS median: -0.366; Figure~\ref{fig:hdl-cad-uncurated})  and CRP (UAS median: -0.118; WAS median: -0.159; Figure~\ref{fig:crp-cad-uncurated}) with negative CIs. By contrast, Ivy does not identify a causal effect (Ivy median: 0.025 and 0.001 for HDL and CRP, respectively), with CIs covering the origin. In Figure~\ref{fig:vtd-cad-uncurated}, the CIs of all three methods cover the origin, indicating successful dismissal. Nonetheless, the median estimates of UAS (0.153) and WAS (0.133) are skewed towards the positive direction, while Ivy's is very close to the origin (-0.012). 

Ivy tends to have a wider confidence interval compared to allele scores, as it selects only a subset of IV candidates. Allele scores make use of all candidates regardless of their validity, and may be hurt by one or more being invalid. In all cases, Association (Assn) fails to dismiss spurious correlation, highlighting the importance of the use of IVs for debiasing causal estimates.

\begin{figure}[t]
\centering
\begin{subfigure}[b]{0.48\textwidth}
\centering
\includegraphics[scale=0.5]{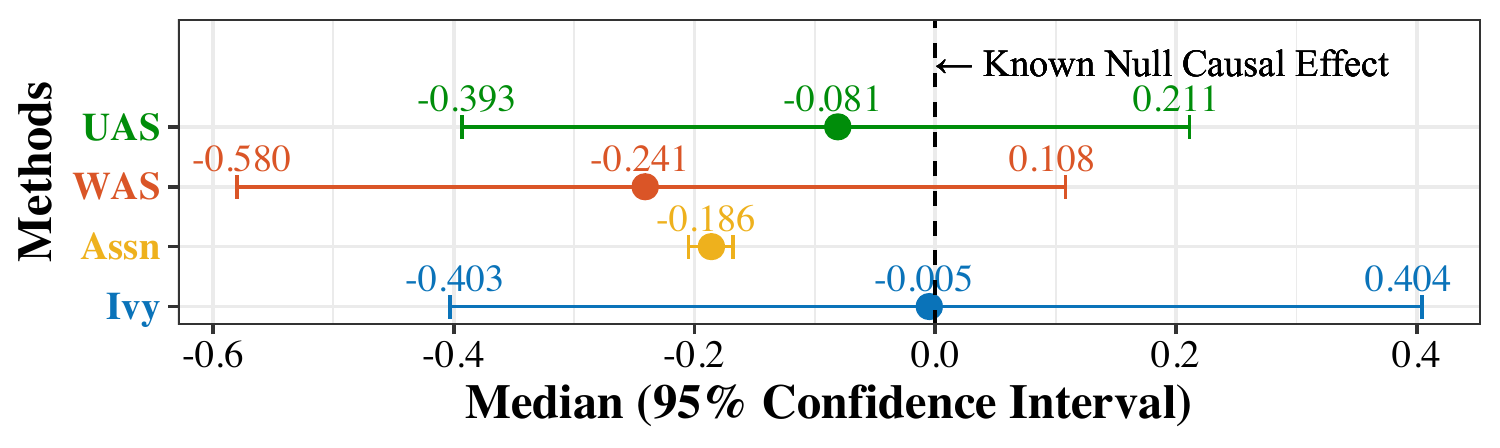}
\caption{\texttt{hdl}$\rightarrow$\texttt{cad}}
\label{fig:hdl-valid}
\end{subfigure}
\begin{subfigure}[b]{0.48\textwidth}
\centering
\includegraphics[scale=0.5]{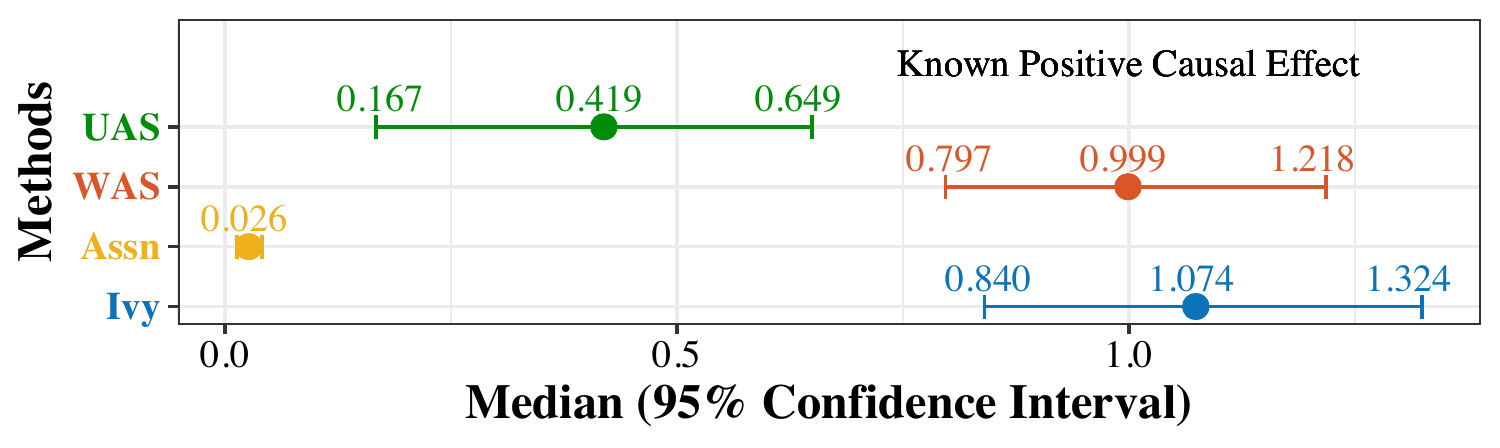}
\caption{\texttt{ldl}$\rightarrow$\texttt{cad}}
\label{fig:ldl-valid}
\end{subfigure}
\caption{Wald ratios in estimating the causal effect between high(low)-density lipoprotein and coronary artery disease using curated IVs. Goal in (a): 0 causal effect. Goal in (b): positive causal effect.}
\label{fig:valid-exp}
\end{figure}

\subsection{MR with Curated IVs}
\label{sec:valid-exp}
Next, we use a summary IV using a set of curated (putatively valid and conditionally independent) candidates with both known non-causal and known causal pairs. While all methods work, for the positive LDL-CAD relationship, Ivy retains the positive performance of WAS over UAS. The results are in Figure~\ref{fig:valid-exp}.

Concretely, since we are now in the fortunate (but rarer) setting in which the IV candidates are ``good,'' we expect that both Ivy and allele scores provide reasonable estimates. We use the known noncausal relationship between HDL and CAD (Example~\ref{example:hdl}) and the known positive causal relationship between LDL and CAD. Ivy is compared with UAS, WAS, and Assn. 
In terms of dismissing spurious correlation (Figure~\ref{fig:hdl-valid}), the $95\%$ CIs~of all three IV-based methods (Ivy, UAS, WAS) cover the origin, indicating successful dismissal. Notably, the median estimate of Ivy is closest to the origin (-0.005) compared to other methods (UAS median: -0.081; WAS median: -0.241), suggesting a potentially less biased estimate from Ivy.  Again, Assn fails to dismiss spurious correlation even in this ``easier'' setting. 

In terms of identifying a true causal relationship
(Figure~\ref{fig:ldl-valid}), all three IV-based methods correctly identify the direction of the causal relationship (UAS median: 0.419; WAS median: 0.999; Ivy median: 1.074), as indicated by the positive CIs~of the causal estimates. The lengths of the CIs of the three IV-based methods are also comparable to each other. On this dataset, Ivy yields an estimate most similar to that of WAS---matching the property that Ivy mimics allele scores in the setting where IV-candidates are high-quality.

\begin{figure}[t]
\centering
\begin{subfigure}[b]{0.49\columnwidth}
\centering
\scalebox{1}{
\begin{tikzpicture}[-latex, auto, node distance=1.0cm and 1.1cm , on grid, semithick, state/.style ={circle, draw, minimum width=0.5cm}, scale=0.50]
\node[state, draw] (W9) {\tiny $w_9$};
\node[state, draw] (Y) [below right=of W9] {\large $y$};
\node[state] (X) [below =of W9] {\large $x$};
\node[text width=3cm] (T) at (-1.5, 2) {Varying strength};
\node[text width=3cm] (Fake) at (-1.0, -6) {};
\node[text width=1cm] (F) at (-1.3, -0.7) {};
\node[state] (Z) [left=of X, dashed] {\large $z$};
\node[state] (W1) [above left=of Z] {\tiny $w_1$};
\node[state] (W2) [left=of Z] {\tiny $w_2$};
\node[state] (W8) [below=of W2] {\tiny $w_{8}$};
\node[] at (-4.45,-3.25) {\tiny $\smash{\vdots}$};
\path (Z) edge[-] (W1);
\path (Z) edge[-] (W2);
\path (Z) edge[-] (W8);
\path (Z) edge[red, very thick, -] (W9);
\path (Z) edge (X);
\path (W9) edge (X);
\path (W9) edge (Y);
\path (T) edge (F);
\end{tikzpicture}
}
\caption{Causal model}
\label{fig:invalid-z}
\end{subfigure}
\begin{subfigure}[b]{0.49\columnwidth}
\centering
\includegraphics[scale=0.55]{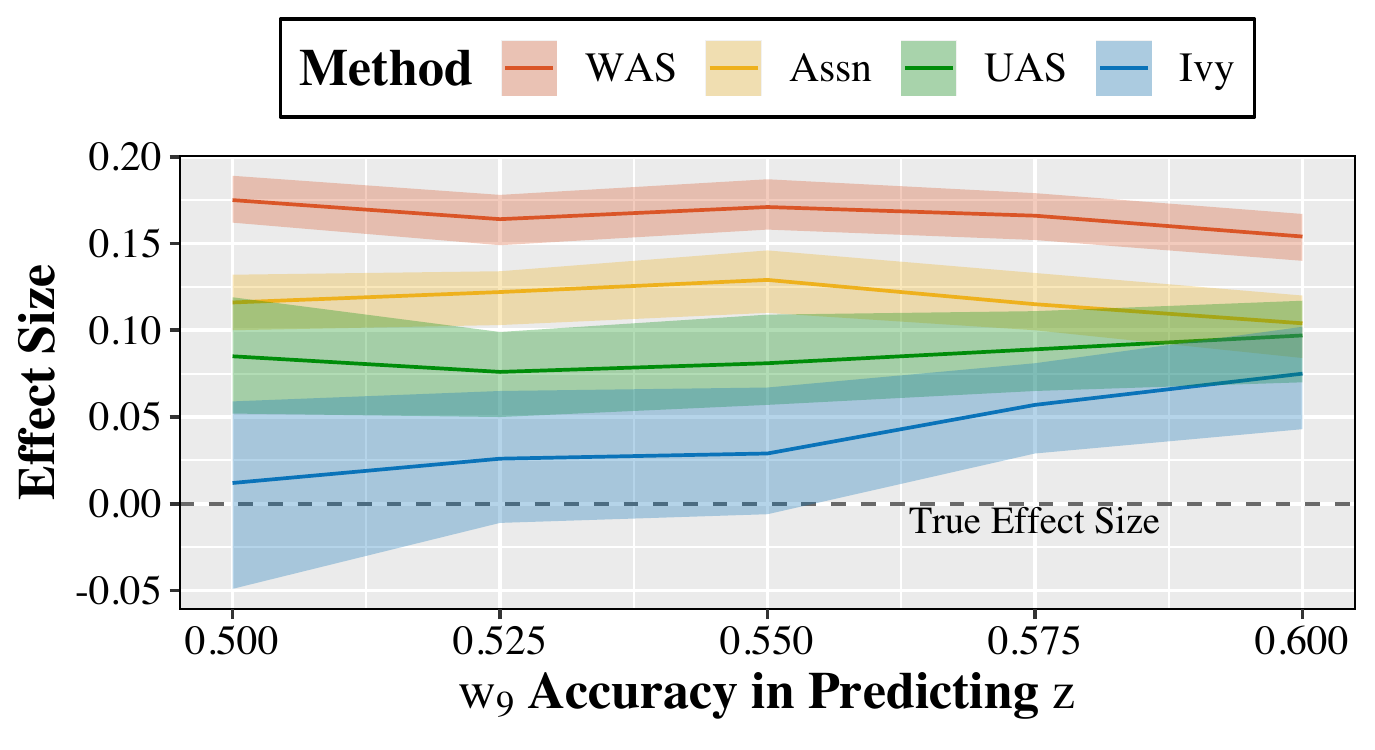}
\caption{Causal effect estimation}
\label{fig:invalid-z-result}
\end{subfigure}
\caption{Dismissing spurious correlations when $z$ is invalid. As the invalidity of $z$, i.e., the accuracy of $w_9$ in predicting $z$, increases, all methods eventually fail. However, Ivy is the most robust.}
\end{figure}
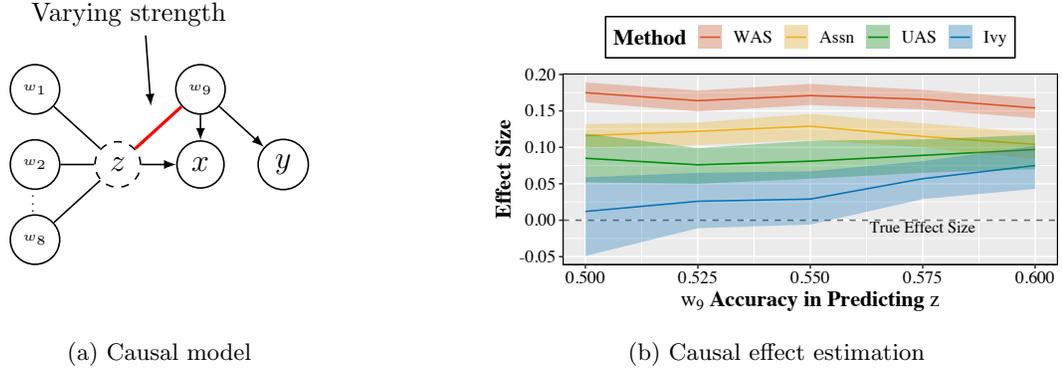

\subsection{Synthetic Experiments}
\label{sec:synthetic-exp}
Now we use synthetic data, controlling candidate properties and the ground-truth. We validate the robustness of Ivy and compare the effect to the ground-truth.

\begin{figure}[t]
\centering
\begin{subfigure}[b]{0.48\textwidth}
\centering
\includegraphics[scale=0.5]{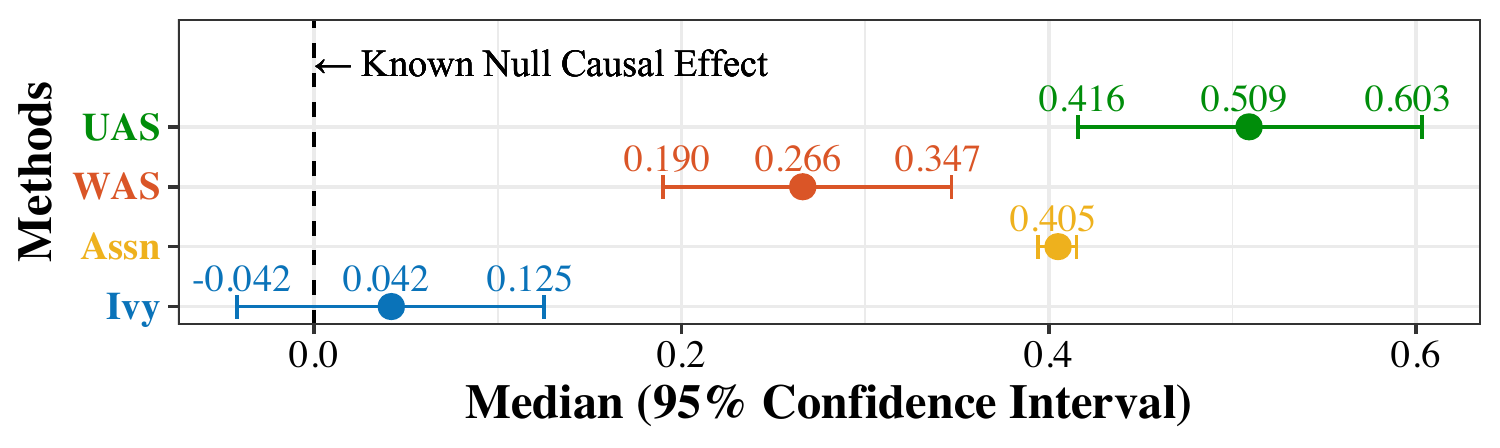}
\caption{Spurious correlation}
\label{fig:pipeline-null}
\end{subfigure}
\begin{subfigure}[b]{0.48\textwidth}
\centering
\includegraphics[scale=0.5]{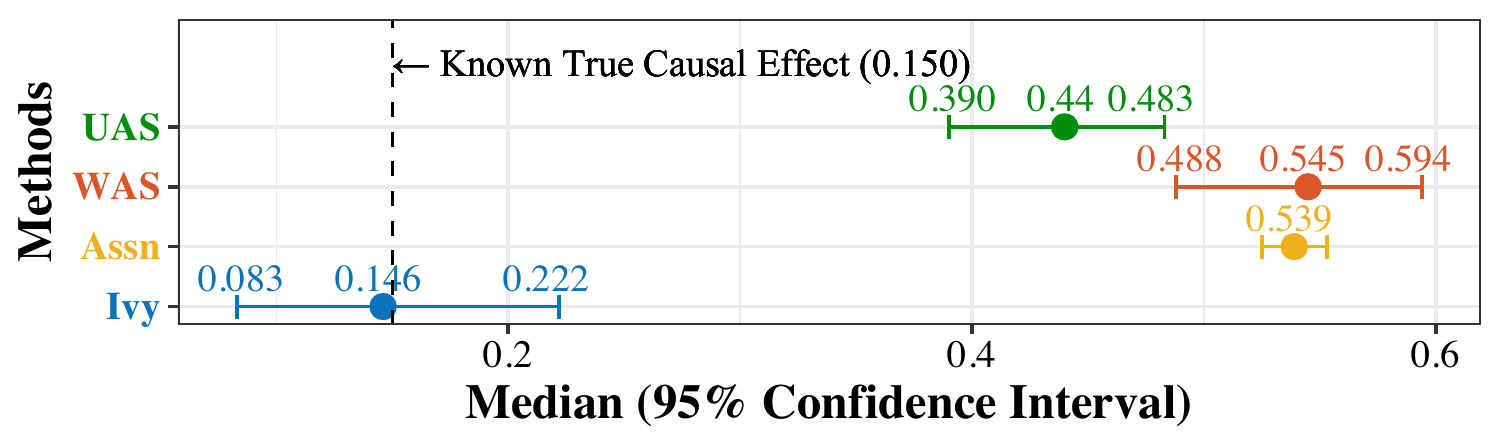}
\caption{True causal effect=$0.150$}
\label{fig:pipeline-true}
\end{subfigure}
\caption{Wald ratios in causal effect estimation using synthetic data. The true causal effects are 0 and 0.15.\vspace{-0.5em}}
\label{fig:pipeline}
\end{figure}

\paragraph{Robustness} We investigate how robust Ivy is to an important violation of our main assumptions (that all the invalid candidates are independent of $z$). Then, the summary $z$ itself may be an invalid IV. We show that Ivy yields a causal estimate that is more robust to this case compared to allele scores. Of course, when the invalidity is sufficiently strong, eventually Ivy also fails to dismiss a spurious correlation (Figure~\ref{fig:invalid-z-result}). 

We use the spurious correlation model in Figure~\ref{fig:invalid-z}. The candidate $w_9$ serves as a confounder between the risk factor and the outcome. Here $z$ is invalid because $z$ is associated with $w_9$, and we increase this association strength (red edge) to force more invalidity. 
We expect Ivy to downweight the influence of $w_9$ while UAS and WAS may not. Indeed, Ivy performs well when $z$ is nearly valid (i.e., nearly independent of $w_9$), and gradually degrades (blue curve), while allele scores immediately struggle. Eventually, increasing the amount of invalidity causes Ivy to fail as well.

\paragraph{Dismissing Spurious Correlations}
Next, we generate synthetic data with no causal effect along with valid and invalid IVs and adding dependencies. The results are in Figure~\ref{fig:pipeline-null}.
Ivy recovers the dependency structure and identifies the invalid candidates. As a result, Ivy can successfully dismiss the spurious correlation by identifying no causal effects (Ivy median: 0.042) while both UAS and WAS fail to do so by yielding estimates that are consistent with the direction of the spurious correlation (UAS median: 0.266, WAS median: 0.509).

\vspace{-0.3em}
\paragraph{Positive Causal Effects} We use synthetic data with positive effects and dependent, partially invalid IV candidates. Experimental results are reported in Figure~\ref{fig:pipeline-true}.  Ivy provides a median estimate (0.146) that is closest to the true effect (0.150) while both UAS (0.440) and WAS (0.545) return median estimates that are biased towards the observational association.

\vspace{-0.3em}
\section{Conclusion}
\vspace{-0.7em}
We introduce Ivy, a framework that synthesizes from IV candidates a summary IV used for downstream causal inference. Through theoretical analysis and empirical studies, we demonstrate the robustness and limitation of Ivy in handling invalidity and dependencies among IV candidates. 
\balance

\section*{Acknowledgements}
The authors thank David Amar, Jason Fries, Sinong Geng, Beliz Gunel, Vitor Hadad, Ramtin Keramati, Allison Koenecke, Sharon Li, Avner May, Tong Mu, Ethan Steinberg, Anna Thomas, Steve Yadlowsky, and Jiaming Zeng 
for helpful conversations and feedback.

We gratefully acknowledge the support of DARPA under Nos. FA87501720095 (D3M), FA86501827865 (SDH), and FA86501827882 (ASED); NIH under No. U54EB020405 (Mobilize), NSF under Nos. CCF1763315 (Beyond Sparsity), CCF1563078 (Volume to Velocity), and 1937301 (RTML);  ONR under No. N000141712266 (Unifying Weak Supervision); the Moore Foundation, NXP, Xilinx, LETI-CEA, Intel, IBM, Microsoft, NEC, Toshiba, TSMC, ARM, Hitachi, BASF, Accenture, Ericsson, Qualcomm, Analog Devices, the Okawa Foundation, American Family Insurance, Google Cloud, Swiss Re, and members of the Stanford DAWN project: Teradata, Facebook, Google, Ant Financial, NEC, VMWare, and Infosys. The U.S. Government is authorized to reproduce and distribute reprints for Governmental purposes notwithstanding any copyright notation thereon. Any opinions, findings, and conclusions or recommendations expressed in this material are those of the authors and do not necessarily reflect the views, policies, or endorsements, either expressed or implied, of DARPA, NIH, ONR, or the U.S. Government.

\bibliography{ivy-paper}
\bibliographystyle{plainnat}

\onecolumn
\appendix
\renewcommand\thefigure{A.\arabic{figure}}
\renewcommand\thetable{A.\arabic{table}}
\section*{Appendix}
\setcounter{figure}{0}
\setcounter{table}{0}

The appendix is organized as follows. First, we discuss related work (Section~\ref{sec:related-work}). Next, we provide theoretical details, including the proofs of our main results, and in particular Theorem~\ref{thm:ci_bound_1} and its generalization. We also provide an analysis of the statistical power for our technique combined with the Wald estimator (Section~\ref{sec:theory-appendix}). Finally, we provide additional experimental details (Section~\ref{sec:extended-experiments}).

\section{Related Work}
\label{sec:related-work}

\subsection{Overview}
Pearl's seminal work on causality \citep{PearlBookOld} defines the causal inference paradigm, including the notion of \emph{intervention}. For example, if an external force was to change the air pressure, a barometer's reading would change, while if we were to change the barometer reading, the pressure would remain the same. Thus, we can define causal relationships via interventions. 

Causal relationships can be encoded in directed acyclic graphs (DAGs), resembling encoding distributions with graphical models. However, causal graphs also carry an additional family of distributions induced by performing interventions. Learning such graphs is a major area in causal inference \citep{Heckerman95, Ellis08}. Recent work establishes nearly optimal algorithms for learning a causal graph with the smallest number of samples and interventions \citep{Kocaoglu17, Acharya18}. The equivalence classes of causal DAGs is explored in \cite{Yang18}. Causal models often include both discrete and continuous variables, motivating research into mixed model structure learning \citep{Lee13}. Learning a network across multiple domains is considered in \cite{Ghassami18}. Although none of these works directly fit our paradigm, structure learning is an important part of our approach as well.

When performing interventions is not possible and we must attempt to estimate causal effects from observational data, instrumental variable approaches are an option. The concept of instrumental variables date back to the 1920s \citep{wright1928tariff}. The traditional approach to IV estimators relies on structural models. For example, linear relationships between the instrumental, treatment, and effect variables inspired the two phase least-squares model (2SLS) \citep{Angrist96}. These types of models can be replaced by deep neural nets, as in \cite{Hartford17}. The presence of invalid instrumental variables motivates the line of research in robust IV methods \citep{pearl1995testability, bonet2001instrumentality, han2008detecting,   bowden2016consistent, kang2016instrumental, sharma2018necessary, Windmeijer18}. Another concern is related to instrumental variables that are weakly correlated with the risk factor \citep{Bound95}. 

Our work is focused on observational rather than interventional approaches. However, unlike the previously mentioned works, we do not examine a particular causal effect estimator, but rather seek to develop a way to synthesize a latent strong summary IV from multiple IV candidates. This is the synthesis phase of a two-phase methods. Our approach, in particular, focuses on predicting a latent variable (the summary). Other causal inference techniques that make use of latent variables related to our work include matrix completion algorithms for observational data with missing entries \citep{Athey18} and algorithms that handle multiple causes \citep{wang2018blessings,wang2019multiple}.

Mendelian randomization \citep{Burgess15} is a popular approach to perform causal inference among clinical variables using genetic variants such as single-nucleotide polymorphisms (SNPs) as instrumental variables. Since SNPs are determined for each individual randomly at conception, it offers a natural randomization among different individuals. Furthermore, since DNA encoding will influence downstream biological and clinical outcomes but not the other way around (the central dogma of molecular biology), using SNPs as instrumental variables is also an effective way to avoid reverse causation. The use of SNPs as IVs also comes with its challenges, such as that a SNP can be correlated with multiple clinical outcomes (a.k.a.\ pleiotropy), and that a SNP can demonstrate weak effects towards downstream outcomes. Moreover, SNPs can potentially be invalid IVs. Allele scores are an effective approach to combine SNPs into a summary variable. The summary variable can then be used downstream in a causal effect estimator such as a Wald estimator in order to produce a causal effect estimate. Allele scores  are an effective approach to handle SNPs that are weakly associated with the risk factor (a.k.a.\ weak IVs). However, allele scores are sensitive to the presence of SNPs that are not valid IVs.

\begin{table}[t]
\centering
\begin{tabularx}{\linewidth}{p{0.16\linewidth}p{0.12\linewidth}p{0.12\linewidth}p{0.11\linewidth}p{0.1\linewidth}p{0.1\linewidth}p{0.12\linewidth}}
\hline
Methods & Candidate \newline Dependency &  Breakdown Level & Data \newline Generation & Invalidity \newline Allowed & Sample \newline Complexity & Two \newline Phase\\
\hline
Ivy (us) & Full-rank \newline Candidates & $50$--$100\%$ \footnotemark & Ising Model & Some ER, Some UC & Non-Asy & \checkmark\\
\hline
Two Stage \newline Least Square & Full-rank \newline Candidates & $0\%$ & Linear \newline Model & None & Asy & \xmark\\
\hline
Inverse Variance Weighted (IVW) & Independent & $0\%$ & Linear \newline Model & None &  Asy & \xmark\\
\hline
UAS (Binary) & Conditionally Independent & $0\%$ & Naive \newline Bayes & None &  Asy & \checkmark\\
\hline
WAS (Binary) & Conditionally Independent & $0\%$ & Naive \newline Bayes  & None &  Asy & \checkmark\\
\hline
\texttt{sisVIVE} \newline \small{\citep{kang2016instrumental}} & Full-rank \newline Candidates  & $<50\%$  & Linear Model  & ER, UC & Non-Asy & \xmark\\
\hline
Simple Median \newline {\scriptsize\citep{bowden2016consistent}} & Independent & $<50\%$ & Linear Model & ER, UC & Asy & \xmark\\
\hline
Weighted Median \newline
{\scriptsize\citep{bowden2016consistent}} & Independent & $<50\%$ \newline Information & Linear Model & ER, UC & Asy & \xmark\\
\hline
Egger Regression \newline 
{\scriptsize \citep{bowden2015mendelian}} & Independent & $100\%$ & Linear Model & Some ER & Asy & \xmark\\
\hline
\cite{Windmeijer18} & Full-rank \newline Candidates & $<50\%$ & Linear Model & ER, UC & Asy & \xmark\\
\hline
$L_1$-GMM \newline \citep{han2008detecting}  & Full-rank \newline Candidates  & $<50\%$ & Linear Model & ER, UC & Asy & \xmark\\
\hline
Deep IV \newline 
{\scriptsize \cite{Hartford17}} & Full-rank \newline Candidates & $0\%$ & Nonlinear & None & None & \xmark\\
\hline
{ \cite{bennett2019deep}} & Full-rank \newline Candidate & $0\%$ & Nonlinear & None & None &  \xmark\\
\hline
\end{tabularx}
\caption{Comparison of assumptions among IV methods. Relevance of the IV candidates is assumed. ER: exclusion restriction; UC: unconfoundedness}
\label{tab:assumptions}
\end{table}

\subsection{Method  Comparison}
Ivy relates to a variety of IV methods in the literature. We describe a number of these and compare the assumptions and use of these methods. This taxonomy can be found in  Table~\ref{tab:assumptions}. 

\footnotetext{See Section~\ref{sec:ci-proof} for a discussion on scenarios when the breakdown level can be more than 50\%.}

In Table~\ref{tab:assumptions}, candidate dependency refers to the properties that the collection of IV candidates must satisfy. Requiring independence is the strongest, most restrictive property, while being only conditionally independent is slightly weaker. Weaker still is only requiring that the data matrix (where each row contains the samples from a particular candidate) is full-rank, which precludes identical copies of candidates.

Breakdown level represents the percentage of invalid IVs allowed before a method fails to return a reliable estimate.  Note that methods that have a 0\% level require \emph{all} valid candidates. Invalidity allowed represents the type of invalid IVs that a method can recognize (i.e., invalid because they do not satisfy exclusion restriction (ER) or unconfoundedness (UC)). Sample complexity describes whether asymptotic (Asy) or non-asymptotic (Non-Asy) estimation guarantees, if any, are known in the literature. Finally, the ``two phase'' property describes whether the method aims to generate a summary variable with a synthesis phase followed by an estimation phase, or whether it is a direct estimation.

The overall goal of Ivy is to perform well in scenarios where there is a less restrictive assumption on dependencies and invalidity, often at the same time. Thus Ivy can handle correlated candidates and a number of invalid IVs simultaneously, which existing methods struggle with. We note, however, that Ivy handles binary variables, while other methods can sometimes handle both categorical and continuous variables. We also seek to provide theoretical guarantees in finite-sample settings, rather than just asymptotic consistency.

The median methods (simple and weighted) have a higher breakdown level compared to allele scores. They are designed to deal with invalidity (naturally, we do not know which candidates are invalid) by producing median measurements that filter out the invalid candidates. However, they also require independence among the candidates \citep{bowden2016consistent}. By contrast, Ivy can handle dependencies. There are a few other differences, including the fact that the median methods have asymptotic guarantees (instead of finite sample bounds). 

An important point is that the goal of Ivy, as a method for the synthesis phase, is to provide an IV of higher quality. This strong IV can then be used downstream in a causal estimator, or in another IV method. Therefore, we stress that Ivy is complementary to existing IV methods in the literature instead of necessarily being a competing alternative. For example, we could use Ivy to provide additional candidates to be used in the simple median method, to provide weights to be used in the weighted median method,  or to plug it into a deep learning-based estimator like DeepIV \citep{Hartford17}.

\paragraph{Allele Scores} Since they synthesize a summary variable of genetic contribution towards elevating the risk factor, two-phase methods including allele scores, e.g., UAS (unweighted allele score) and WAS (weighted allele score), are most similar in spirit to Ivy.  Specifically, when the risk factor is binary, \cite{sebastiani2012naive} point out the equivalence of the decision rules between a Naive Bayes classifier and an allele score whose weights are derived from univariate logistic regressions (as in WAS). Since Ivy with conditional independence can be viewed as a Naive Bayes classifier with a hidden label, it shares similar statistical dependency assumptions with allele scores (Figure~\ref{fig:ivy-allele-score}). The difference is that allele scores may use the observed risk factor as the label of the classification, while Ivy assumes a valid hidden IV. When the conditional independence assumption is lifted, Ivy generalizes beyond allele scores with additional potential to handle dependencies and certain types of invalidity among candidates. 

Ivy and allele scores explicitly construct a summary IV, while other methods directly obtain a causal estimate, and are thus not modular (right-most column of Table~\ref{tab:assumptions}). Some key differences compared to allele score methods: Ivy has finite sample bounds, not just asymptotic results, has a weaker assumption for dependencies, and has a higher invalidity breakdown level. We note that our breakdown level assumption is 50\% by default, but under certain scenarios, we can handle even more invalid IVs  (Section~\ref{sec:ci-proof}).

\paragraph{Deconfounder \citep{wang2018blessings}} \cite{wang2018blessings} proposed the deconfounder, a causal inference framework that estimates causal effects of multiple causes from observational data under the assumption that there is no unobserved single-cause confounder (a.k.a.\ single ignorability). The deconfounder first learns a set of latent confounders from the data using latent factor models. These confounders are then used as surrogates to the actual confounders in the data, along with the multiple cause variables, to be fed into the adjustment formula to achieve causal effect estimation. \cite{wang2018blessings} show that the residue after adjusting for the confounders and the treatment variables can be used as instrumental variables.

While Ivy also makes use of latent variables, the latent variables are used to act as instrumental variables. This is different from the deconfounder, where latent variables are used as confounders. Nonetheless, here we offer an explanation of how Ivy may be interpreted in the deconfounder framework under certain circumstances. In MR, the SNPs used as IV candidates are usually only associated with the risk factors, instead of causal to the risk factors. Suppose that all the SNPs are valid IV candidates, and many of the SNPs are correlated with each other. Such correlations can be potentially explained by a causal yet unobserved genetic variant \citep{Burgess15} modeled as a hidden variable. In the deconfounder framework, these hidden variables are viewed as confounders. Since all the SNPs are valid, Ivy in this scenario uses a latent variable to model a summary IV. Unlike deconfounder that seeks to estimate the causal relationship between the SNPs and the risk factor, Ivy seeks to synthesize a summary IV that is better associated with the risk factor so as to provide a stronger IV to infer causation between the risk factor and the outcome.

\begin{figure}[t]
\begin{subfigure}[b]{0.5\textwidth}
\centering
\begin {tikzpicture}[-latex ,auto ,node distance
=1.0cm and 1.1cm , on grid, semithick ,
state/.style ={ circle, draw, minimum width=0.5cm}, scale=0.50]
\node[state, draw] (C) [dashed] {\large $c$};
\node[state, draw] (Y) [below right=of C] {\large $y$};
\node[state] (X) [below =of C] {\large $x$};
\node[state] (W1) [above left=of X] {\tiny $w_1$};
\node[state] (W2) [left=of X] {\tiny $w_2$};
\node[state] (W3) [below left=of X] {\tiny $w_3$};
\path (X) edge[-] (W1);
\path (X) edge[-] (W2);
\path (X) edge[-] (W3);
\path (C) edge (X);
\path (C) edge (Y);
\path[dashed] (X) edge (Y);
\end{tikzpicture}
\caption{Allele Score}
\label{fig:allele-score}
\end{subfigure}~
\begin{subfigure}[b]{0.5\textwidth}
\centering
\begin {tikzpicture}[-latex ,auto ,node distance
=1.0cm and 1.1cm , on grid, semithick ,
state/.style ={ circle, draw, minimum width=0.5cm}, scale=0.50]
\node[state, draw] (C) [dashed] {\large $c$};
\node[state, draw] (Y) [below right=of C] {\large $y$};
\node[state] (X) [below =of C] {\large $x$};
\node[state] (Z) [left=of X, dashed] {\large $z$};
\node[state] (W1) [above left=of Z] {\tiny $w_1$};
\node[state] (W2) [left=of Z] {\tiny $w_2$};
\node[state] (W3) [below=of W2] {\tiny $w_3$};
\path (Z) edge[-] (W1);
\path (Z) edge[-] (W2);
\path (Z) edge[-] (W3);
\path (Z) edge (X);
\path (C) edge (X);
\path (C) edge (Y);
\path[dashed] (X) edge (Y);
\end{tikzpicture}
\caption{Ivy (Conditionally Independent)}
\label{fig:ivy-ci}
\end{subfigure}
\caption{Equivalence of allele score and Ivy to a Naive Bayes classifier. Note that in Figure~\ref{fig:ivy-ci}, when $z$ is perfectly predictive of $x$, it become equivalent to Figure~\ref{fig:allele-score}.}
\label{fig:ivy-allele-score}
\end{figure}
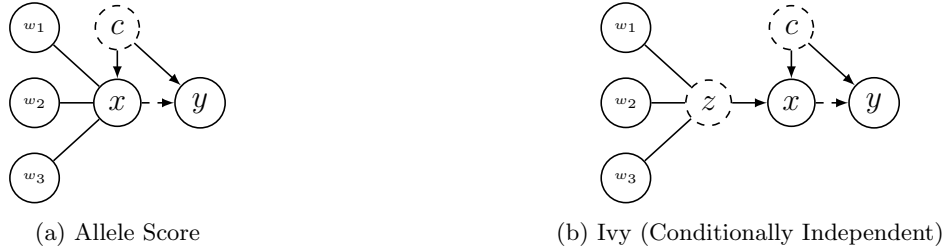

\section{Extended Theoretical Results}
\label{sec:theory-appendix}
First, we provide an additional algorithm that explicitly computes the correlations between the candidate IVs (Section~\ref{sec:appendix-additional-alg}). Next, we discuss assumptions (Section~\ref{sec:assumptions-appendix}) and identifiability (Section~\ref{sec:appendix-identifiability}). We then provide the proof of our main results, including an extended version of Theorem~\ref{thm:ci_bound_1} (Section~\ref{sec:ci-proof} -  \ref{sec:lower-bound-causal-effect-appendix}). In Section~\ref{sec:power-appendix}, we give a bound on the error in statistical \emph{power}  when using our approach (compared to having access to the true $z$). Afterwards, we detail why the conditional independence assumptions in our Ising model enable us to produce independent accuracies, a key component of our approach  (Section~\ref{sec:cond-indep-appendix}). Finally, we give further examples that our algorithm does not tackle, demonstrating the limitation of our technique (Section~\ref{sec:counterexample-appendix}).

\subsection{Additional Algorithm Details}
\label{sec:appendix-additional-alg}
We provide some additional information on our algorithms. We start with some notation for our structure learning component, which applies the one in \cite{Varma19}.
Let us write for our model's covariance and inverse covariance matrices,
\begin{align*}
\Sigma :=
	\begin{bmatrix}
		\Sigma_O & \Sigma_{Oz} \\
		\Sigma_{Oz}^T & \Sigma_z
	\end{bmatrix}
\text{  and  } 
	\Sigma^{-1} 
	&:=
	K
	=
	\begin{bmatrix}
		K_O & K_{Oz} \\
		K_{Oz}^T & K_z
	\end{bmatrix}.
\end{align*}

Here $O$ is the set of observed IV candidates and $z$ is the valid, but unobserved, summary IV. The key idea behind the algorithm is that in the inverse covariance matrix $K$, a 0 entry $K_{ij}$ indicates that there is no dependency between $w_i$ and $w_j$ \citep{loh13}. Therefore, if we had access to $K$, we would be able to directly read off the graph. Unfortunately, this full $K$ includes the unobserved latent IV
$z$, so we cannot observe the full covariance matrix $\Sigma$ and invert it to estimate $K$. We only have access to an estimate of $\Sigma_O$, the matrix given by the observed candidates. If we directly invert $\Sigma_O$, we do not obtain the block $K_O$, but rather this block corrupted by adding an additional low-rank matrix that is non-sparse and hides the graph structure.

Structure learning attempts to break up the $\Sigma_O^{-1}$ matrix, which we can estimate, into a sparse matrix $S$ that we use to approximate $K_O$, and a rank one symmetric matrix $L = \ell \ell^T$. Then, a simple transformation yields the estimated covariance between $z$ and each candidate; this enables us to read off both the valid IVs (those whose covariance with $z$ is larger than zero) and the dependencies between the invalid IVs from $S$. 

When estimating $\hat{O}$, we compute $\frac{1}{n} \sum_{i=1}^n w^{(i)}_{\hat{V}}(w^{(i)}_{\hat{V}})^T$. In practice, one may also compute $\frac{1}{n-1} \sum_{i=1}^n w^{(i)}_{\hat{V}}(w^{(i)}_{\hat{V}})^T$ (as this is an unbiased estimator) if there are very few samples; asymptotically, either method yields the same error bounds. Finally, note that we clip our predicted $\hat{\mu}$ to lie within $[-1, +1]$, since $z$ and the $w_i$'s (and thus the $w_iz$'s) are all in $\{-1, +1\}$.

\subsection{Assumptions}
\label{sec:assumptions-appendix}
We summarize the assumptions that we make in order to provide theoretical guarantees for the performance of Ivy. We discuss our assumptions made in the synthesis phase and estimation phase, respectively.

\paragraph{Synthesis Phase} We detail the assumptions  used in the estimation phase. First, the majority of IV candidates are valid IVs, and for the invalid candidates $(i \not\in V$), $w_i \perp z$. Second, the joint distribution of $w$'s and $z$ follows an Ising model. Here, $(V,E)$ consists of valid candidates and their edges, while $\bar{V} := \{1, \ldots, m\} \setminus V$ and $E'$ are the corresponding invalid candidates and their edges:
\begin{equation*}
\prob(w, z) = \frac{1}{\mathcal{Z}} \exp  ( \theta_z^* z + \Sigma_{i\in V} \theta_{i}^* w_i z  +  \Sigma_{(i,j)\in E} \theta_{ij}^* w_i w_{j} + \Sigma_{(i,j) \in E'} \theta'_{ij} w_i w_j )  .
\end{equation*}
We assume that the graph $G$ above satisfies the singleton separator set property. That is, the intersections of maximal cliques among the nodes in $G$ are always of cardinality at most one. Another way to state this is to say that candidates form components that intersect (at most) in the latent variable. We write $d$ for the maximum degree of a candidate dependency. 

Next, for each valid candidate $w_i$ there are at least two others that are independent of $w_i$ and each other conditioned on $z$. 
We further assume that valid IV candidates agree with $z$ more often than not on average.

Since we use a variant of the procedure in \cite{Varma19} as an instance of the structure learning algorithm, we review the assumptions made in \cite{Varma19}. These conditions are standard and were originally introduced in \citet{Chandrasekaran12} and \citet{Wu17}. Let $h_X(Y) := \frac{1}{2}(XY + YX).$ We write $\mathcal{P}_S$ for orthogonal projection on subspace $S$.  Let  $\alpha_{\Omega} :=   \min_{M \in \Omega, \|M\|_{\infty} = 1}    \|\mathcal{P}_{\Omega} h_{\Sigma_O}(M)\|_\infty$,  $\delta_{\Omega} :=  \min_{M \in \Omega, \|M\|_{\infty} = 1}    \|\mathcal{P}_{\Omega^\perp} h_{\Sigma_O}(M)\|_\infty$, 
$\alpha_{T} :=    \min_{M \in T, \|M\| = 1}    \|\mathcal{P}_{T} h_{\Sigma_O}(M)\|$, 
$\delta_{T} :=    \min_{M \in T, \|M\| = 1}     \|\mathcal{P}_{T^{\perp}} h_{\Sigma_O}(M)\|$, 
$\beta_{T} :=    \max_{M \in T, \|M\|_{\infty} = 1}    \|h_{\Sigma_O}(M)\|_\infty$, 
$\beta_{\Omega} := \max_{M \in \Omega, \|M\| = 1}     \|h_{\Sigma_O}(M)\|$.
Next, set 
$\alpha := \min\{\alpha_\Omega, \alpha_T \}$, $\beta := \max\{ \beta_T, \beta_\Omega \}$, and $\delta := \max\{\delta_\Omega, \delta_T\}$.

With this notation, we require that there exists $\nu \in (0,1/2)$ with $\delta/\alpha < 1 - 2\nu$, and   $\mu(\Omega) \xi(T) \leq \frac{1}{2} \left( \frac{\nu \alpha}{(2-\nu)\beta} \right)^2$. 

\paragraph{Estimation Phase} We enumerate some standard regularity conditions with respect to univariate logistic regressions in order to characterize the error induced by the Wald estimator. Specifically, let $\ell (x; \theta)$ be the negative log-likelihood function of the univariate logistic regression parameterized by $\theta \in \Theta$. Let $\mathcal{I}(\theta) := \mathbb{E} \left[\frac{\partial^2 \ell (x; \theta)}{\partial \theta^2}\right]$ be the corresponding Fisher information matrix. We assume that for any given $\theta \in \Theta$, there exists an unbiased estimator $\check{\theta}$ of $\theta$ that is a function of some number $n$ of independent samples $\curly{x^{(i)}}_{i=1}^n$ drawn from $\mathcal{D}$ such that, for some constant $\lambda$, $\text{Cov} (\check{\theta}) \preceq \frac{I}{2 \lambda n}$.  We further assume that for some $\kappa \in (0,1)$, we have that $\max\{\lvert \hat{\beta}_{\hat{z}y} - \beta_{zy}^* \rvert, \lvert \hat{\beta}_{\hat{z}x} - \beta_{zx}^* \rvert\} \le \kappa \beta_{zx}^*$, where $\beta^*_{zx}$ and $\beta^*_{zy}$ be the population-level coefficients of $z$ from the logistic regressions to predict $x$ and $y$ under $\mathcal{D}$, and $\hat{\beta}_{\hat{z}y},\hat{\beta}_{\hat{z}x}$ the corresponding regression coefficients of $\hat{z}$.

\subsection{Identifiability of Causal Effects}
\label{sec:appendix-identifiability}
Recall that we follow a two-phase approach that consists of a synthesis phase followed by an estimation phase. If identifiability can be achieved in both phases, overall identifiability can be achieved. The assumptions made in the synthesis phase ensure the identifiability of the synthesis phase. Afterwards, we can use existing standard conditions for identifiability of the estimation phase \citep{balke1997bounds, burgess2018mendelian, swanson2018partial, Amour} to identify causal effects.

\subsection{Extension and Proof of Theorem~\ref{thm:ci_bound_1}}
\label{sec:ci-proof}
First, we give some additional details on our approach. Afterwards, we give an extended version of Theorem~\ref{thm:ci_bound_1} (Theorem~\ref{thm:extended}); this more general result subsumes Theorem~\ref{thm:ci_bound_1}. 

We need some additional notation. 
Recall that $V \subseteq \{1, \ldots, m\}$ is the subset of valid candidates. We write
\begin{equation*}
o_1(x) := \lfloor x^2 \rfloor  \qquad \text{and}\qquad o_2(x) := {\sum_{(i,j) \in E} \mathbb{I} \{(K_O)_{ij} \leq x\}}.
\end{equation*}
The second function involves the order statistics among certain entries in the inverse covariance matrix; its use is explained below.
Note that $o_1(x) \rightarrow 0$ as $x \rightarrow 0$; in fact, it is 0 once $|x| < 1$. Similarly, $o_2(x) \rightarrow 0$: it is 0 once $x$ is below the minimal entry in the matrix. Next, to clean up the notation, we use constants $c_3$ and $c_4$,  defined in \cite{Varma19}; these are a function of the maximal degree $d$ and upper bounds on the conditioning of covariance matrix, which we assume are fixed. Recall that our thresholds are $T_1, T_2$ in Algorithm~\ref{alg:sl1}. We write
\begin{equation*}
i_S := o_1\left(\frac{c_3 m^2}{T_1 \sqrt{n}} \right) \qquad  \text{  and  } \qquad  e_S := o_2\left(T_2 + c_4 \sqrt{\tfrac{m}{n}} \right).
\end{equation*}
This notation indicates that $i_S$ is the number of invalid IVs and $e_S$ the number of missed edges (among the valid candidates) after structure learning.

We define $q_{\max}$ to be the largest-magnitude entry of $\hat{q}$, and $R_{\min} > 0$ to be the smallest entry of $\mathbb{E}[w_Vw_V^T]$. We write $M^\dag$ for the pseudoinverse of the matrix $M:=M_{\hat{\Omega}}$. Note that in Theorem~\ref{thm:ci_bound_1} we have $M=M_{\Omega}$ because we recover the true $\Omega$, as we show below. Let $r_M$ be the least-squares residual for  $M \hat{\ell} = \hat{q}$. We let $S$ be a matrix (not to be confused with the sparse matrix for structure learning; it shall be clear from the context) so that $SM=M_{\Omega}$ is the corrected form of $M$, removing dependencies and invalid candidates. We call $S$ the \emph{row selection} matrix. Let the SVD of $M$ be $M = U_M \Sigma_{M} V_M^\top$. We let $\rho_{SU_M}$ be the smallest-magnitude singular value of $SU_M$. 
Finally, we let $u_{\max}$ to be the largest norm of a row of $U_M$.

Our main result in this section is Theorem~\ref{thm:extended}.  

\begin{theorem}
\label{thm:extended}
Let $\hat{\mu}$ be the result of Algorithm~\ref{alg:ivy} run on $n$ samples of $m$ IV candidates, for $n > n_0$ for some constant $n_0$. Here, $s$ of the $m$ candidates are invalid and $p = m-s$ are valid. Set $\mu^*$ to be the true parameters. Then, with probability at least $1-1/m$,
\begin{align}
\mathbb{E}[\| |\hat{\mu}| - |\mu^*| \| ]  &\leq 
\frac{16 (i_S+p)^{\frac{5}{2}}}{R_{\min}}  \|M^\dagger\| \sqrt{\frac{2\pi}{n}}\nonumber \\
& + m^{3/2} \sigma_{\min}^{-1}(M)  \left( (i_S(i_S+p) + e_S)u_{\max} + \rho_{SU_M}^{-1} - \rho_{SU_M} \right) \norm{U_M^\perp (U_M^\perp)^\top} q_{\max}.
\label{ew:big}
\end{align}
Additionally, if we bypass structure learning and run the latter part of the algorithm, we obtain the following explicit result, where $s$ is the number of invalid IVs and $e$ is the number of dependencies among the valid sources. 
\begin{align*}
&\mathbb{E}[\| |\hat{\mu}| - |\mu^*| \| ]\leq  \\
& \frac{1}{\sqrt{s+p-2}} \left( 
\frac{16 (s+p)^{\frac{5}{2}}}{R_{\min}}  \sqrt{\frac{2\pi}{n}} + m^{\frac{3}{2}}  \left( \frac{\sqrt{2}(s(s+p)+ e)}{\sqrt{s+p-2}}+ \rho_{SU_M}^{-1} - \rho_{SU_M} \right) \norm{U_M^\perp (U_M^\perp)^\top} q_{\max}\right).
\end{align*}
\end{theorem}

Note that the norm above is taken implicitly over the parameters among variables in joint set $\hat{V} \cap V$ (which goes to $V$ once $n$ is large enough); these are the parameters in common between the recovered set of candidates $\hat{V}$ and the valid candidates $V$. Of course, if this set becomes too small, e.g., below three variables, we cannot even recover. Fortunately, we know the rate at which $i_S$ goes to 0. Below, we implicitly assume that $|\hat{V} \cap V| \geq 3$.
Before we give the proof, it is worth commenting on Theorem~\ref{thm:extended}.

When $m$ and $n$ are as large as prescribed, we have that with the desired probability $1 - 1/m$ that the correct structure is recovered, in which case the $o_1$ and $o_2$ functions defined are equal to 0, and thus $i_S = 0$ and $e_s = 0$. With this, the first term inside the sum of the second term is 0. We also have that $S$ (the row selection matrix) is the identity, and $\rho_{SU_M} = 1$, and the entire right-hand side goes to 0. In the first term, $i_S + p = p$. Finally, under these assumptions on $m$ and $n$ we can also recover the signs, yielding Theorem~\ref{thm:ci_bound_1}.

More generally we can think of the left-hand summand above as being an estimation error, which goes to 0 in $n$, and the right-hand term as a penalty for misspecification. Here we upper bound this term, in order to measure our robustness to such misspecification, but using the same argument we made in Section~\ref{sec:estimation-error}, we can show that it is always positive regardless of $n$, so long as $i$ and $e$ are positive.

We also briefly comment on the difference between the two cases above. In the top case, where we use structure learning, the obtained $M$ matrix has a potentially complicated structure. The $M$ obtained from assuming conditional independence for all pairs of candidates, which is all we can do without knowledge of the graph, has a simple structure that enables us to compute terms like $\sigma^{-1}_{\min}$ explicitly. 

Before we start, we give the following simple result.

\begin{lemma}
\label{lem:conc}
If we estimate $\hat{O}$ from samples $w^{(1)}, w^{(2)}, \ldots, w^{(n)}$ by $\hat{O} := \frac{1}{n} \sum_{i=1}^{n} w^{(i)}w^{(i)^T}$, we can bound $\mathbb{E} \left[\|\hat{O} - O^*\| \right]$ as
\[\mathbb{E} \left[\|\hat{O} - O^*\|  \right] \leq m^2 \sqrt{\frac{32 \pi}{n}}.\]
\end{lemma}
\begin{proof}
We use a concentration bound to obtain $\|\hat{O} - O^*\|$. We shall use the matrix Hoeffding inequality \citep{Tropp11}. It states that for any finite sequence of independent random symmetric $m \times m$ matrices $\{X_k\}$ that are centered (mean-zero), and satisfy $X_k^2 \preceq A_k^2$, then 
\begin{equation}
\label{eq:tropp}
\prob \left( \|\sum_k X_k\| \geq t \right) \leq m \exp \left( \frac{-t^2}{8\sigma^2} \right),
\end{equation}
where $\sigma^2 := \| \sum_k A_k^2 \|$.

To apply the result, we take $X_k = \frac{1}{n} (\hat{O} - w^k (w^k)^T)$ for $1 \le k \le n$.
Clearly, $\mathbb{E}[X_k] = 0$.
The $X_k$'s are also clearly symmetric and mutually independent. We now argue that suitable $A_k$ matrices exist. First, it is easy to see that, from Cauchy-Schwartz, for any two vectors $v_1$ and $v_2$, $v_1^T( \|v_2\|^2I - v_2v_2^T) \geq 0$, so $\|v_2\|^2I \succeq v_2v_2^T$. Thus, as each vector $w^k \in \{-1,+1\}^m$, we have that 
\[m^2 I = \|w^k\|^4 I \succeq \|w^k\|^2 w^k(w^k)^T = (w^k (w^k)^T)^2.\]

Note that $w^k (w^k)^T \succeq 0$ and $\hat{O} \succeq 0$, so $(w^k (w^k)^T + \hat{O})^2 \succeq 0$. This yields
\begin{align*}
(nX_k)^2 &= (w^k (w^k)^T - \hat{O})^2  \\
&\preceq (w^k(w^k)^T - \hat{O})^2 + (w^k(w^k)^T +\hat{O})^2 \\
&= 2((w^k(w^k)^T)^2 + \hat{O}^2) \\
&\preceq 2(m^2I + \hat{O}^2).
\end{align*}
Setting $A_k^2 = \frac{2}{n^2}(m^2I + \hat{O}^2)$, we have that $\|\hat{O}^2\| \leq m^2$, and then that $\sigma^2 = \norm{\sum_k A_k^2} \le
\sum_k \norm{A_k^2} \le
\frac{2}{n^2} \sum_k (\norm{m^2 I} + \norm{\hat{O}^2}) =
\frac{2}{n} (m^2 + \|\hat{O}^2\|) \leq 4m^2/n$, so applying the bound \eqref{eq:tropp}, we get
\begin{align}
\prob \left(\|\hat{O} - O^*\| \geq t \right) \leq m \exp \left( \frac{-nt^2}{32m^2} \right).
\label{eq:concbound}
\end{align}

Next, we can integrate the result to obtain
\[
\mathbb{E} \left[\|\hat{O} - O^*\|  \right] \leq m^2 \sqrt{\frac{32 \pi}{n}}.
\]
\end{proof}

Now we are ready for the proof of Theorem~\ref{thm:extended}. We start with a lemma that tackles the structure learning component of the algorithm:

\begin{lemma}
Run Algorithm~\ref{alg:sl1} on $n$ samples of $m$ IV candidates, where $s$ of the $m$ candidates are invalid and $p = m-s$ are valid. Suppose the assumptions detailed above are met. Let $\hat{G} = (\hat{V}, \hat{E})$ be the resulting graph. Then, with probability at least $1-1/m$,
\[
|\hat{V} \setminus V| \leq o_1\left(\frac{c_3 m^2}{T_1 \sqrt{n}} \right) \qquad \text{ and } \qquad |E \setminus \hat{E}| \leq o_2\left(T_2 + c_4 \sqrt{\tfrac{m}{n}} \right).
\]
That is, $\hat{G}$ contains at most $i_S$ invalid candidates and has at most $e_S$ missing edges among valid candidates.
\label{lem:sllemma}
\end{lemma}
This result characterizes the performance of the structure learning component. It tells us how many invalid IVs we may inadvertently be using in the estimation phase (due to noise) and how many such edges we may include. The proof is a simple modification of the result in \cite{Varma19}.

\begin{proof}
First, we note a difference from the result in \cite{Varma19} and our application of it: in that work, all of the nodes are connected to the latent node. In our version, the invalid candidates are not connected to any of the valid candidates or the latent variable. 

This ensures that in $\Sigma\ell$, the terms corresponding to the valid candidates are zero, which enables us to estimate the set of valid candidates $\hat{V}$. The result in \cite{Varma19} still holds in this setting; it does not require that all of the observed variables are connected to the latent variable. Next, we need to map our assumptions into those of \cite{Varma19}. The main requirement here is the singleton separator set assumption; for us, this is exactly equivalent to requiring that candidates that are dependent are all mutually dependent. The other assumptions are directly borrowed.

Next, note that in \cite{Varma19}, Theorem 1 is stated in terms of the number of samples sufficient to recover the structure exactly; this is done by driving the error below the smallest magnitude of the sparse component encoding the structure. The number of samples $n$ is determined by the smallest error sufficient to do this. That is, the authors obtain
\[n > c_1 d^2  m,\]
where we set $c_1$ to be the term (in the notation of \cite{Varma19})
\[ \left[ \frac{6 c_2  \beta(3-2\nu)(2-\nu)\psi_1}{\nu \alpha^2 \psi_m} \max \left\{\frac{1}{\psi_m}, \frac{\gamma}{K_{O,\min}}, \sigma^{-1}\right\} \right]^2.\]
In fact, a stronger version is possible where $m$ in the preceding expression on sample complexity can be reduced to $m^\tau$ for some $\tau \leq 1$.

Instead, we use the version of the result that computes the number of errors as for a particular number of samples via the step functions $i_S$ and $e_S$. Note that we consider both $\hat{S}$ (encoding the edges) and $\hat{\Sigma}\hat{\ell}$ (encoding the valid/invalid candidates). We have, using the proof of Theorem 1 in \cite{Varma19}, itself following \cite{Wu17} (top of Step 2 in the proof of Theorem 4.1) for some constants $c_4, c_5$ that are a function of $c_1$ above, that
\[\|\hat{L} - L\| \leq c_5 \sqrt{\frac{m}{n}},\]
and
\[\|\hat{K}_O - K_O\|_{\infty} \leq c_4 \sqrt{\frac{m}{n}},\]
where the $\inf$-norm here refers to the norm taken over the vectorized version of the matrix.

\paragraph{Missing edges}
It is easier to deal with the second term. Which edges will we fail to recognize among the valid candidates when running Algorithm~\ref{alg:sl1}? Precisely those entries of $K_O$ where $\hat{K}_O$ is no larger than our threshold than $c_4 \sqrt{\frac{m}{n}}$ (as, from the bound above, the gap cannot be any larger than this). Then, recalling that $o_2(x) = {\sum_{(i,j) \in E} \mathbb{I} \{(K_O)_{ij} \leq x\}}$, we have that indeed there are $o_2\left(T_2 + c_4 \sqrt{\tfrac{m}{n}} \right)$ missing edges among the valid candidates.

\paragraph{Invalid candidates} Next, we consider how many invalid candidates may be left after structure learning. The idea is similar, but requires several additional steps.

First, we have a bound on the gap between $\hat{L}$ and $L$. Since $\hat{L} = \hat{\ell}\hat{\ell}^T$ and $L = \ell \ell^T$, we will convert this to a bound on the gap between $\ell$ and $\ell^T$. Next, since our choice of the valid candidates in Algorithm~\ref{alg:sl1} is based on thresholding $\hat{\Sigma} \hat{\ell}$, we bound the gap between this term and $\Sigma\ell$. We start with the first of these steps.

We have that $\|\ell \ell^T - \hat{\ell} \hat{\ell}^T\| \leq c_5 \sqrt{\frac{m}{n}}$ as our starting point. First, consider, for some vector $x$ and an equal-sized vector of $1$'s, the quantity $x1^T + 1x^T$. Note that  $\|x1^T + 1x^T\|_F^2$ consists of the sum of a series of square terms that include (from the diagonal) $4 \sum x_i^2$. Then we see that 
\[ \|x\| \leq \frac{1}{2} \|x1^T + 1x^T\|_F.\]

Next, let $y$ be some other vector and $y_{\min}>0$ be the smallest magnitude of entry of $y$, where we assume $y$ has no zero entries. Then,
\[\frac{y_{\min}}{2} \|x1^T + 1x^T\|_F \leq  \frac{1}{2} \|xy^T + yx^T\|_F.\]

Next, we can write
\[ \|x\| \leq \frac{1}{2} \|x1^T + 1x^T\|_F \leq \frac{1}{2y_{\min}} \|xy^T + yx^T\|_F .\]

Now we move to 2-norms on the right, getting
\[ \|x\| \leq \frac{\sqrt{m}}{2y_{\min}} \|xy^T + yx^T\| .\]

Afterward, we can write, using the fact that in general $\|A+B\|-\|C\| \leq \|A+B+C\|$, that
\begin{align*}
\|x\| - \frac{\sqrt{m}}{2y_{\min}}\|x\|^2 & \leq \frac{\sqrt{m}}{2y_{\min}} \left( \|xy^T + yx^T\| - \|x\|^2 \right) \\
&=   \frac{\sqrt{m}}{2y_{\min}} \left( \|xy^T + yx^T\| - \|xx^T\| \right) \\
&\leq \frac{\sqrt{m}}{2y_{\min}} \left( \|xy^T + yx^T + xx^T\| \right).
\end{align*}

This can also be written as 
\[ \|x\|(1 - \frac{\sqrt{m}}{2y_{\min}}\|x\|) \leq \frac{\sqrt{m}}{2y_{\min}} \left( \|xy^T + yx^T + xx^T\| \right).\]

Getting back to our initial problem, let us write $\ell = \hat{\ell} + \Delta_\ell$. Then, 
\[\|\ell \ell^T - \hat{\ell} \hat{\ell}^T\| = \|\ell \ell^T - (\ell + \Delta_\ell)(\ell + \Delta_\ell)^T \| = \|\ell\Delta_\ell^T + \Delta_\ell \ell^T + \Delta_\ell \Delta_\ell^T\|. \]

Now take $y$ to be $\ell$ and $x$ to be $\Delta_\ell$. We get
\begin{align*}
\|\Delta_\ell\| (1 - \frac{\sqrt{m}}{2\ell_{\min}}\|\Delta_\ell\|) &\leq  \frac{\sqrt{m}}{2\ell_{\min}}  \|  \Delta_\ell \ell^T + \ell\Delta_\ell^T + \Delta_\ell \Delta_\ell^T\| \\
&=  \frac{\sqrt{m}}{2\ell_{\min}}\|\ell\ell^T - \hat{\ell}\hat{\ell}^T\|  \leq  \frac{\sqrt{m}}{2\ell_{\min}}  c_5 \sqrt{\frac{m}{n}} =  \frac{c_5}{2\ell_{\min}}   \frac{m}{\sqrt{n}} .
\end{align*}

Now, say our number of samples $n$ is large enough (i.e., greater than some $n_0$) to ensure that $\|\Delta_\ell\| \leq (\sqrt{m} / \ell_{\min})^{-1}$ . Then, the left hand side is at least $\|\Delta_\ell\|/2$, so that
\begin{align}
\Delta_\ell = \|\ell - \hat{\ell} \| \leq   \frac{c_5}{\ell_{\min}}   \frac{m}{\sqrt{n}} .
\label{eq:bdell}
\end{align}
Now, we have to translate \eqref{eq:bdell} into the terms we are actually thresholding, $\hat{\Sigma} \hat{\ell}$. This is not difficult:

\begin{align*}
\| \Sigma \ell - \hat{\Sigma} \hat{\ell}\| &= \|(\Sigma \ell - \Sigma \hat{\ell}) + (\Sigma \hat{\ell} - \hat{\Sigma} \hat{\ell})\| \\
&\leq \|\Sigma \ell - \Sigma \hat{\ell}\| + \| \Sigma \hat{\ell} - \hat{\Sigma} \hat{\ell} \| \\
&\leq \|\Sigma\| \| \ell - \hat{\ell}\| + \|\hat{\ell}\| \|\Sigma - \hat{\Sigma}\| \\
&\leq \|\Sigma\|  \frac{c_5}{\ell_{\min}}   \frac{m}{\sqrt{n}} + \|\hat{\ell}\| c_6 \frac{m^2}{\sqrt{n}} \\
&\leq c_3  \frac{m^2}{\sqrt{n}}.
\end{align*}

In the penultimate step, we use \eqref{eq:bdell} and also bound $ \|\Sigma - \hat{\Sigma}\|$; this is a conventional application of matrix concentration. The calculations are explicitly spelled out in \cite{Ratner19}, but the proof of Lemma~\ref{lem:conc} above is almost identical. In the last step, we use the fact that $m \leq m^2$ and bring all of the constant terms together into $c_3$. 

Now, from our independence assumption, $(\Sigma_{Oz})_i = 0$ for $i\not\in V$. This term is just $\Sigma \ell$. Since we use $\hat{\Sigma} \hat{\ell}$ to estimate it, we wish to know how many of these entries are potentially above our threshold $T_1$ and thus will be incorrectly interpreted as valid candidates.

Since $(\Sigma \ell)_i = 0$ for $i\not\in V$, and $\| \Sigma \ell - \hat{\Sigma} \hat{\ell}\| \leq c_3  \frac{m^2}{\sqrt{n}}$, the sum of the squares of terms indexed by $i \not\in V$ in $\hat{\Sigma} \hat{\ell}$ is at most $\left(c_3  \frac{m^2}{\sqrt{n}}\right)^2$. The maximum number of such terms whose magnitude is at least $T_1$ is just 
\[ \left \lfloor \left(c_3 \frac{m^2}{\sqrt{n}}\right)^2 / T_1^2 \right \rfloor.\]
This is just $o_1\left(\frac{c_3 m^2}{T_1 \sqrt{n}} \right)$, as desired.
\end{proof}

Now we proceed with the rest of the theorem.
\begin{proof}
First, we consider the problem setting in the noiseless population-level case, and then we proceed with the sampled results. 

We set $w_1, \ldots, w_p$ be valid candidates and $w_{p+1}, \ldots,w_{p+s}$ be invalid, without loss of generality. Next, let $\mu^*$ be the true parameters. Note that since the invalid IVs are independent of $z$ by assumption and hence are not part of the actual model, $\mu^*_{p+i} = 0$ for $1 \leq i \leq s$. We treat the invalid candidates as part of the model purely for convenience in our notation; of course, in the population-level setting, we would be able to split off the valid model immediately. In the sampled setting, which we encounter in practice, we do not know which of these IVs are valid and which are not.

The true graph model $G$ involving our IV candidates has the following structure: $G := (V,E)$, where $V = \{w_1, \ldots, w_p\}$.
$E$ contains edges between valid candidates $w_1, \ldots, w_p$ only, and no edges for the invalid candidates, which we already detect as follows: we know that for $w_i$ valid and $w_j$ invalid, $\mathbb{E}[w_iw_j] - \mathbb{E}[w_i]\mathbb{E}[w_j] = 0$, by the assumption that invalid IVs are independent of $z$. We also required that there are more valid IVs than invalid ones, so that we can immediately recover the valid IVs (they form the largest connected component, with at least $m/2$ candidates) and then set $\mu^*_j = 0$ for the invalids. It should be noted that the assumption that there are more valid IVs than invalid ones can be further relaxed: as long as the valid IVs form the largest connected component, we can distinguish between valid IVs and the invalid ones. Such a relaxation suggests that the breakdown level of Ivy can be above $50\%$, as shown in Table~\ref{tab:assumptions}. Note also that from structure learning in the noiseless case, we also recover the exact graph $G$ (Lemma 1 in \cite{Varma19}).

We show that under this correctly-specified setup, and with no noise, we recover the remaining $\mu^*$ parameters. We write $O^*$ for the population-level overlaps matrix, $\E[w_Vw_V^T]$. We recall that $q^*_{ij} = \log ((O^*_{ij})^2)$, and that we wish to solve the system $M_{\Omega} \ell^* = q^*$, where $\ell^*_i := \log ((\mu^*_i)^2)$. 

The matrix $M_\Omega$ has a row for each pair of valid IVs that are conditionally independent given $z$. 
By assumption, for each IV candidate, there exists another pair of IVs forming a full-rank $3\times 3$ submatrix $\left[\begin{array}{ccc} 1 & 1  & 0  \\ 1&  0 & 1\\ 0 & 1 &1\end{array}\right]$ in $M_\Omega$. Thus, the column corresponding to this IV candidate cannot be written as a linear combination of any of the other columns in $M_\Omega$, as each row has exactly two nonzeros so none of the other columns have any zeros in these locations. Thus, $M_\Omega$ has full column rank. So, there exists a \textit{unique} solution to $\min \norm{M_\Omega \ell - q^*}_2$ given by the normal equations. Since the population-level results $\ell^*$ satisfy $M_\Omega \ell^* = q^*$, i.e. $\norm{M_\Omega \ell^* - q^*}_2 = 0$, we have that $\ell^*$ is the unique solution to this system, and thus $\mu^*$ is unique as well.

We are only missing one aspect: we need to recover the signs of each of the recovered terms. Here, we use the assumption on the agreement, on average, of the candidates with $z$. For example, if a valid IV candidate $i$ has a better than random chance of agreeing with $z$, we get that $\E[a_i] = \E[w_iz] > 0$. Note also that as soon as we have determined one sign, say for $a_i$, every other candidate accuracy (that forms a row in $M$ with $w_i$) has its sign determined. This concludes the noiseless case.

\paragraph{Sampling results} 
In practice, we do not observe $O^*$, but rather a sampled version $\hat{O}$ that we obtain from samples $w^{(1)}, w^{(2)}, \ldots, w^{(n)}$, estimated as $\hat{O} = \frac{1}{n} \sum_{i=1}^{n} w^{(i)}w^{(i)^T}$. Then, applying Lemma~\ref{lem:conc}, we get that
\begin{align}
\mathbb{E} \left[\|\hat{O} - O^*\|  \right] \leq m^2 \sqrt{\frac{32 \pi}{n}}.
\label{eq:expbound}
\end{align}

If we had access to the true set of edges in $\Omega$ (and thus $M_{\Omega}$), we could then solve the system $M_{\Omega} \tilde{\ell} = \tilde{q}$, where $\tilde{q}:=\hat{q}_\Omega$ represents the subvector of $\hat{q}$ with $\hat{q}_{ij} := \log (\hat{O}_{ij}^2)$ that is associated with $\Omega$. To do so would require that $|\hat{V}| \geq 3$; if $|\hat{V}| \leq 2$, so that we only have two estimated valid candidates after structure learning, we will not have enough signal to obtain accuracy estimates. However, this happens with sufficiently low probability that we can condition on it not occurring (recall that the result holds with probability at least $1-1/m$). 
In practice, though, we do not even know $\Omega$, but rather an estimated version $\hat{\Omega}$. Then, we end up solving $M_{\hat{\Omega}} \hat{\ell} = \hat{q}_{\hat{\Omega}}$, where we note that  $M=M_{\hat{\Omega}}$ and $\hat{q}_{\hat{\Omega}} = \hat{q}$.

We work with a series of perturbation terms. Our final goal is to bound $\|\hat{\mu} - \mu^*\|$. Since we obtain the estimate $\hat{\mu}$ from the estimate $\hat{\ell}$, we will then write $\|\hat{\mu} - \mu^*\|$ as a function of $\|\hat{\ell} - \ell^*\|$. We use the triangle inequality to write
\begin{equation}
\|\hat{\ell} - \ell^*\| = \|\hat{\ell}  - \tilde{\ell} + (\tilde{\ell}- \ell^*)\| \leq \|\hat{\ell} - \tilde{\ell}\| + \|\tilde{\ell} - \ell^*\| .
\label{eq:ldec}
\end{equation}

Here, the first term involves misspecification with respect to the number of edges by using $\hat{\Omega}$ instead of the true $\Omega$, while the second term involves just sampling noise. We control each of these terms separately. In particular, we shall control the second term as a function of the sampling error $\|\hat{O} - O^*\|$. 

Suppose that there are $o$ edges in the true edge set $E$ that are not in $\hat{E}$, our recovered set. Such a non-edge $(i,j)$ in $\hat{E}$ is then included in $\hat{\Omega}$ but not in $\Omega$, leading to additional rows in $M_{\hat{\Omega}}$ that are not in $M_{\Omega}$. Lastly, say that there are $\Upsilon$ non-edges due to the failure of excluding the invalid IVs in our estimated graph (we can think of our true graph as having edges between every invalid IV and any other IV, valid or invalid, since we wish to exclude such rows). Then, we have an additional $\Upsilon$ rows among $M_{\hat{\Omega}}$, for a total of $o + \Upsilon$ spurious constraints. Note that additional spurious constraints are not an issue as long as we can still solve the (normally overdetermined) linear system.

With this setup complete, we proceed to bound each of the two terms on the right-hand side of \eqref{eq:ldec} separately. We call the left term the misspecification term.

\paragraph{Misspecification Term}
To avoid overly cumbersome notation, let us write $M$ for $M_{\hat{\Omega}}$.  Let the SVD of $M$ be $M = U_M \Sigma_M V_M^\top$. Note that since $M_{\Omega}$ is full-rank, and it is a submatrix of $M$ with the same number of columns, $M$ is also full-rank. Thus, $U_M \in \mathbb{R}^{|\hat{\Omega}| \times p}, \Sigma_M \in \mathbb{R}^{p \times p}$, and $V_M \in \mathbb{R}^{p \times p}$. 

Recall that $S$ is the row selection matrix so that $SM = M_{\Omega}$, the corrected form of $M$; in other words, $S$ selects out all the spurious rows. It is a 0/1 matrix of dimensions $|\Omega| \times |\hat{\Omega}|$. 

Recall that $M^\dag = V_M \Sigma_M^{-1} U_M^\top$ and that the residual of the least-squares problem is $r_M = \norm{q - M (M^\dag q)}_2 = \norm{U_M^\perp (U_M^\perp)^\top q}_2$, where $U_M^\perp$ is an orthogonal matrix whose orthonormal columns span a subspace orthogonal to $U_M$. 

We use an argument established in \cite{Mahoney}. That work sought to subsample constraints in a linear regression problem and establish bounds between the result of using all the constraints versus sampling. We use the same strategy, but in our case we are adding rather than removing constraints.

Following section 4.2 in \cite{Mahoney}, we have that
\begin{align*}
\hat{\ell} - \tilde{\ell} &=  M^\dagger \hat{q} - (SM)^\dagger (S \hat{q})  \\
&= V_M \Sigma_M^{-1} U_M^\top \hat{q} - (S U_M \Sigma_M V_M^\top)^\dag S \hat{q}  \\
&= V_M \Sigma_M^{-1} U_M^\top \hat{q} - V_M \Sigma_M^{-1} (SU_M)^\dagger S\hat{q}  \\
&= V_M \Sigma_M^{-1} U_M^\top \hat{q} - V_M \Sigma_M^{-1} (SU_M)^\dagger S (U_M^\perp (U_M^\perp)^\top + U_M U_M^\top)\hat{q}  \\ 
& = V_M \Sigma_M^{-1} U_M^\top \hat{q} -V_M \Sigma_M^{-1} (SU_M)^\dagger SU_M^\perp (U_M^\perp)^\top \hat{q} -V_M \Sigma_M^{-1} (SU_M)^\dagger S U_M  U_M^\top \hat{q}\\
& = V_M \Sigma_M^{-1} U_M^\top \hat{q} -V_M \Sigma_M^{-1} (SU_M)^\dagger SU_M^\perp (U_M^\perp)^\top \hat{q} - V_M \Sigma_M^{-1} U_M^\top \hat{q}\\
&= -V_M \Sigma_M^{-1} (SU_M)^\dagger SU_M^\perp (U_M^\perp)^\top \hat{q},
\end{align*}
where we have used the fact that $U_M^\perp (U_M^\perp)^\top + U_M U_M^\top = I$ and  $(SU_M)^\dagger S U_M = I$. 

Setting  $\Gamma := (SU_M)^\dagger - (SU_M)^\top$, this is 
\begin{align*}
\hat{\ell} - \tilde{\ell}& = -V_M \Sigma_M^{-1} (SU_M)^\dagger SU_M^\perp (U_M^\perp)^\top \hat{q} \\
&=  -V_M \Sigma_M^{-1} ((SU_M)^\top + \Gamma) SU_M^\perp (U_M^\perp)^\top \hat{q}.
\end{align*}

Now, we have the fact that $\Gamma = (SU)^\dagger - (SU)^\top$ satisfies $\norm{\Gamma}_2 = \norm{\Sigma_{SU_M}^{-1} -  \Sigma_{SU_M}}_2$, where $\Sigma_{SU_M}$ is the diagonal matrix from the SVD of $SU_M$. In our case, $\Sigma_{SU_M}$ has entries that are all larger than 0 (and up to 1). If $\rho_{SU_M} = \sigma_{\text{min}}(SU_M)$ is the smallest singular value, then $\|\Gamma\|_2 = \rho_{SU_M}^{-1} - \rho_{SU_M}$. Now we take norms above to get
\begin{align*}
\norm{ \hat{\ell} - \tilde{\ell}}_2 & = \norm{ V_M \Sigma_M^{-1} ((SU_M)^\top + \Gamma) SU_M^\perp (U_M^\perp)^\top \hat{q}}_2 \\
&= \norm{ \Sigma_M^{-1} ((SU_M)^\top + \Gamma) SU_M^\perp (U_M^\perp)^\top \hat{q}}_2  \\
&\leq \norm{ \Sigma_M^{-1} (SU_M)^\top  SU_M^\perp (U_M^\perp)^\top \hat{q}}_2 +  \norm{ \Sigma_M^{-1} \Gamma SU_M^\perp (U_M^\perp)^\top \hat{q}}_2 \\
&\leq \sigma_{\min}^{-1}(M) \left( \norm{ U_M^\top S^\top SU_M^\perp (U_M^\perp)^\top \hat{q}}_2 + \norm { \Gamma\| \| SU_M^\perp (U_M^\perp)^\top \hat{q}}_2 \right) \\
&=  \sigma_{\min}^{-1}(M) \left( \norm{ U_M^\top (I-\Xi) U_M^\perp (U_M^\perp)^\top \hat{q}}_2 + (\rho_{SU_M}^{-1} - \rho_{SU_M}) \norm{ SU_M^\perp (U_M^\perp)^\top \hat{q}}_2 \right).
\end{align*}

In the last step, we use the fact that $S^\top S \in \mathbb{R}^{|\Omega| \times |\Omega|}$ has a submatrix that is $I_{k}$ and is 0 elsewhere. We set $\Xi = I-S^\top S$; $\Xi$ contains an $I_{o + \Upsilon}$ submatrix and 0's elsewhere. Now, $U_M^\top I U_M^\perp = U_M^\top U_M^\perp = 0$, so we have
\begin{align*}
\norm{  \hat{\ell} - \tilde{\ell}}_2 & = \sigma_{\min}^{-1}(M) \left( \norm{ U_M^\top \Xi U_M^\perp (U_M^\perp)^\top b }_2 + (\rho_{SU_M}^{-1} - \rho_{SU_M}) \norm{ SU_M^\perp (U_M^\perp)^\top \hat{q} }_2 \right) \\
& \leq \sigma_{\min}^{-1}(M) \left(
\norm{U_M^\top \Xi}_{\text{F}} \norm{ U_M^\perp (U_M^\perp)^\top b }_2 + ( \rho_{SU_M}^{-1} - \rho_{SU_M} ) \norm{ U_M^\perp (U_M^\perp)^\top \hat{q} }_2 \right)\\
&\leq  \sigma_{\min}^{-1}(M) \left( (o + \Upsilon) u_{\max} + \rho_{SU_M}^{-1} - \rho_{SU_M} \right) \norm{ U_M^\perp (U_M^\perp)^\top \hat{q}}_2  \\
&\leq  \sigma_{\min}^{-1}(M)  \left( (o + \Upsilon)u_{\max} + \rho_{SU_M}^{-1} - \rho_{SU_M} \right) \norm{U_M^\perp (U_M^\perp)^\top} \|\hat{q}\|.\\
&\leq  \sigma_{\min}^{-1}(M)  \left( (o + \Upsilon)u_{\max} + \rho_{SU_M}^{-1} - \rho_{SU_M} \right) \norm{U_M^\perp (U_M^\perp)^\top} \sqrt{|\hat{\Omega}|}q_{\max}.
\end{align*}

Here, we used the fact that $S$ has maximal singular value 1, along with the fact that the norm of largest row of $U_M$ is $u_{\max}$, that the 2-norm of a matrix is no larger than its Frobenius norm, and that $\Xi$ has only $o + \Upsilon$ non-zero entries (all of which are 1's on the diagonal). 

\paragraph{Noise Term}
Now we work on the rightmost term, $\|\tilde{\ell} - \ell^*\|$, where $\tilde{\ell}$ is formed from the correct $M_{\Omega}$ matrix, but we still have sampling noise. 

Recall that $|\tilde{\mu}| = \exp(\frac{\tilde{\ell}}{2})$ and similarly $|\mu^*| = \exp(\frac{\ell^*}{2})$, where the exponential is applied elementwise. We have that, since our matrix $M$ in both cases is full-rank and we have a unique solution,
\begin{align*}
\left \lVert |\tilde{\mu}| - |\mu^*| \right\rVert &= \left \lVert \exp\left(\frac{\tilde{\ell}}{2}\right)  -  \exp\left(\frac{\ell^*}{2}\right) \right\rVert \\
&=   \left \lVert   \exp\left(\frac{\ell^*}{2}\right) \left( \exp\left(\frac{ \tilde{\ell} - \ell^*}{2}\right) - 1\right)  \right\rVert \\
&\leq  \left \lVert   \exp\left(\frac{\ell^*}{2}\right)  \right\rVert \left\lVert  \exp\left(\frac{ \tilde{\ell} - \ell^*}{2}\right) - 1   \right\rVert \\
&=  \left \lVert  \mu^*  \right\rVert \left\lVert  \exp\left(\frac{ \tilde{\ell} - \ell^*}{2}\right) - 1   \right\rVert \\
&\leq \sqrt{m} \left\lVert  \exp\left(\frac{ |\tilde{\ell} - \ell^*|}{2}\right) - 1   \right\rVert,
\end{align*}
where the absolute value in the last expression is applied elementwise.

Note that for all $x \le 1$, $\exp(x) - 1 \leq 2x$.
Using this, we get that in the case $\norm{\tilde{\ell}-\ell^*}_\infty \le 2$,
\[\| |\tilde{\mu}| - |\mu^*|\| \leq 2\sqrt{m} \|\tilde{\ell} - \ell^* \|.\]

Note that as $\mu^* = \E[w \cdot z]$, the entries of $\mu^*$ (and $\hat{\mu}$, by construction) lie in $[-1, +1]$. Thus $\norm{|\hat{\mu}|-|\mu^*|} \le \sqrt{m}$ always. So in the case $\norm{\tilde{\ell}-\ell^*}_\infty > 2$, we trivially have that $\norm{|\hat{\mu}|-|\mu^*|} \le 2\sqrt{m}\norm{\tilde{\ell}-\ell^*}_\infty$.

Recall that $M_{\Omega} \tilde{\ell} = \tilde{q}$ and $M_\Omega \ell^* = q^*$, so $\|\tilde{\ell}-\ell^*\| = \|M^\dag (\tilde{q}-q^*)\|$. Combining this with the above analysis, we have that
\begin{align}
\| |\tilde{\mu}| - |\mu^*|\| &\leq 2\sqrt{m} \|M^\dagger(\tilde{q} - q^*)\|  \nonumber \\
&\leq 2\sqrt{m} \|M^\dagger\| \|\tilde{q} - q^*\|. 
\label{eq:mutoq}
\end{align}
So we just need to bound $\|\tilde{q} - q^*\|$. Recall that $q_{i,j} = \log (O_{ij})^2$, then we have that 
\begin{align*}
	\|\tilde{q} - q^*\|^2
	&=
	\sum_{(i,j)\in \Omega} \left( \log(\hat{O}_{i,j}^2) - \log((O^*_{i,j})^2) \right)^2 \\
	&=
	4 \sum_{(i,j)\in \Omega} \left( \log(|\hat{O}_{i,j}|) - \log(|O^*_{i,j}|) \right)^2 \\
	&= 
	4 \sum_{(i,j)\in \Omega} \left( \log(|O_{i,j}^* + (\Delta_O)_{i,j}|) - \log(|O^*_{i,j}|) \right)^2 \\
	&\leq
	4 \sum_{(i,j)\in \Omega} \left[ \log \left(1 + \left|\frac{(\Delta_O)_{i,j}}{P_{i,j}} \right| \right)\right]^2  \\
	&\leq 
	4 \sum_{(i,j) \in \Omega} \left( \frac{|(\Delta_O)_{i,j}|}{|O^*_{i,j}|} \right)^2  \\
	&\leq 
	\frac{4}{(O^*_{\min})^2} \sum_{(i,j) \in \Omega} (\Delta_O)_{i,j}^2 .
	\end{align*}
Here, we define $\Delta_O := \hat{O} - O$. Note that $O^*_{\min}$ is the same as $R_{\min}$. In the second inequality above, we use $\log(1+x))^2 \leq x^2$. Next, taking square roots and applying \eqref{eq:expbound} by taking expectations, we get that
\begin{align*}
\mathbb{E} [\|\tilde{q} - q^*\|] &\leq \frac{2}{O^*_{\min}} \|\Delta_O\| \\
&\leq  \frac{8m^2}{O^*_{\min}} \sqrt{\frac{2\pi}{n}} .
\end{align*}

We plug this into \eqref{eq:mutoq} to obtain
\begin{equation}
\mathbb{E}[\||\tilde{\mu}| - |\mu^*|\|] \leq \frac{16 m^{\frac{5}{2}}}{O^*_{\min}}  \|M^\dagger\| \sqrt{\frac{2\pi}{n}}. 
\label{eq:firstmu}
\end{equation}

The only remaining step is to combine this with the misspecification step. Using the same idea as earlier, we have that
\[\||\hat{\mu}| - |\tilde{\mu}| \| \leq \sqrt{m} \|\hat{\ell} - \tilde{\ell} \|.\]

Next, using our earlier bound, we have that 
\begin{align*}
\norm{  \hat{\ell} - \tilde{\ell}} &\leq \sigma_{\min}^{-1}(M)  \left( (o + \Upsilon)u_{\max} + \rho_{SU_M}^{-1} - \rho_{SU_M} \right) \norm{U_M^\perp (U_M^\perp)^\top} \sqrt{|\hat{\Omega}|}q_{\max}.
\end{align*}
Then,
\[\||\hat{\mu}| - |\tilde{\mu}| \| \leq \sqrt{m} \sigma_{\min}^{-1}(M)  \left( (o + \Upsilon)u_{\max} + \rho_{SU_M}^{-1} - \rho_{SU_M} \right) \norm{U_M^\perp (U_M^\perp)^\top} \sqrt{|\hat{\Omega}|}q_{\max}.\]
Taking expectations, and using the fact that $|\hat{\Omega}| \leq m^2$, we get 
\begin{equation}
\mathbb{E}[\| |\hat{\mu}| - |\tilde{\mu}| \|] \leq m^{3/2} \sigma_{\min}^{-1}(M)  \left( (o + \Upsilon)u_{\max} + \rho_{SU_M}^{-1} - \rho_{SU_M} \right) \norm{U_M^\perp (U_M^\perp)^\top} q_{\max}.
\label{eq:secmu}
\end{equation}

Then, from triangle inequality.
\[\| |\hat{\mu}| - |\mu^*| \| \leq  \| |\hat{\mu}| - |\tilde{\mu}| \| + \| |\tilde{\mu}| - |\mu^*| \|  .\]

We combine \eqref{eq:firstmu} with \eqref{eq:secmu} to get
\begin{align*}
\mathbb{E}[\| |\hat{\mu}| - |\mu^*| \| ]  &\leq 
\frac{16 m^{\frac{5}{2}}}{O^*_{\min}}  \|M^\dagger\| \sqrt{\frac{2\pi}{n}} \\
&\qquad + m^{3/2} \sigma_{\min}^{-1}(M)  \left( (o + \Upsilon)u_{\max} + \rho_{SU_M}^{-1} - \rho_{SU_M} \right) \norm{U_M^\perp (U_M^\perp)^\top} q_{\max}.
\end{align*}

We assumed, initially, that we had $\Upsilon$ edges from the invalid IVs. After structure learning, from Lemma~\ref{lem:sllemma}, with probability at least $1-1/m$, we have $i_S$ invalid IVs and $e_S$ edges. The $i_S$ invalid IVs can lead to up to $\Upsilon = i_S(p+i_S)$ invalid edges (between any invalid IV and any other IV), while we have $o = e_S$ dependencies. Then, after this substitution, we get our main result. 

The second part of the theorem statement involves the case where we bypass structure learning and directly plug in our IVs, assuming conditional independence, despite the presence of $i$ invalid sources and $e$ dependencies. The only distinction in this case is that we can explicitly compute the minimal singular value of the structured matrix $M$, which is $\sqrt{s+p-2}$, and the $u_{\max}$ term, which is $\sqrt{2/({s+p-2})}$.
\end{proof}

Theorem~\ref{thm:ci_bound_1} follows from Theorem~\ref{thm:extended} since when $n > c_1d^2m$, the structure learning component correctly identifies the valid IV candidates and the correct dependencies among them with high probability (in which case we recover $M_\Omega$ and thereby the correct signs for $\hat{\mu}$ as well). For convenience, we restate Theorem~\ref{thm:ci_bound_1} below.
\thmcibound*

\subsection{Auxiliary Lemmas for Theorem~\ref{thm:ci-wald}}
\label{sec:aux-lemma}
Next, we present some useful results that will help us with our proof of Theorem~\ref{thm:ci-wald}. We first present Lemma~\ref{lemma:strong-cvx}, which details a mild regularity condition under which we can reason within a feasible region of the parameter space that yields a strongly convex population level negative log-likelihood function.

\begin{lemma}
\label{lemma:strong-cvx} Let $\ell (x; \theta)$ be the negative log-likelihood function parameterized by $\theta \in \Theta$. Let $\mathcal{I}(\theta) := \mathbb{E} \left[\frac{\partial^2 \ell (x; \theta)}{\partial \theta^2}\right]$ be the corresponding Fisher information matrix. Suppose that for any given $\theta \in \Theta$, there exists an unbiased estimator $\check{\theta}$ of $\theta$ that is a function of some number $n$ of independent samples $\curly{x^{(i)}}_{i=1}^n$ drawn from $\mathcal{D}$ such that, for some constant $\lambda$, $\text{Cov} (\check{\theta}) \preceq \frac{I}{2 \lambda n}$. Then $L(\theta) := \mathbb{E}[\ell(x;\theta)]$ is $\lambda$-strongly convex with respect to $\theta$.
\end{lemma}

\begin{proof}
The proof follows the rationale of that in \cite{Ratner16}. From the Cram\'{e}r-Rao lower bound, we know in general that the variance of any unbiased estimator is bounded by the inverse of the Fisher information:
\begin{equation*}
\text{Cov} (\check{\theta}) \succeq (\mathcal{I(\theta)})^{-1}.
\end{equation*}
Since the unbiased $\check{\gamma}$ by construction is learned from $n$ independent samples from $\mathcal{D}$, it follows that the Fisher information is $n$ times the Fisher information of a single sample:
\begin{equation*}
\mathbb{E} \left[\sum_{i=1}^n \frac{\partial^2}{\partial \theta^2} \ell (x^{(i)}; \theta) \right] = \sum_{i=1}^n \mathbb{E} \left[\frac{\partial^2  \ell (x^{(i)}; \theta)}{\partial \theta^2} \right] = n \mathbb{E} \left[\frac{\partial^2  \ell (x; \theta)}{\partial \theta^2} \right]
= n \mathcal{I}(\theta).
\end{equation*}

Combining this with the bound in the lemma statement on the covariance, we get
\begin{equation*}
\frac{I}{2 \lambda n} \succeq  (n\mathcal{I(\theta)})^{-1}.
\end{equation*}
It follows that
\begin{equation*}
\mathbb{E} \left[\frac{\partial^2 \ell (x; \theta)}{\partial \theta^2} \right]  = \mathcal{I} (\theta) \succeq 2\lambda I,
\end{equation*}
which means $L(\gamma)$ is $\lambda$-strongly convex.
\end{proof}

Let $l(y,z; \gamma) := \log [1 + \exp (-y(k +\beta z))]$, with $\gamma := (k,\beta)$. Lemma~\ref{lemma:1} upper bounds the error in the parameter of the logistic regression model that uses $z$ to predict $y$ by the error in the parameters of Ivy. Notice that the same lemma can also be applied to upper bound the error in the parameters of the logistic regression model that uses $z$ to predict $x$.

\begin{lemma}
\label{lemma:1}
Let $L(\gamma) := \mathbb{E}[l(y,z;\gamma)]$, $\gamma^* := \argmin_{\gamma} L(\gamma)$, and let $\hat{\gamma}$ be the logistic regression parameters learned using the data set and the Ivy estimator $\hat{z}$. Suppose that there exists a constant $c_5>0$ such that $\max\{l(y,z;\hat{\gamma}),\,l(y,z;\gamma^*))\} \le c_5$, and let $c_6$ be a constant. Suppose further that the assumptions in Lemma~\ref{lemma:strong-cvx} hold for $l(y,z;\gamma)$. Then, $L(\gamma)$ is $\lambda$-strongly convex, and
\begin{equation*}
\norm{\hat{\gamma}-\gamma^*}_2^2 \le \frac{2c_6}{\lambda}  \sqrt{\frac{2 \pi}{n}} + \frac{32}{\lambda} c_5 \left(\mathbb{E} \left[ \norm{\hat{\mu}-\mu^*}_{\infty} \right ] + \mathbb{E}\left[ \norm{\hat{O}-O^*}_{\infty} \right ] \right).
\end{equation*}
\end{lemma}

\begin{proof}
In words, $\gamma^* \in \mathbb{R}^2$ is the optimal parameter vector of the population level logistic regression when $z$ is observed. Similarly, $\hat{\gamma} \in \mathbb{R}^2$ is the optimal parameter vector of the logistic regression when using $\hat{z}$ to predict $y$. First, we would like to characterize
$\lvert L(\hat{\gamma}) -  L(\gamma^*) \rvert = \lvert \mathbb{E} [l(y,z;\hat{\gamma})  - l(y,z;\gamma^*) ]\rvert$.
Define $L_{\mu, O}(\gamma) := \mathbb{E}_{y,w\sim\mathcal{D}}[\mathbb{E}_{z \sim \prob_{\mu, O}( z \mid w)}[l(y,z;\gamma)]]$.
Note that $L(\gamma) =
\mathbb{E}_{y,w\sim\mathcal{D}}[\mathbb{E}_{z \sim \prob_{\mu^*, O^*}(z \mid w)}[l(y,z;\gamma)]] = L_{\mu^*, O^*}(\gamma)$ by definition of $\mu^*, O^*$.
Furthermore,
\begin{align*}
L(\hat{\gamma}) - L(\gamma^*) = & L_{\mu^*, O^*}(\hat{\gamma}) + L_{\hat{\mu},\hat{O}}(\hat{\gamma}) - L_{\hat{\mu},\hat{O}}(\hat{\gamma}) + L_{\hat{\mu},\hat{O}}(\tilde{\gamma}) - L_{\hat{\mu},\hat{O}}(\tilde{\gamma}) - L_{\mu^*, O^*}(\gamma^*)\\
\le & L_{\mu^*, O^*}(\hat{\gamma}) + L_{\hat{\mu},\hat{O}}(\hat{\gamma}) - L_{\hat{\mu},\hat{O}}(\hat{\gamma}) + L_{\hat{\mu},\hat{O}}(\gamma^*) - L_{\hat{\mu},\hat{O}}(\tilde{\gamma}) - L_{\mu^*, O^*}(\gamma^*)\\
\le & L_{\hat{\mu},\hat{O}}(\hat{\gamma}) - L_{\hat{\mu},\hat{O}}(\tilde{\gamma}) + \abs{L_{\mu^*, O^*}(\hat{\gamma}) - L_{\hat{\mu},\hat{O}}(\hat{\gamma}) } + \abs{ L_{\hat{\mu},\hat{O}}(\gamma^*) - L_{\mu^*, O^*}(\gamma^*)}\\
\le & \xi(n) + 2 \abs{L_{\hat{\mu},\hat{O}}(\gamma') - L_{\mu^*, O^*}(\gamma') },
\end{align*}
where $\tilde{\gamma} := \argmin_{\gamma} L_{\hat{\mu},\hat{O}}(\gamma)$, $\xi(n)$ is the estimation error $L_{\hat{\mu},\hat{O}}(\hat{\gamma}) - L_{\hat{\mu},\hat{O}}(\tilde{\gamma})$, and $\gamma' := \argmax_{\gamma \in \curly{\hat{\gamma},\gamma^*}} \abs{L_{\hat{\mu},\hat{O}}(\gamma)-L_{\mu^*, O^*}(\gamma)}$. It remains to control $\abs{L_{\hat{\mu},\hat{O}}(\gamma') - L_{\mu^*, O^*}(\gamma')  }$. Specifically,
\begin{align*}
\abs{L_{\hat{\mu},\hat{O}}(\gamma') - L_{\mu^*, O^*}(\gamma') }
= & \Abs{\mathbb{E}_{y,w\sim\mathcal{D}}[\mathbb{E}_{z \sim \prob_{\hat{\mu},\hat{O}}(\cdot \mid w)}[l(y,z;\gamma')]] -\mathbb{E}_{y,w\sim\mathcal{D}}[\mathbb{E}_{z \sim \prob_{\mu^*, O^*}(\cdot \mid w)}[l(y,z;\gamma')]]} \\
= & \Abs{\mathbb{E}_{y,w\sim\mathcal{D}} \left[\sum_z l(y,z;\gamma') \left( \prob_{\hat{\mu},\hat{O}}(z \mid w) - \prob_{\mu^*, O^*}(z \mid w) \right) \right]} \\
\le & c_5 \sum_{z\in\curly{-1,1}} \mathbb{E}_{y,w\sim\mathcal{D}} \left[\Abs{\prob_{\hat{\mu},\hat{O}}(z \mid w) - \prob_{\mu^*, O^*}(z \mid w)} \right]\\
\le & 2 c_5 \max_{z\in\curly{-1,1}} \mathbb{E}_{y,w\sim\mathcal{D}} \left[\Abs{\prob_{\hat{\mu},\hat{O}}(z \mid w) - \prob_{\mu^*, O^*}(z \mid w)} \right]\\
\le & 2 c_5 \max_{z\in\curly{-1,1}} \mathbb{E}_{y,w\sim\mathcal{D}} \left[\Abs{\log \prob_{\hat{\mu},\hat{O}}(z \mid  w) - \log \prob_{\mu^*, O^*}(z \mid  w)} \right]\\
\le & 8 c_5 \norm{\hat{\theta}-\theta^*}_{\infty},
\end{align*}
where in the first inequality we use the assumption that $l(y,z;\gamma') \le c_5$, in the penultimate inequality we use the fact that $\abs{x-y}\le \abs{\log x -\log y}$ for $x$,$y\in [0,1]$, and in the last inequality  we follow the argument in \citealt[Appendix B.3]{Ratner19} and use the fact that $\Abs{\prob_{\hat{\mu},\hat{O}}(z, w) - \prob_{\mu^*, O^*}(z, w)} \le 2\norm{\hat{\theta}-\theta^*}_\infty$ due to \citealt[Lemma 19]{honorio2012lipschitz}. Here, $\hat{\theta}$ and $\theta^*$ are the canonical parameters of the graphical models. It remains to bound $\norm{\hat{\theta}-\theta^*}$ with $\norm{\hat{\mu}-\mu^*}$. To this end, notice that $\nabla A(\theta) = \mu$ \citep{wainwright2008graphical}, where $A(\theta)$ is the log partition function. Furthermore, $\nabla^2 A(\theta)$ is the covariance matrix whose smallest eigenvalue value is $\sigma_{\min}>0$. We therefore have that $\nabla A(\theta)$ is $\sigma_{\min}$-strongly convex. By Fenchel duality \citep{zhou2018fenchel}, $\nabla A^*(\mu)$---the dual of the $\sigma_{\min}$-strongly convex $\nabla A(\theta)$---is $1/\sigma_{\min}$-Lipschtiz. As a result,
\begin{equation*}
\norm{\hat{\theta}-\theta^*}_{\infty} \le \norm{\hat{\theta}-\theta^*}_2=\norm{\nabla A^*(\hat{\mu}) - \nabla A^*(\mu^*)}_2 \le \frac{1}{\sigma_{\min}} \left( \norm{\hat{\mu}-\mu^*}_2+\norm{\hat{O}-O^*}_2 \right),
\end{equation*}
where we have used the fact that $\nabla^*A(\mu) = \theta$.
Therefore,
\begin{equation*}
L(\hat{\gamma})-L(\gamma^*) \le \xi(n) + 16 \tilde{c_5} \left(\mathbb{E} \left[ \norm{\hat{\mu}-\mu^*} \right ] + \mathbb{E}\left[ \norm{\hat{O}-O^*} \right ] \right),
\end{equation*}
where $\tilde{c}_5 = \frac{c_5}{2\sigma_{\min}}$.
Using the fact that $L(\gamma)$ is $\lambda$-strongly convex and $\gamma^*$ is the global optimizer of  $L(\gamma)$, we have that
\begin{align*}
\frac{\lambda}{2} & \norm{\hat{\gamma}-\gamma^*}_2^2 \le L(\hat{\gamma})-L(\gamma^*) \le \xi(n) + \frac{16 \tilde{c}_5}{\sigma_{\min}} \left(\mathbb{E} \left[ \norm{\hat{\mu}-\mu^*} \right ] + \mathbb{E}\left[ \norm{\hat{O}-O^*} \right ] \right) \\
\Rightarrow & \norm{\hat{\gamma}-\gamma^*}_2^2 \le \frac{2}{\lambda}\xi(n) + \frac{32 \tilde{c}_5}{\lambda \sigma_{\min}}  \left(\mathbb{E} \left[ \norm{\hat{\mu}-\mu^*} \right ] + \mathbb{E}\left[ \norm{\hat{O}-O^*} \right ] \right).
\end{align*}

What remains is to bound the $\xi(n)$ estimation term; this is standard. First, we can use the Lipschitz property of the functions involved to write 
\[\xi(n) \leq c_6 \E_{y,w\sim \mathcal{D}}\left[ \mathbb{E}_{z_1, \ldots, z_n \sim P_{\hat{\mu},\hat{O}}(\cdot | w)} \left[ \left|\frac{z_1 + \ldots  + z_n}{n} - \bar{z}\right|\right]\right],\] 
where $\bar{z} = \mathbb{E}_{P_{\hat{\mu},\hat{O}}}(z)$ and $c_5$ combines the Lipschitz constants. Then, it remains to apply Hoeffding's inequality, noting that $z$ takes on values in $\{-1, +1\}$. Thus, we have
\[P\left(\left|\frac{z_1 + \ldots  + z_n}{n} - \bar{z} \right| \geq t\right) \leq 2 \exp(-nt^2/2).\]
Finally, integrating this over $t$, we get that 
\[\xi(n) \leq c_6  \sqrt{\frac{2 \pi}{n}}. \]
\end{proof}

Lemma~\ref{lemma:2} upper bounds the error in the Wald causal effect estimates with the error in the parameters of the corresponding logistic regression models.

\begin{lemma}
\label{lemma:2}
Let $\hat{\beta}_{\hat{z}y}$ and $\hat{\beta}_{\hat{z}x}$ be estimates of Ivy from $n$ data points. Let $\gamma^*_{zy} := \argmin_{\gamma}\mathbb{E}\left[l(y,z;\gamma)\right]$ and $\gamma^*_{zx} := \argmin_{\gamma}\mathbb{E}\left[l(x,z;\gamma)\right]$, with $\gamma^*_{zy} = (k^*_{zy},\beta^*_{zy})$ and $\gamma^*_{zx} = (k^*_{zx},\beta^*_{zx})$. That is, $\beta^*_{zy}$ and $\beta^*_{zx}$ are the population-level regression coefficients of $z$ when $z$ is observed.  If there exists $0<\kappa<1$ such that,
\begin{equation}
\label{eq:lemma:2-assumption}
\delta := \max\curly{\lvert \hat{\beta}_{\hat{z}y} - \beta_{zy}^* \rvert, \lvert \hat{\beta}_{\hat{z}x} - \beta_{zx}^* \rvert} \le \kappa \lvert \beta_{zx}^* \rvert,
\end{equation}
then the Wald causal effect estimator $\hat{\alpha}_{x \rightarrow y} := \frac{\hat{\beta}_{\hat{z}y}}{\hat{\beta}_{\hat{z}x}}$ and the population-level Wald ratio $\alpha^*_{x \rightarrow y} := \frac{\beta_{zy}^*}{\beta_{zx}^*}$ satisfy:
\begin{equation*}
\abs{\hat{\alpha}_{x\rightarrow y}-\alpha^*_{x\rightarrow y}} \le  \frac{2(\beta_{zx}^*+\beta_{zy}^*)}{(1-\kappa)\beta_{zx}^{*2}} \delta.
\end{equation*}
\end{lemma}
\begin{proof}
By the assumed inequality \eqref{eq:lemma:2-assumption},
\begin{align}
\label{eq:lemma:2-unpack-assumption}
\begin{split}
-\delta \le \hat{\beta}_{\hat{z}y} - \beta_{zy}^* \le \delta \Rightarrow \beta_{zy}^*-\delta \le \hat{\beta}_{\hat{z}y} \le \beta_{zy}^* + \delta,\\
-\delta \le \hat{\beta}_{\hat{z}x} - \beta_{zx}^* \le \delta \Rightarrow \beta_{zx}^*-\delta \le \hat{\beta}_{\hat{z}x} \le \beta_{zx}^* + \delta.
\end{split}
\end{align}
Without loss of generality, we assume that $\beta^*_{zy}\ge 0$ and $\beta^*_{zx} \ge 0$, because we can always make sure that $\beta^*_{zy}$ and $\beta^*_{zx}$ are nonnegative with the appropriate representation of our data.  By the assumption in \eqref{eq:lemma:2-assumption}, $\delta \le \kappa \beta_{zx}^*$, hence $\beta^*_{zx} > \delta$ since $\kappa \in (0,1)$. Intuitively, this means we should use an IV that is sufficiently predictive of $x$ (i.e.\ not a weak IV). Using \eqref{eq:lemma:2-unpack-assumption},
\begin{itemize}[leftmargin=*]
\item When $\beta_{zy}^*-\delta \ge 0$,
\begin{align*}
0 & \le \frac{\beta_{zy}^*-\delta}{\beta_{zx}^*+\delta} \le \frac{\hat{\beta}_{\hat{z}y}}{\hat{\beta}_{\hat{z}x}} \le \frac{\beta_{zy}^*+\delta}{\beta_{zx}^*-\delta}\\
\text{ and } 0  & \le \frac{\beta_{zy}^*-\delta}{\beta_{zx}^*+\delta} \le \frac{\beta_{zy}^*}{\beta_{zx}^*} \le \frac{\beta_{zy}^*+\delta}{\beta_{zx}^*-\delta}\\
\Rightarrow \abs{\hat{\alpha}_{x\rightarrow y}-\alpha^*_{x\rightarrow y}} & \le \frac{\beta_{zy}^*+\delta}{\beta_{zx}^*-\delta}-\frac{\beta_{zy}^*-\delta}{\beta_{zx}^*+\delta} = \frac{2(\beta_{zx}^*+\beta_{zy}^*)}{\beta_{zx}^{*2}-\delta^2} \delta\\
& \le \frac{2(\beta_{zx}^*+\beta_{zy}^*)}{(1-\kappa^2)\beta_{zx}^{*2}} \delta \le \frac{2(\beta_{zx}^*+\beta_{zy}^*)}{(1-\kappa)\beta_{zx}^{*2}} \delta.
\end{align*}

\item When $\beta_{zy}^*-\delta \le 0 \le \hat{\beta}_{\hat{z}y}$,
\begin{align*}
0 & \le \frac{1}{\beta_{zx}^*+\delta} \le \frac{1}{\hat{\beta}_{\hat{z}x}} \le \frac{1}{\beta_{zx}^*-\delta} \\
\text{ and } 0 &  \le \frac{\beta_{zy}^*}{\beta_{zx}^*} \le
\frac{\beta_{zy}^*+\delta}{\beta_{zx}^*-\delta}\\
\Rightarrow 0 & \le \frac{\hat{\beta}_{\hat{z}y}}{\hat{\beta}_{\hat{z}x}} \le
\frac{\beta_{zy}^*+\delta}{\beta_{zx}^*-\delta}\\
\Rightarrow \abs{\hat{\alpha}_{x\rightarrow y}-\alpha^*_{x\rightarrow y}} & \le \frac{\beta_{zy}^*+\delta}{\beta_{zx}^*-\delta} \le \frac{\beta_{zy}^*+\delta}{\beta_{zx}^*-\delta}-\frac{\beta_{zy}^*-\delta}{\beta_{zx}^*+\delta}\\
& \le \frac{2(\beta_{zx}^*+\beta_{zy}^*)}{(1-\kappa^2)\beta_{zx}^{*2}} \delta \le \frac{2(\beta_{zx}^*+\beta_{zy}^*)}{(1-\kappa)\beta_{zx}^{*2}} \delta.
\end{align*}

\item When $\hat{\beta}_{\hat{z}y} \le 0$,
\begin{align*}
0 & \le \frac{1}{\beta_{zx}^*+\delta} \le \frac{1}{\hat{\beta}_{\hat{z}x}} \le \frac{1}{\beta_{zx}^*-\delta} \\
\text{ and } 0 & \le -\hat{\beta}_{\hat{z}y} \le -(\beta_{zy}^*-\delta) \\
\Rightarrow 0 & \le - \frac{\hat{\beta}_{\hat{z}y}}{\hat{\beta}_{\hat{z}x}} \le -\frac{\beta_{zy}^*-\delta}{\beta_{zx}^*-\delta}
\Rightarrow \frac{\beta_{zy}^*-\delta}{\beta_{zx}^*-\delta} \le \frac{\hat{\beta}_{\hat{z}y}}{\hat{\beta}_{\hat{z}x}} \le 0 \\
\Rightarrow \abs{\hat{\alpha}_{x\rightarrow y}-\alpha^*_{x\rightarrow y}} &  \le \alpha^*_{x\rightarrow y} - \frac{\beta_{zy}^*-\delta}{\beta_{zx}^*-\delta} = \frac{\beta_{zx}^*-\beta_{zy}^*}{\beta_{zx}^*(\beta_{zx}^*-\delta)} \delta \\
& \le \frac{\beta_{zx}^*-\beta_{zy}^*}{(1-\kappa)\beta_{zx}^{*2}} \delta \le \frac{2(\beta_{zx}^*+\beta_{zy}^*)}{(1-\kappa)\beta_{zx}^{*2}} \delta.
\end{align*}
\item Thus, for all the cases discussed above, we have that
\begin{equation*}
\abs{\hat{\alpha}_{x\rightarrow y}-\alpha^*_{x\rightarrow y}} \le  \frac{2(\beta_{zx}^*+\beta_{zy}^*)}{(1-\kappa)\beta_{zx}^{*2}} \delta.
\end{equation*}
\end{itemize}

\end{proof}

\subsection{Proof of Theorem \ref{thm:ci-wald}}
\label{sec:proof-wald-appendix}
For convenience, we restate Theorem \ref{thm:ci-wald}:
\thmciwald*

\begin{proof}
The proof of Theorem~\ref{thm:ci-wald} follows from combining Lemma~\ref{lemma:1}, Lemma~\ref{lemma:2}, and Theorem~\ref{thm:ci_bound_1}. 
First, from Lemma~\ref{lemma:2},
\begin{equation}
\label{eq:thm2-step-1}
|\hat{\alpha}_{x\rightarrow y}-\alpha^*_{x\rightarrow y}| 
\le \frac{2(\beta_{zx}^*+\beta_{zy}^*)}{(1-\kappa)\beta_{zx}^{*2}} \delta 
\le \frac{2(\beta_{zx}^*+\beta_{zy}^*)}{(1-\kappa)\beta_{zx}^{*2}}  \norm{\hat{\gamma}-\gamma^*},
\end{equation}
where the second inequality is due to the fact that $\delta =  \max\curly{\lvert \hat{\beta}_{\hat{z}y} - \beta_{zy}^* \rvert, \lvert \hat{\beta}_{\hat{z}x} - \beta_{zx}^* \rvert} \le \norm{\hat{\gamma}-\gamma^*}$, where $\gamma^*$ corresponding to the regression coefficient vector of using either $x$ or $y$ as the dependent variable.
Combining \eqref{eq:thm2-step-1} with Lemma~\ref{lemma:1} and Lemma~\ref{lem:conc} yields:
\begin{align}
(\hat{\alpha}_{x\rightarrow y}-\alpha^*_{x\rightarrow y})^2 
\le & \frac{4(\beta_{zx}^*+\beta_{zy}^*)^2}{(1-\kappa)^2\beta_{zx}^{*4}} \left[
\frac{2c_6}{\lambda}  \sqrt{\frac{2 \pi}{n}} + \frac{32}{\lambda} \tilde{c}_5 \left(\mathbb{E} \left[ \norm{\hat{\mu}-\mu^*}_{\infty} \right ] + \mathbb{E}\left[ \norm{\hat{O}-O^*}_{\infty} \right ] \right) \right]\nonumber \\ 
\Rightarrow \mathbb{E} [(\hat{\alpha}_{x\rightarrow y}-\alpha^*_{x\rightarrow y})^2] \le &
\frac{8(\beta_{zx}^*+\beta_{zy}^*)^2}{\lambda (1-\kappa)^2 \beta_{zx}^{*4}} \left(c_6 + 16 \tilde{c}_5 m^2 \right) \sqrt{\frac{2\pi}{n}} + 
\frac{128 c_1}{\lambda (1-\kappa)^2} \cdot \frac{(\beta_{zx}^*+\beta_{zy}^*)^2}{\beta_{zx}^{*4}} \cdot \mathbb{E} [ \norm{\hat{\mu}-\mu}_{\infty}].
\label{eq:thm2-step-2}
\end{align}

Applying Theorem~\ref{thm:ci_bound_1} to \eqref{eq:thm2-step-2} and using the fact that $\norm{\cdot}_\infty \le \norm{\cdot}_2$, we have that
\begin{align*}
\mathbb{E}[|\hat{\alpha}_{x\rightarrow y}-\alpha^*_{x\rightarrow y}|] &\le \frac{8(\beta_{zx}^*+\beta_{zy}^*)^2}{\lambda (1-\kappa)^2 \beta_{zx}^{*4}} \left(c_6 + 16 \tilde{c}_5 m^2 \right) \sqrt{\frac{2\pi}{n}} + \frac{2048c_1\sqrt{2\pi} m^{\frac{5}{2}}\|M^\dagger\| (\beta_{zx}^*+\beta_{zy}^*)^2}
{R_{\min} \lambda (1-\kappa)^2 \beta_{zx}^{*4}\sqrt{n}} \\
&= \frac{(\beta_{zx}^*+\beta_{zy}^*)^2}{\lambda (1-\kappa)^2 \beta_{zx}^{*4}}  \sqrt{\frac{2\pi}{n}}  \left( 8c_6 + 128 \tilde{c}_5 m^2 + \frac{2048 c_1}{R_{\min}} m^{\frac{5}{2}} \|M^\dagger\| \right) \\
&\leq \sqrt{\frac{1}{n}} \cdot \frac{6000 c_2 m^{\frac{5}{2}}(\beta_{zx}^*+\beta_{zy}^*)^2(1 + \|M^\dagger\|)}{R_{\min} \lambda (1-\kappa)^2 \beta_{zx}^{*4}}.
\end{align*}
with probability at least $1-\frac{1}{m}$ if $n > c_1d^2m$. Here, we set $c_2 := \max \{c_1, \tilde{c}_5, c_6\}$, and we used the fact that $0 < R_{\min} \leq 1$ since the $w_i$'s are in $\{-1, +1\}$.
\end{proof}

\subsection{Non-Zero Error in Causal Effect Estimation}
\label{sec:lower-bound-causal-effect-appendix}
Suppose that $0 < \epsilon < \min\curly{\lvert \hat{\beta}_{\hat{z}y} - \beta_{zy}^* \rvert, \lvert \hat{\beta}_{\hat{z}x} - \beta_{zx}^* \rvert}$. Here is an example where the error of the causal effect estimate is lower bounded. Consider the event  $E_1 =\curly{\hat{\beta}_{\hat{z}y} > \beta_{zy}^*>0 \text{\ and\ } 0 < \hat{\beta}_{\hat{z}x} < \beta_{zx}^*}$. When $E_1$ happens, we have that $\epsilon < \hat{\beta}_{\hat{z}y}-\beta_{zy}^*$. Therefore,
\begin{align*}
0 &  < \frac{\beta_{zy}^*+\epsilon}{\hat{\beta}_{\hat{z}x}} < \frac{\hat{\beta}_{\hat{z}y}}{\hat{\beta}_{\hat{z}x}} \\
\text{ and } 0 & < \frac{\beta_{zy}^*}{\beta_{zx}^*} < \frac{\beta_{zy}^*}{\hat{\beta}_{\hat{z}x}} \Rightarrow - \frac{\beta_{zy}^*}{\hat{\beta}_{\hat{z}x}} < - \frac{\beta_{zy}^*}{\beta_{zx}^*}\\
\Rightarrow 0 & <   \frac{\epsilon}{\hat{\beta}_{\hat{z}x}}  < \hat{\alpha}_{x\rightarrow y}-\alpha_{x\rightarrow y}^* \\
\Rightarrow 0 & < \frac{\epsilon}{\beta_{zx}^*} < \hat{\alpha}_{x\rightarrow y}-\alpha_{x\rightarrow y}^*.
\end{align*}

In general, $E_1$ happens with non-zero probability. In this case $\norm{\hat{\alpha}_{x\rightarrow y}-\alpha_{x\rightarrow y}^*}_2$ is bounded away from zero.

\subsection{Statistical Power Estimation}
\label{sec:power-appendix}
In addition to accurately estimating the underlying causal effects (when such effects are present), it is also useful to characterize the reliability of such inferences. That is, when the algorithm produces a claim on the presence of causal effects, can we confidently trust such a result? 

To answer this question, we work with a standard statistical power estimator and characterize its behavior when the Ivy estimator is used as input. Statistical power is the probability of rejecting a false null hypothesis; here, the null hypothesis is that there is no causal effect between the risk factor and the outcome. We denote the probability of rejecting a true null hypothesis as $\alpha$ (type-I error rate), and we denote the probability of not rejecting a false null hypothesis as $\beta$ (type-II error rate). Therefore, the power of the statistical test is $1-\beta$. When $z$ is observed, Lemma~\ref{lem:power} provides an estimator of statistical power based on standard normality approximations \citep{Freeman13}. We write $p_0 = \prob(y=0)$ and $p_1 = \prob(y=1)$ for convenience. We also let $\zeta_{\delta}$ be such that $\Phi(-\zeta_{\delta}) = \delta$, where $\Phi$ is the cdf of the standard normal distribution. The following lemma follows from standard arguments on power estimation,

\begin{lemma}
The statistical power of the Wald estimator at level $1-|\beta_{xz}|$ with respect to the null hypothesis that there is no causal effect between a binary risk factor and a binary outcome when $z$ is observed with $n$ samples is given by:
\[
\label{eq:power}
\pi := 1 -  \Phi \left(\zeta_{\frac{a}{2}} - \sqrt{n p_1 p_0} \abs{\alpha^*_{xy}} \abs{\beta^*_{zx}} \right).
\]
\label{lem:power}
\end{lemma}
If we had access to the true $z$, the above expression would use $\beta_{xz}^*$ to yield the true power $\pi^*$; instead, we use the Ivy procedure to estimate $\beta_{xz}^*$ by $\hat{\beta}_{xz}$. We denote the resulting power estimates by  $\hat{\pi}$. Our next result shows that, despite relying on IV candidates, the Ivy procedure still produces a power that approximates the ideal power $\pi^*$ arbitrarily well in the case where we have a correctly specified model.
\begin{theorem}
\label{thm:power-diff}
Let $\hat{\pi}$ be the power estimated using Ivy according to \eqref{eq:power} with $\hat{\beta}_{xz}$ in lieu of $\beta_{xz}^*$. If $n$ is the number of samples, then the power difference $\abs{\hat{\pi}-\pi^*} $ satisfies
\[ 
\abs{\hat{\pi}-\pi^*} \leq \sqrt{\frac{p_1 p_0}{2 \pi}} C \alpha^*_{xy} \exp \left(\zeta_{\frac{\alpha}{2}} - \sqrt{n p_1 p_0} \alpha^*_{x\rightarrow y} \beta_{zx}^{*}  \right),
\]
where $C$ is a constant.
\end{theorem}

Before presenting the proof, we briefly comment on this result. Note that as $n \rightarrow \infty$, the exponent term becomes arbitrarily small. The impact of the estimation part is limited to the coefficient $\sqrt{p_1 p_0} C \alpha^*_{x\rightarrow y}$.

\begin{proof}
To ease the notation, let us write $p_1 := \prob(Y=1)$ and $p_0 := \prob(Y=0)$. Then,
\begin{align*}
\abs{\hat{\pi}-\pi^*} =  & \Abs{ \Phi (z_{\frac{\alpha}{2}} - \sqrt{n p_1 p_0} \alpha_{YX}^* \hat{\beta}_{ZX} ) -  \Phi (z_{\frac{\alpha}{2}} - \sqrt{n p_1 p_0} \alpha^*_{YX} \beta_{ZX}^{*} )}\\
&\leq \Abs{ \Phi (z_{\frac{\alpha}{2}} - \sqrt{n p_1 p_0} \alpha_{YX}^* \left(\beta^*_{ZX} - \frac{C}{\sqrt{n}}\right) -  \Phi (z_{\frac{\alpha}{2}} - \sqrt{n p_1 p_0} \alpha^*_{YX} \beta_{ZX}^{*} )}\\
&= \Abs{ \Phi ( [z_{\frac{\alpha}{2}} + \sqrt{p_1p_0} C \alpha^*_{YX}] - \sqrt{n p_1 p_0} \alpha_{YX}^*\beta^*_{ZX} ) -  \Phi (z_{\frac{\alpha}{2}} - \sqrt{n p_1 p_0} \alpha^*_{YX} \beta_{ZX}^{*} )}.
\end{align*}
The first step uses our result that $| \hat{\beta}_{ZX} - \beta^*_{ZX} | \leq \frac{C}{\sqrt{n}}$ for some constant term $C$.

The previous expression can be written as $\Phi(B) - \Phi(A)$. Note that
\[\Phi(B) - \Phi(A) = \int_A^B \frac{1}{\sqrt{2\pi}} \exp(-x^2/2) dx \leq \int_A^B \frac{1}{\sqrt{2\pi}} \exp(-A^2/2) dx = \frac{1}{\sqrt{2\pi}} \exp(-A^2/2) (B-A) .\]

Now, replacing $A$ and $B$ with their corresponding terms, we have that

\[ 
\abs{\hat{\pi}-\pi^*} \leq \sqrt{\frac{p_1 p_0}{2 \pi}} C \alpha^*_{YX} \exp \left(z_{\frac{\alpha}{2}} - \sqrt{n p_1 p_0} \alpha^*_{YX} \beta_{ZX}^{*}  \right),
\]
as desired.

\end{proof}

\subsection{Conditional Independent Model and Unary Potentials}
\label{sec:cond-indep-appendix}
One of the properties we used in our algorithms was that the accuracies are independent when the candidates are conditionally independent and distributed according to our Ising model. We prove this property formally below.

\begin{proposition}
\label{prop:1}
Consider the following conditional independent model between IV candidates $w_j$'s and the true IV $z$. 
\begin{equation*}
\prob (w_1,\cdots,w_p,z) = \frac{1}{Z(\theta)} \exp \left(\theta_z z + \sum_{j \in V} \theta_{jz} w_j z \right).
\end{equation*}
We have that $\prob (w_j = 1 \mid z=1) = \prob (w_j = -1 \mid z= -1) = \prob (w_j=z) = \prob (a_j) = \prob(a_j \mid z)$, and $\prob(a_j,a_k) = \prob(a_j)\prob(a_k)$ for all $j,k \in V$, and $a_j:= w_j z$.
\end{proposition}

\begin{proof}
Consider $\prob (w_j, z)$ and $\prob (z)$:
\begin{align*}
\prob (w_j, z) = & \sum_{j'\ne j, j' \in V} \frac{\exp\left(\theta_z z + \theta_{jz} w_j z\right)}{Z(\theta)} \exp \left( \sum_{j'\ne j, j' \in V} \theta_{j'z} w_{j'} z\right)\\
= & \frac{\exp\left(\theta_z z + \theta_{jz} w_j z\right)}{Z(\theta)} \left[ \sum_{j'\ne j, j' \in V} \exp \left( \sum_{j'\ne j, j' \in V} \theta_{j'z} w_{j'} z \right) \right]\\
= & \frac{\exp\left(\theta_z z + \theta_{jz} w_j z\right)}{Z(\theta)} f_{-j}(z).
\end{align*}
\begin{align*}
\prob(z) = & \frac{\exp\left(\theta_z z\right)}{Z(\theta)} \left[ \sum_{j \in V} \exp\left(\sum_{j \in V} \theta_{jz} w_j z\right)\right] = \frac{\exp\left(\theta_z z\right)}{Z(\theta)} f_V(z).
\end{align*}
Since $\prob(w\mid z) = \prob(w,z) /\prob(z)$,
\begin{align*}
\prob (w_j=1 \mid z=1) = & \frac{\exp(\theta_z + \theta_{jz})}{Z(\theta)} \cdot \frac{f_{-j}(1)}{\prob (z=1)} = \exp \left(\theta_{jz}\right) \cdot \frac{f_{-j}(1)}{f_V(1)} \\
\prob (w_j=-1 \mid z=-1) = & \frac{\exp(-\theta_z + \theta_{jz})}{Z(\theta)} \cdot \frac{f_{-j}(-1)}{\prob (z=-1)} = \exp(\theta_{jz}) \cdot \frac{f_{-j}(-1)}{f_V(-1)}.
\end{align*}

Notice that,
\begin{align*}
f_V(z) = & \sum_{w_j \in \curly{-1,1}} \exp \left( \theta_{jz} w_jz\right) \sum_{j'\in V, j'\ne j} \exp \left(\sum_{j'\in V, j'\ne j} \theta_{j'z} w_{j'}z \right) \\
= & f_{-j}(z) \left[\exp (\theta_{jz} z) + \exp (-\theta_{jz} z)\right].
\end{align*}
Therefore,
\begin{align*}
\prob (w_j=1 \mid z=1) = & \exp \left(\theta_{jz}\right) \cdot \frac{f_{-j}(1)}{f_{-j}(1)} \cdot \frac{1}{\exp (\theta_{jz}) + \exp (-\theta_{jz})},\\
\prob (w_j=-1 \mid z=-1) = & \exp \left(\theta_{jz}\right) \cdot \frac{f_{-j}(-1)}{f_{-j}(-1)} \cdot \frac{1}{\exp (-\theta_{jz}) + \exp (\theta_{jz})} \\
\Rightarrow \prob (w_j=1 \mid z=1) = & \prob (w_j=-1 \mid z=-1).
\end{align*}
Furthermore, $\prob (a_j) = \prob(w_j=z) = \prob(w_j=1,z=1)+\prob(w_j=-1,z=-1) = \prob(w_j=1 \mid z=1) p(z=1)+\prob(w_j=-1 \mid z=-1) p(z=-1) = \prob(w_j=1 \mid z=1) = \prob(w_j=-1 \mid z=-1)$, where we have used the fact that $p(z=1)+p(z=-1)=1$ for the last two equalities. 

Finally, when $z=1$, $\prob(a_j=1 \mid z=1) = \prob(w_j=1 \mid z=1) = \prob (a_j=1)$ and  $\prob(a_j = -1 \mid z=1) = \prob(w_j=-1 \mid z=1) = \prob (a_j=-1)$. Similarly, when $z=-1$, $\prob(a_j=1 \mid z=-1) = \prob(w_j=-1 \mid z=-1)  = \prob (a_j=1)$ and $\prob(a_j=-1 \mid z=-1) = \prob(w_j=1 \mid z=-1)  = \prob (a_j=-1)$. Therefore, we can conclude that $\prob(a_j | z) = \prob(a_j)$. This further implies that $\prob(a_j,a_k) = \sum_{z\in\curly{-1,1}} \prob(a_j,a_k|z) \prob(z) = \sum_{z\in\curly{-1,1}} \prob(a_j|z) \prob(a_k|z) \prob(z) = \sum_{z\in\curly{-1,1}} \prob(a_j) \prob(a_k) \prob(z) = \prob(a_j) \prob(a_k)$.
\end{proof}

Proposition~\ref{prop:2} shows that how the accuracy parameters of a conditional independent Ising model of  \eqref{eq:ising} are independent of each other, using a model of three IV candidates as an example.
\begin{proposition}
\label{prop:2}
Let $w_1$, $w_2$, and $w_3$ follow:
\begin{equation*}
\prob(w_1,w_2,w_3,z) = \frac{1}{Z(\theta)} \exp \left(\theta_{1} w_1 + \theta_{z} z + \theta_{1z} w_1 z + \theta_{2z} w_2 z + \theta_{3z} w_3 z\right)
\end{equation*}
We have that $\prob(a_1,a_2) = \prob(a_1) \prob(a_2)$ and $\prob(a_1,a_3) = \prob(a_1) \prob(a_3)$.
\end{proposition}

\begin{proof}
Intuitively, rewrite $\prob(w_1,w_2,w_3,z)$ as:
\begin{equation*}
\prob(a_1,a_2,a_3,z) = \frac{1}{Z(\theta)} \exp \left(\theta_{1} a_1 z + \theta_{z} z + \theta_{1z} a_1 + \theta_{2z} a_2 + \theta_{3z} a_3\right).
\end{equation*}
Then $\prob(a_1,a_2,a_3,z) $ factorizes as 
$\prob(a_1,a_2,a_3,z) = \prob(a_1,z) \prob(a_2) \prob(a_3)$. It follows that $\prob(a_1,a_2) = \prob(a_1) \prob(a_2)$ and $\prob(a_1,a_3) = \prob(a_1) \prob(a_3)$.
\end{proof}

Proposition~\ref{prop:3} shows  how the soft label of $z$ is computed given $w$.
\begin{proposition}
\label{prop:3}
Let $w_1$,$w_2$,$\cdots$,and $w_p$ be given. The posterior probability of  $z=1$, i.e.\ $\prob (z=1 \mid w_1,\cdots,w_p)$, is given as
\begin{equation*}
\prob (z=1 \mid w_1,\cdots,w_p) =  
\sigma \left(\sum_{j=1}^p \log \frac{\prob (w_j \mid z=1)}{\prob (w_j \mid z=-1)} + \log \frac{\prob (z=1)}{\prob (z=-1)}\right),
\end{equation*}
where $\sigma(t) = \frac{1}{1+\exp(-t)}$ is the sigmoid function.
\end{proposition}

\begin{proof}
\begin{align*}
\prob (z=1 \mid & w_1,\cdots,w_p)  \\
= & \frac{\prob (w_1,\cdots,w_p \mid z=1) \prob (z=1)}{\prob(w_1,\cdots,w_p)} \\
= & \frac{\prob (w_1,\cdots,w_p \mid z=1) \prob (z=1)}{\prob(w_1,\cdots,w_p\mid z=1) \prob(z=1) + \prob(w_1,\cdots,w_p\mid z=-1) \prob(z=-1)} \\
= & \frac{1}{1+\frac{\prob(w_1,\cdots,w_p\mid z=-1) \prob(z=-1)}{\prob(w_1,\cdots,w_p\mid z=1) \prob(z=1)}} \\
= & \sigma \left( \log \frac{ \prob(w_1,\cdots,w_p\mid z=1)  \prob(z=1)}{ \prob(w_1,\cdots,w_p\mid z=-1) \prob(z=-1)} \right)\\
= & \sigma \left( \log \frac{\prod_{j=1}^p\prob (w_j \mid z=1) \prob(z=1)}{\prod_{j=1}^p\prob (w_j \mid z=-1) \prob(z=-1)} \right) \\
= & \sigma \left( \sum_{j=1}^p \log \prob (w_j \mid z=1) - \sum_{j=1}^p \log \prob (w_j \mid z=-1) + \log \prob(z=1) - \log \prob (z=-1)\right) \\
= & \sigma \left(\sum_{j=1}^p \log \frac{\prob (w_j \mid z=1)}{\prob (w_j \mid z=-1)} + \log \frac{\prob (z=1)}{\prob (z=-1)}\right)
\end{align*}
Note that when $w_j=1$, $\log \frac{\prob(w_j=1 \mid z=1)}{\prob (w_j=1 \mid z=-1)} = 1 \times \log \frac{\prob(w_j=1 \mid z=1)}{\prob (w_j=1 \mid z=-1)}$. When $w_j = -1$, $\log \frac{\prob(w_j=-1 \mid z=1)}{\prob (w_j=-1 \mid z=-1)} = -1 \times  \log \frac{\prob (w_j=-1 \mid z=-1)}{\prob(w_j=-1 \mid z=1)}$.
\end{proof}

\subsection{Example of Limitations of Ivy}
\label{sec:counterexample-appendix}
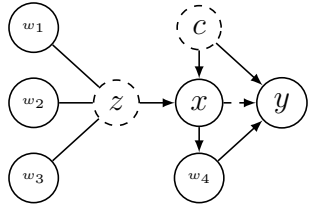
\begin{figure}
\centering
\centering
\begin {tikzpicture}[-latex ,auto ,node distance
=1.0cm and 1.1cm , on grid, semithick ,
state/.style ={ circle, draw, minimum width=0.5cm}, scale=0.50]
\node[state, draw] (C) [dashed] {\large $c$};
\node[state, draw] (Y) [below right=of C] {\large $y$};
\node[state] (X) [below =of C] {\large $x$};
\node[state] (Z) [left=of X, dashed] {\large $z$};
\node[state] (W1) [above left=of Z] {\tiny $w_1$};
\node[state] (W2) [left=of Z] {\tiny $w_2$};
\node[state] (W3) [below=of W2] {\tiny $w_3$};
\node[state] (W5) [below =of X] {\tiny $w_4$};
\path (Z) edge[-] (W1);
\path (Z) edge[-] (W2);
\path (Z) edge[-] (W3);
\path (Z) edge (X);
\path (C) edge (X);
\path (C) edge (Y);
\path[dashed] (X) edge (Y);
\path (X) edge (W5);
\path (W5) edge (Y);

\end{tikzpicture}
\caption{$w_4$ is an invalid IV that does not meet the assumption made by Ivy. }
\label{fig:counter}
\end{figure}
Here we show a counterexample of invalid IV that does not meet the assumption made by Ivy, as given in Figure~\ref{fig:counter}. As can be seen, $w_4$ is an invalid IV because it is directly linked to the outcome, violating the exclusion restriction assumption. However, Ivy cannot identify $w_5$ as invalid because $w_5$ is dependent on $z$.  $w_4$ is called a mediator. 
Thankfully, in Mendelian randomization, SNPs used as IVs are usually not mediators because a risk factor is usually a downstream product of genetic variation and hence is not causal to the status of a SNP.

\section{Extended Experiments}
\label{sec:extended-experiments}
Next we provide additional experiments and detail, including synthetic data experiments. We also present a series of experiments where we violate the key assumptions, investigating Ivy's robustness in cases where not all of them are met.

\subsection{Details of Experiments}
\label{sec:appendix-real-exp}
\textbf{Data Preprocessing} For real-world data, we acquire raw data from UK Biobank, which are subsequently binarized. For SNPs as IV candidates, we use the $\curly{-1,0,1}$ representation that reflects the dominant/recessive genetic model. To determine the encoding of the IV candidates that we anticipate to label the latent IV, we choose the encoding of each IV candidate that is positively correlated with the value of the risk factor. Individual-level data from unrelated subjects of European descent are used. 

\textbf{Allele Scores} Unweighted allele score assigns equal weight to the count of every genetic variant (IV candidate). Weighted allele score regresses the risk factor on the IV candidates to derive a weighted combination of the IV candidates. Since we have access to individual-level data, we derive the weights of the weighted allele score in a multiple regression fashion \citep{angrist1999jackknife,burgess2013use}.

\textbf{Implementation of Ivy} When covariance matrices are calculated, we treat these candidates as numeric variables. When curated putative valid IV candidates are used to estimate causal relationships we use  conditional independent Ivy models to learn the accuracy of the IV candidates. In other cases, we run the full Algorithm~\ref{alg:ivy} to estimate causal effects. 

Observe that using the loss function in Section~\ref{sec:ivy-framework}, we do not even need as many samples of the candidates as there are candidates---which would prevent us from inverting the sample covariance matrix. However, since in practice, many more samples are available, a direct approach is to perform this inversion and then apply the algorithm above directly to the inverted matrix, and we do so in our experiments.

It should be noticed that when we have access to a conditional independent Ivy model, one could directly estimate $\prob(z=1\mid w)$ by $\hat{\mu}$ due to Proposition~\ref{prop:3}. On the other hand, when we need to handle the dependencies among IV candidates, we cannot apply Proposition~\ref{prop:3} anymore. Instead, we make use of moment matching \citep{KollerFriedman} to map the mean parameters $(\hat{\mu},\hat{O})$ of the graphical model to its canonical parameters $\hat{\theta}$. Having access to $\hat{\theta}$, we can compute $\prob_{\hat{\theta}}(z=1\mid w)$ via standard graphical model inference procedures.

\textbf{Causal Effect Estimation} Once $\hat{G} = (\hat{V},\hat{E})$ is determined, we split the dataset into two separated halves at random, where the first half  is used to derive the instrumental variable model, and the second half is used to estimate causal effect. Doing so can avoid overfitting the data, similar to purpose of the practice described in \cite{burgess2013use,burgess2017review}. This procedure is repeated for $1000$ times to compute the median and the $95\%$ confidence interval of the causal effect estimate. We use the Wald estimator as our causal effect estimator. The interpretation of the Wald estimator is that the change of log-odd-ratio in the occurrence of outcome per unit change of the log-odd-ratio in the occurrence of the risk factor. Median of the Wald ratio estimate is recommended to describe the causal effect size \citep{Burgess15}.  A $95\%$ of confidence interval that covers the origin suggests that no causal relationship between the risk factor and the outcome. When conducting causal effect estimate using allele scores, we also obtain synthesized IV samples based on the probability suggested by the allele scores to account for uncertainty in the same way as we do for Ivy. 

\paragraph{Model Selection}  We consider a score-based model selection procedure, which can be viewed as an alternative to cross validation when it comes to choosing an appropriate set of hyperparameters \citep{hastie01statisticallearning}. Such a model selection procedure 
is used to determine the hyperparameters of Algorithm~\ref{alg:sl1}, specifically $\lambda$, $\gamma$, $T_1$, and $T_2$. We run Algorithm~\ref{alg:sl1} over the entire dataset using a grid of hyperparameters. From Line~\ref{state:scale-covariance} of Algorithm~\ref{alg:sl1} we have access to scores that correspond to the covariance between each of the $p+s$ IV candidates and $z$. We sort the absolute values of these scores from low to high and compute the ratios of the latter score over the former score. We denote the largest of the ratios corresponding to a given pair of $\lambda$ and $\gamma$ as $\tau_{\lambda,\gamma}$ and we denote its corresponding index in the sorted array as  $t_{\lambda,\gamma}$. 
Therefore, for each $\tau_{\lambda,\gamma}$, we consider the following model selection score:  $\log (\tau_{\lambda,\gamma}) \cdot \mathbb{I}(\tau_{\lambda,\gamma}>10) \cdot \exp(p+q-t_{\lambda,\gamma})$, and choose $\lambda$ and $\gamma$ corresponding to the largest score. Such a model selection score is designed to strike a balance between the number of IV candidates viewed as valid and the strength of the accuracy signal encoded by the covariance that indicates validity. Determining $T_1$ requires taking into consideration of various factors such as the total number of candidates, prior knowledge about the proportion of valid IV candidates available in the dataset, and the level of uncertainty of the causal estimate desired. We sort the values of  $\abs{\hat{\Sigma}\hat{l}}$ in ascending order and choose one of the values as $T_1$. The higher the total number of candidates and the higher the proportion of valid candidates the larger the index of $T_1$ in the sorted array we can choose. A larger $T_1$ can reduce the variance of the estimate but could also potentially induce more bias. In practice, we consider a $T_1$ that is indexed by $\xi \cdot (p+s-t_{\lambda,\gamma})$ with $\xi \in \curly{2,3}$. After selecting $\tau_{\lambda,\gamma}$, we select $T_2$ by providing the values of $\hat{S}$ to the Tukey's fence, which is an outlier detection rule \citep{yu1977exploratory}.  We then use the smallest outlier as the threshold for $T_2$. If there is no outlier, we view the model as conditional independent.

\begin{table}[t]
\centering
\begin{tabular}{cccccc}
\hline
Dataset & Task & \# Samples & \makecell{\# IVs \\ (Valid/Invalid)} & Ground Truth & Section \\
\hline
\texttt{hdl}$\Rightarrow$\texttt{cad} & Does HDL cause CAD? & 286,501 & 49 (19/30) & Noncausal & \ref{sec:uncurated-exp}\\
\texttt{crp}$\Rightarrow$\texttt{cad} & Does CRP cause CAD? & 311,442 & 160 (N.A.) & Noncausal & \ref{sec:uncurated-exp}\\
\texttt{vtd}$\Rightarrow$\texttt{cad} & Does VTD cause CAD? & 298,386 & 41 (N.A.) & Noncausal & \ref{sec:uncurated-exp}\\
\texttt{sbp}$\Rightarrow$\texttt{cad} & Does SBP cause CAD? & 332,998 & 35(N.A.) & Causal & \ref{sec:sbp-cad}\\
\hline
\texttt{hdl}$\to$\texttt{cad} & Does HDL cause CAD? & 286,501 & 19 (19/0) & Noncausal & \ref{sec:valid-exp} \\
\texttt{ldl}$\to$\texttt{cad} & Does LDL cause CAD? & 311,559 & 19 (19/0) & Causal & \ref{sec:valid-exp}\\
\texttt{sbp}$\to$\texttt{cad} & Does SBP cause CAD? & 332,998 & 26(26/0) & Causal & \ref{sec:sbp-cad}\\
\hline
\texttt{hdl} & valid vs invalid IVs of HDL & 286,501 & 49 (19/30) & N.A. & \ref{sec:robust-exp}\\
\texttt{ldl} & valid vs invalid IVs of LDL & 311,559 & 42 (19/23) & N.A. & \ref{sec:robust-exp}\\
\texttt{trg} & valid vs invalid IVs of TRG & 311,861 & 68 (27/41) & N.A. & \ref{sec:robust-exp}\\
\texttt{hdl}-\texttt{ldl} & HDL IVs vs LDL IVs & 286,062 & 38 (19/19) & N.A. & \ref{sec:robust-exp}\\
\texttt{hdl}-\texttt{trg} & HDL IVs vs TRG IVs & 286,289 & 46 (19/27) & N.A. & \ref{sec:robust-exp}\\
\texttt{ldl}-\texttt{trg} & LDL IVs vs TRG IVs & 311,368 & 46 (19/27) & N.A. & \ref{sec:robust-exp}\\
\hline
\end{tabular}
\captionof{table}{Summary of real-world data used in the experiments. HDL: high-density lipoprotein; LDL: low-density lipoprotein; TRG: triglyceride; SBP: systolic blood pressure; CRP: C-reactive protein; VTD: vitamin D; CAD: coronary artery disease.}
\label{tab:datasets}
\end{table}

\subsection{Extended Real-World Experiments}

\begin{figure}[t]
\begin{subfigure}[b]{0.5\textwidth}
\centering
\includegraphics[scale=0.5]{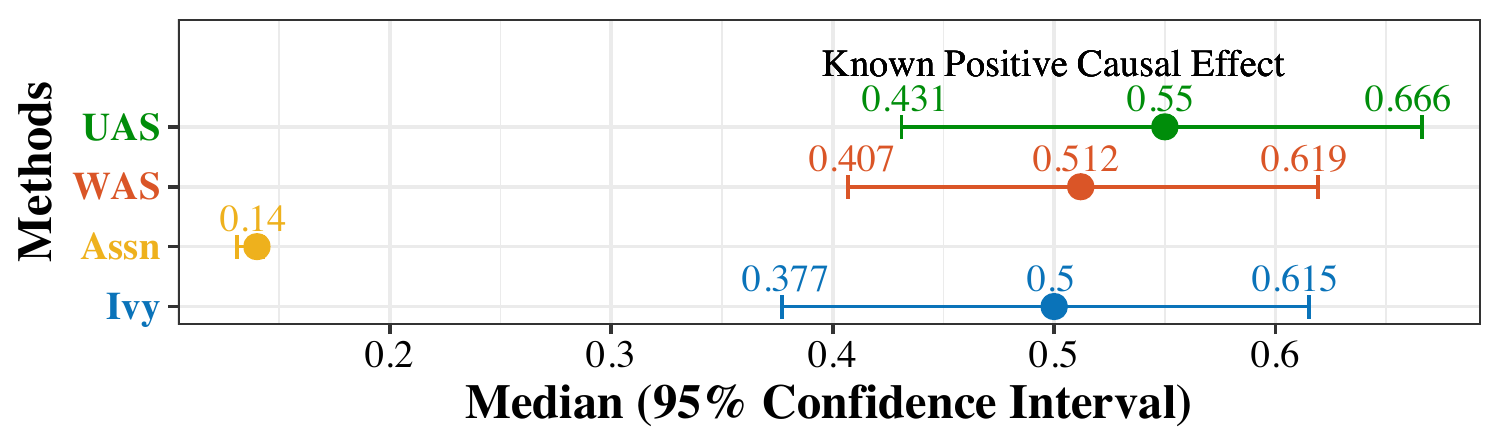}
\caption{\texttt{sbp}$\rightarrow$\texttt{cad} with curated IV candidates}
\label{fig:sbp-cad-curated}
\end{subfigure}
~
\begin{subfigure}[b]{0.5\textwidth}
\centering
\includegraphics[scale=0.5]{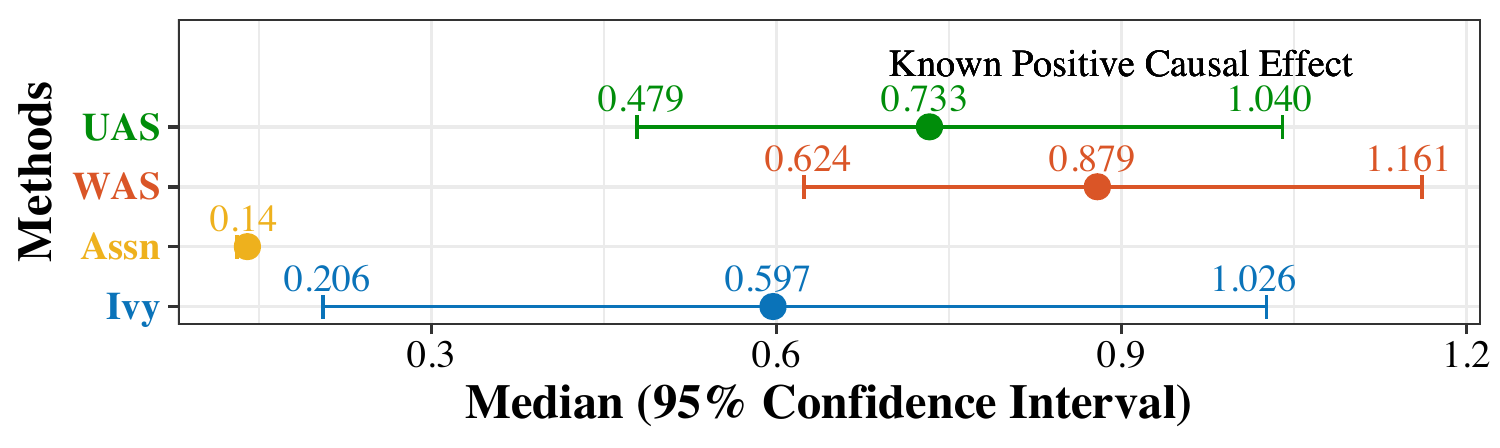}
\caption{\texttt{sbp}$\rightarrow$\texttt{cad} with uncurated IV candidates}
\label{fig:sbp-cad-uncurated}
\end{subfigure}
\caption{Estimation of the known positive causal effect of systolic blood pressure on coronary artery disease using both curated and uncurated IV candidates}
\end{figure}

\begin{figure}[t]
\centering
\includegraphics[scale=0.8]{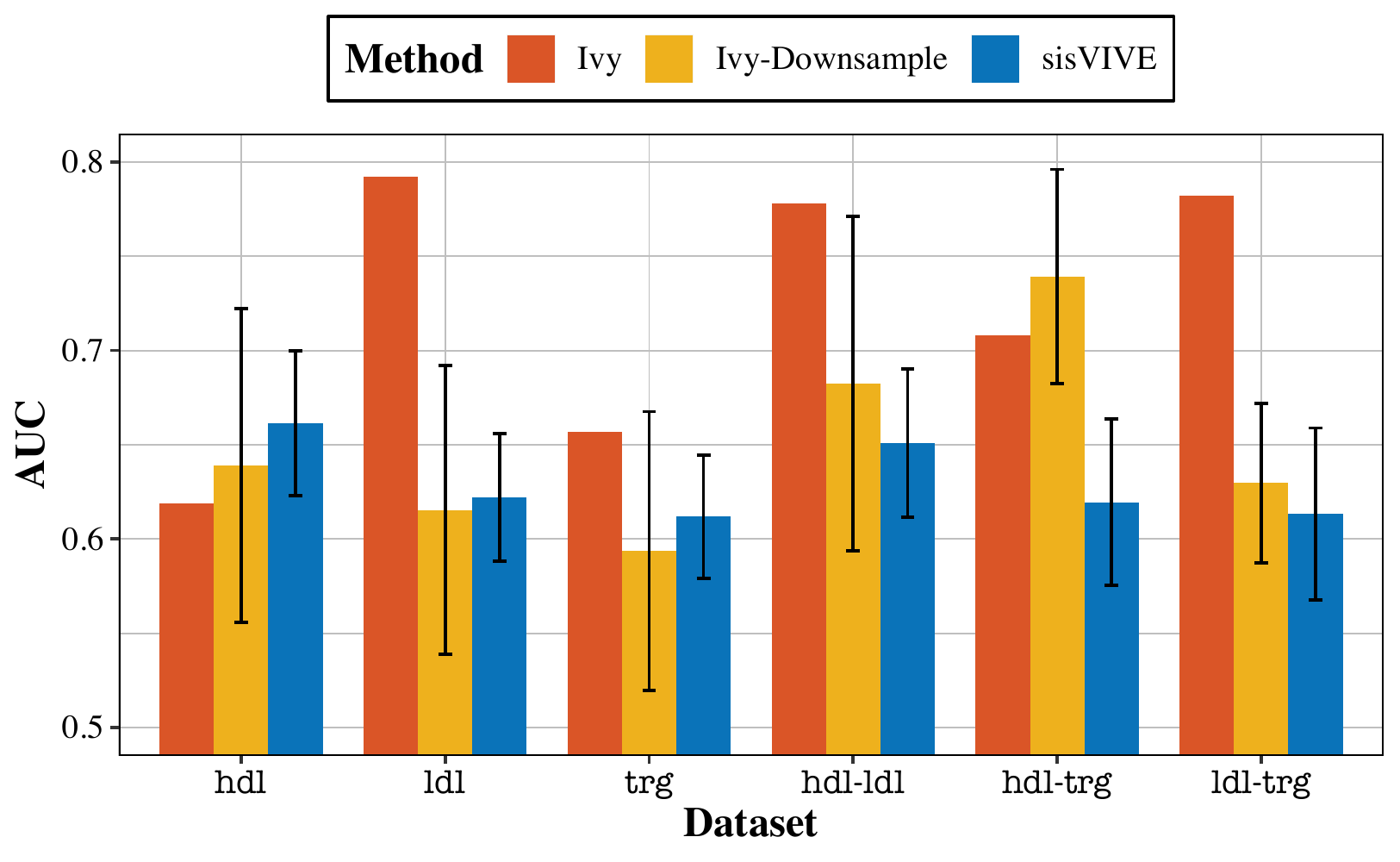}
\caption{AUCs of Ivy and \texttt{sisVIVE} in distinguishing between valid and invalid IVs across six datasets. Ivy: the proposed method run on the full datasets with model selection. Ivy-Downsample: only use the subsets of data (20,000 samples) that match those used in \texttt{sisVIVE} with model selection. \texttt{sisVIVE}: best performer run on 20,000 data points.}
\label{fig:invalid-exp}
\end{figure}

We first discuss the curation process of the twelve real-world datasets that we use in our experiments (Section~\ref{sec:dataset}). We then report the experimental results of estimating a true causal relationship between systolic blood pressure and coronary artery disease on two real-world datasets with curated and uncurated IV candidates respectively (Section~\ref{sec:sbp-cad}). Finally, we compare Ivy with a leading IV-based robust causal inference approach \texttt{sisVIVE}  in terms of distinguishing between valid and invalid IVs on six real-world datasets (Section~\ref{sec:robust-exp}).

\subsubsection{Real-World Datasets}
\label{sec:dataset}
The real-world datasets used in our experiments are summarized in Table~\ref{tab:datasets}. We describe how each dataset is produced. All the datasets consist of individual-level data from UK Biobank including data of the risk factor, the outcome, and IV candidates (SNPs). We report what SNPs are chosen as IV candidates for each dataset. In \texttt{hdl$\Rightarrow$cad}, the SNPs are chosen according to \citealt{holmes2014mendelian}, where 19 SNPs are reported to be putatively valid IVs and 30 are invalid. For \texttt{crp$\Rightarrow$cad} and \texttt{vtd$\Rightarrow$cad}, SNPs are chosen as IV candidates as long as they are reported to be associated with the corresponding risk factor among individuals of European descent in the GWAS Catalog \citep{buniello2018nhgri}. In this case, we do not know the validity of the IV candidates, faithfully reflecting the challenges of MR in practice. See Section~\ref{sec:sbp-cad} for the curation process of \texttt{sbp$\Rightarrow$cad} and \texttt{sbp$\rightarrow$cad}. The IV candidates and their validity of the rest of the datasets are also determined according to \cite{holmes2014mendelian}.

\subsubsection{Estimate True Causal Relationships using Ivy}
\label{sec:sbp-cad}
Here we consider estimating the true causal effect of SBP to CAD using 26 curated IV candidates and 35 uncurated IV candidates. The curated IV candidates are due to the Mendelian randomization conducted in \cite{lieb2013genetic}. For the uncurated IV candidates, we identify 15 SNPs that are most significantly correlated with SBP based on the findings in \cite{ehret2011genetic}. As a proxy to noisy candidates weakly correlated with SBP, we also identify 20 additional SNPs from the same study whose correlations are less significant. This results in a total of 35 uncurated IV candidates.

Experimental results of using the curated IV candidates are reported in Figure~\ref{fig:sbp-cad-curated}. Using curated IVs, Ivy performs similarly compared to UAS and WAS both in terms of the median estimate and the length of confidence intervals. 

Experimental results of using the uncurated IV candidates are reported in Figure~\ref{fig:sbp-cad-uncurated}. With uncurated IVs, Ivy maintains a median estimate similar to that when the curated IVs are used. However, both UAS and WAS yield different estimates compared to the case where curated IV candidates are used.

\subsubsection{Valid/Invalid IVs Classification}
\label{sec:robust-exp}
Since properly handling invalidity is a crucial aspect of the synthesis phase, we conduct ablation experiments of valid/invalid IV candidate classification, on datasets where such ground truth is available. On six real-world datasets, Ivy outperforms or remains comparable to a leading approach (\texttt{sisVIVE}, \citealt{kang2016instrumental}) for this classification task (as depicted in Figure~\ref{fig:invalid-exp}).

\texttt{sisVIVE} is a leading robust IV-based causal inference approach.
As it is a one-phase method (unlike Ivy), \texttt{sisVIVE} is not designed to synthesize a summary IV and is not usually combined with other causal effect estimators. Nevertheless, one of its intermediate outputs is an estimate of which candidates are valid. This leads us to ask whether Ivy is competitive with this method on this task, despite being primarily designed for IV synthesis
(Note that we do not compare to UAS and WAS in Section~\ref{sec:robust-exp}, because they assume all candidates are valid and thus do not distinguish between valid and invalid candidates).

We frame distinguishing between valid and invalid IVs as a binary classification problem. Therefore, we can use the area under curve (AUC) of the receiver operating characteristic of the classification to measure the capacity of a method to tell apart valid IVs from invalid ones. Algorithm~\ref{alg:sl1} is used for classification in Ivy.  
A total of six datasets are used for evaluation (see Table~\ref{tab:datasets} for details). 
Results are presented in Figure~\ref{fig:invalid-exp}.

We report the results of two variants of Ivy. For the first one, we run Ivy on the full datasets with model selection and report the AUC. For the second  (Ivy-Downsample), we run Ivy on subsets of 20,000 data points across the full dataset with model selection and report the mean and standard deviation of the AUC across all the subsets for each dataset. This is because \texttt{sisVIVE} fails to run on the full datasets due to its large memory footprint; thus, we run \texttt{sisVIVE} on subsets 20,000 data points for each dataset and compare it with Ivy run on the same subsets of the data (Ivy-Downsample). For \texttt{sisVIVE}, we report the result of the \emph{best} performer.
As can be seen in Figure~\ref{fig:invalid-exp}, both variants of Ivy result in competitive performance in AUCs compared with the best performer of \texttt{sisVIVE}. 
This suggests that Ivy is capable of handling, and generally benefits from, higher sample sizes, as shown by the increase in AUC from the downsampled version of Ivy to the full Ivy.

\subsection{Further Synthetic Experiments}
\label{sec:appendix-synthetic-exp}
We evaluate various aspects of the empirical performance of Ivy via a series of experiments on synthetic data.  We seek to show that:
\begin{itemize}[leftmargin=*]
\item Ivy can estimate causal effects with noisy, dependent, and potentially invalid IV candidates (Section~\ref{sec:appendix-pipeline}).
\item When the accuracies of IV candidates vary, Ivy can benefit from estimating the accuracies of IV candidates compared to UAS that views all candidates to be of the same accuracy (Section~\ref{sec:accuracy}).
\item When the IV candidates are dependent on each other, Ivy can benefit from estimating and utilizing these dependencies, in contrast to UAS and WAS that do not model such dependencies (Section~\ref{sec:dependency}).
\item When $z$ itself becomes an invalid IV, Ivy can demonstrate certain level of robustness while UAS and WAS can be more sensitive to the assumption violation (Section~\ref{sec:invalid-z}).

\end{itemize}

\subsubsection{Ivy With Synthetic Data}
\label{sec:appendix-pipeline}

\paragraph{Dismissing Spurious Correlations}
Next, we give more details on the synthetic experiment with null causal effect. To demonstrate the use of Ivy in causal inference using noisy,  dependent, contradicting, and partially invalid IV candidates, we consider the use of 20 IV candidates to dismiss a positive spurious correlation between a risk factor and an outcome. Among the 20 candidates, 10 of them are valid IVs and 10 of them are invalid by being associated with the confounder that produces the spurious correlation. Within the 10 valid candidates, a clique of four valid candidates and a clique of two valid candidates are formed.  The remaining four valid candidates are conditional independent upon $z$. All the ten invalid candidates are conditional independent upon the confounder. A total of 100,000 samples are generated from this model. UAS and WAS are used in comparison to Ivy. Observational association between the risk factor and the outcome is also computed as a reference. We expect Ivy to dismiss the spurious correlation successfully, while WAS and UAS will fail to do so. The causal effects estimate are reported in Figure~\ref{fig:pipeline-null}, medians and $95\%$ confidence intervals are generated through 100 times of subsampling. Ivy is capable of recovering the dependency structure among the candidates and identify invalid candidates. As a result, Ivy can successfully dismiss the spurious correlation by identifying no causal effects while both UAS and WAS fail to do so by yielding estimates that are consistent with the direction of the spurious correlation.

\paragraph{Estimating True Causal Effects} Finally, we discuss the last experiment, where there is a ground truth (synthetic) positive causal effect. We use the same experiment setup described in the previous paragraph to estimate true causal effects. The only difference is that there is a true causal effect from the risk factor to the outcome in the data generation model. The true effect size measured by the log odd ratio is 0.150. Experimental results are reported in Figure~\ref{fig:pipeline-true}.  Ivy provides a median estimate that is closest to the true causal effect size while both UAS and WAS return median estimates that bias towards the observational association due to their failure in distinguishing between valid candidates and invalid ones that are associated with the confounder.

Next, we perform several more synthetic experiments, where we vary the accuracies and the dependencies.

\subsubsection{Candidates with Varying Accuracy}
\label{sec:accuracy}

\begin{table}
\centering
\begin{tabular}{ccc}
\hline
Method & Median & $95\%$ CI\\
\hline
Ivy &  0.266 &  [-0.247, 0.784] \\
UAS &  0.322 & [-0.571, 1.308] \\
WAS &  0.300 & [-1.342, 1.994] \\
Association & 0.432 & [0.374, 0.492] \\
\hline
\end{tabular}
\caption{Dismiss spurious correlations with candidates of varying accuracies}
\label{tab:accuracy}
\end{table}

We demonstrate the utility of Ivy in dealing with candidates of varying accuracies by considering a model of ten conditional independent valid IV candidates. The ten candidates are moderately accurate with accuracies of around 0.6 while $\prob(z=1) \approx 0.6$. We further generate 50 independent binary noise variables to represent (invalid) candidates that are not predictive of $z$ at all. A total of 5,000 samples are generated to dismiss the spurious correlation between a risk factor and an outcome. Median and 95\% confidence intervals are calculated through 1,000 times subsampling. We expect that Ivy can generate a narrower confidence interval compared to allele scores because Ivy is capable of estimating the accuracy of the candidates and downweight those that are less accurate so as to reach a more certain estimation of $z$. Results are reported in Table~\ref{tab:accuracy}. We observed that all methods are successful at dismissing the spurious correlation while Ivy indeed yields a narrower confidence interval compared to UAS and WAS. The estimate of WAS is especially uncertain. This demonstrate the need of more samples for WAS in order to yield more certain estimate when the number of candidates are relatively large 60 candidates in this case).

\subsubsection{Candidates with Dependencies}
\label{sec:dependency}

\begin{figure}
\centering
\begin{subfigure}[b]{0.38\textwidth}
\centering
\begin{tikzpicture}[-latex ,auto ,node distance
=1.0cm and 1.1cm , on grid, semithick ,
state/.style ={ circle, draw, minimum width=0.5cm}, scale=0.50]
\node[state] (Z) [left=of X, dashed] {\large $z$};
\node[state] (W1) [above left=of Z] {\tiny $w_1$};
\node[state] (W2) [left=of Z] {\tiny $w_2$};
\node[state] (W3) [below=of W2] {\tiny $w_3$};
\node[state] (W4) [below =of Z] {\tiny $w_4$};
\node[state] (W5) [below right=of Z] {\tiny $w_5$};
\node[state] (W6) [right =of Z] {\tiny $w_6$};
\node[state] (W7) [above right=of Z] {\tiny $w_7$};
\node[state] (W8) [above =of Z] {\tiny $w_8$};
\path (Z) edge[-]  (W1);
\path (Z) edge[-] (W2);
\path (Z) edge[-] (W3);
\path (Z) edge[-] (W4);
\path (Z) edge[-] (W5);
\path (Z) edge[-] (W6);
\path (Z) edge[-] (W7);
\path (Z) edge[-] (W8);
\path (W1) edge[-] (W2);
\path (W1) edge[-, bend right=40] (W3);
\path (W1) edge[-, bend right=90, looseness=2] (W4);
\path (W2) edge[-] (W3);
\path (W2) edge[-] (W4);
\path (W3) edge[-] (W4);
\end{tikzpicture}
\caption{Dependency graph}
\label{fig:dependency}
\end{subfigure}
\begin{subfigure}[b]{0.58\textwidth}
\centering
\begin{tabular}{ccc}
\hline
Method & Median & $95\%$ CI\\
\hline
Ivy &  -0.092  &  [ -0.375, 0.109] \\
UAS &  -0.188  & [-0.735, 0.153] \\
WAS &  -0.039  & [-0.338, 0.180] \\
Association &  0.379 & [0.355, 0.400] \\
\hline
\end{tabular}
\caption{Causal effect estimation}
\label{tab:dependency}
\end{subfigure}
\caption{Dismissing spurious correlations using candidates with dependencies}
\end{figure}

We investigate the use of Ivy in handling IV candidates that are dependent on each other. This scenario arises in common practice of Mendelian randomization when the SNPs served as IV candidates are in linkage disequilibrium. We consider a model with eight valid candidates, as shown in Figure~\ref{fig:dependency}. Four of the candidates are conditional independent upon $z$, while the remaining four form a clique of high dependency that yield Pearson correlations among these four candidates of about 0.77. The four conditionally independent candidates are more predictive of $z$ than the four dependent ones. A total of 50,000 samples are generated. We use these data to dismiss the spurious correlation between a risk factor and an outcome. UAS and WAS are used as a comparison to Ivy. Median and $95\%$ confidence interval of the Wald ratio is calculated through 100 times of subsampling. We expect that Ivy can learn and utilize the dependencies among candidates and yields a reasonably precise estimate. Results are summarized in Table~\ref{tab:dependency}. Ivy, UAS, and WAS can all dismiss the spurious correlation, with the confidence interval of Ivy being the narrowest. 

In order to understand the impact of the failure of modeling the dependencies among the candidates, we also use a version of Ivy that assumes that all candidates are conditional independent to estimate causal effects. Under the same experiment configuration as aforementioned, the conditional independent Ivy produces a median of $-0.327$ and a $95\%$ confidence interval of $[-16.967, 15.398]$. By ignoring the strong dependencies among candidates, Ivy essentially fails in the estimation by producing a highly imprecise estimate. This stresses the importance of handling dependency appropriately within the Ivy framework, either by direct modeling or by just making use of candidates that are conditional independent of each other.

\subsubsection{Violating the Key Assumption}
\label{sec:invalid-z}
We give more details on the experiment where we investigate the robustness of Ivy against the violation of the key assumption---that $z$ is a valid IV. On synthetic data, we show that Ivy yields a causal estimate that deviates the least from the ground truth compared to allele score methods (Figure~\ref{fig:invalid-z-result}).

Here, we consider the spurious correlation model given in Figure~\ref{fig:invalid-z}. There are nine IV candidates in the model. $w_9$ serves as a confounder between the risk factor and the outcome. $z$ is invalid because $z$ is moderately associated with $w_9$. Then, we vary the strength of this association (i.e., tune it) and examine the results. We take
$\prob(y=1 \mid w_9=1)=\prob(y=-1 \mid w_9=-1)=0.55$, $\prob(x=1 \mid w_9=1)=0.764$, and $\prob(x=-1 \mid w_9=-1)=0.776$. Furthermore, $\prob(w_j=1 \mid z=1)=\prob(w_j=-1 \mid z=-1)=0.73$, where $j \in [8]$. We vary the accuracy of $w_9$ in predicting $z$ as $\prob(w_9=1 \mid z=1)=\prob(w_9=-1 \mid z=-1) \in \curly{0.5,0.525,0.55,0.575,0.6}$.
50,000 samples are generated from each of these variations. We apply Ivy, UAS, and WAS to these nine candidates for causal effect estimation.  Median and $95\%$ confidence intervals are computed through 100 times sampling. We expect Ivy to demonstrate certain level of robustness by downweighting the influence of $w_9$ while UAS and WAS will not be able to do so. Results are given in Figure~\ref{fig:invalid-z-result}, where Ivy does not detect causal effects up to the accuracy of $w_9$ in predicting $z$ being $0.55$ while UAS and WAS fail to dismiss the spurious correlation because they consider the invalid $w_9$, which is fairly predictive of $x$ by construction, as a valid IV. It should be noted that as the invalidity of $z$ becomes more significant, all three methods eventually fail to dismiss the spurious correlation eventually. This emphasizes the importance of the validity assumption upon $z$.

\subsubsection{Calibration of Confidence Intervals}
We conduct further experiments on synthetic data to show that the confidence intervals are well-calibrated. We follow the protocol established in \cite{burgess2013use}. Specifically, we consider a data generation model with 10 conditional independent valid IV candidates to estimate a spurious relationship (true causal effect size = 0). We sample 1,000 datasets of 10,000 samples each from the data generation model. For each dataset, we compute a $95\%$ confidence interval of the causal estimate, yielding 1000 empirical confidence intervals in total. The percentage of these empirical confidence intervals that cover 0 can then serve as a measure of the calibration quality: if the model is well-calibrated, this percentage should be close to $95\%$. In our experiment, we observe a $94.6\%$ coverage. This supports the hypothesis that the confidence intervals produced by Ivy are well-calibrated.

\end{document}